\PassOptionsToPackage{round,authoryear}{natbib}
\RequirePackage{fix-cm}
\documentclass[twoside,11pt,tablecaptionbottom]{jmlr}
\makeatletter\@twosidetrue\@mparswitchtrue\makeatother

\usepackage[T1]{fontenc}
\usepackage[utf8]{inputenc}
\usepackage{amsmath,amssymb,amsfonts,mathtools,bm}
\DeclareMathSizes{7.7}{7}{5}{5}
\DeclareMathSizes{7.1}{7}{5}{5}
\DeclareMathSizes{5.2}{5}{5}{5}
\DeclareMathSizes{6.55}{7}{5}{5}
\usepackage{booktabs,multirow,array,longtable}
\usepackage{graphicx,adjustbox}
\makeatletter
\@ifundefined{@org@Ginclude@graphics}{}{\let\Ginclude@graphics\@org@Ginclude@graphics}
\makeatother
\usepackage{xcolor}
\usepackage{enumitem}
\usepackage{microtype}
\usepackage{siunitx}
\usepackage{float}
\usepackage{placeins}
\usepackage{url}
\usepackage{hyperref}
\usepackage[nameinlink,capitalize,noabbrev]{cleveref}
\usepackage{tikz}
\usetikzlibrary{arrows.meta,positioning,fit,calc}

\hypersetup{
  hidelinks,
  hypertexnames=false,
  pdftitle={Training Under Challenge: Executable Certificates and Challenge-Closed Optimality for Neural Networks},
  pdfauthor={Farhang Yeganegi, Arian Eamaz, and Mojtaba Soltanalian},
  pdfsubject={Executable neural training certificates, current-state challenge coverage, challenge-closed optimality, representation sufficiency, and quantized neural networks},
  pdfkeywords={training assurance, executable challenges, spectral coverage, representation sufficiency, quantized neural networks}
}

\graphicspath{{figures/}}
\definecolor{tucblue}{HTML}{1F5A85}
\definecolor{tucgreen}{HTML}{238B57}
\definecolor{tucgold}{HTML}{A36A00}
\definecolor{tucred}{HTML}{C8374A}
\definecolor{tucgray}{HTML}{606770}
\definecolor{palegreen}{HTML}{E8F5EE}
\definecolor{paleblue}{HTML}{EAF2F8}
\definecolor{palegold}{HTML}{FFF7DB}
\definecolor{palered}{HTML}{FAE8EB}
\definecolor{palegray}{HTML}{F1F3F4}

\definecolor{DeepBlue}{HTML}{1F4E79}
\definecolor{SoftBlue}{HTML}{EAF2F8}
\definecolor{SeaGreen}{HTML}{287A5A}
\definecolor{SoftGreen}{HTML}{EAF6F0}
\definecolor{WarmOrange}{HTML}{B76517}
\definecolor{SoftOrange}{HTML}{FDF1E6}
\colorlet{tucorange}{WarmOrange}
\definecolor{GlanceGold}{HTML}{9A7200}
\definecolor{SoftGold}{HTML}{FFF8D8}
\definecolor{SlateBlue}{HTML}{4C6573}
\definecolor{SoftSlate}{HTML}{F3F6F8}
\newcommand{\glancebox}[3]{%
  \begingroup
  \setlength{\fboxsep}{5.5pt}%
  \setlength{\fboxrule}{0.8pt}%
  \noindent\fcolorbox{#1}{#2}{%
    \begin{minipage}{\dimexpr\linewidth-2\fboxsep-2\fboxrule\relax}
      #3
    \end{minipage}}%
  \endgroup
}
\newcommand{\glanceboxfixed}[4]{%
  \begingroup
  \setlength{\fboxsep}{5.5pt}%
  \setlength{\fboxrule}{0.8pt}%
  \noindent\fcolorbox{#1}{#2}{%
    \begin{minipage}[t][#3][t]{\dimexpr\linewidth-2\fboxsep-2\fboxrule\relax}
      #4
    \end{minipage}}%
  \endgroup
}
\newcommand{\glanceragged}{%
  \raggedright
  \hyphenpenalty=10000
  \exhyphenpenalty=10000
  \emergencystretch=1.5em
}

\crefname{theorem}{theorem}{theorems}
\Crefname{theorem}{Theorem}{Theorems}
\crefname{proposition}{proposition}{propositions}
\Crefname{proposition}{Proposition}{Propositions}
\crefname{corollary}{corollary}{corollaries}
\Crefname{corollary}{Corollary}{Corollaries}
\crefname{assumption}{assumption}{assumptions}
\Crefname{assumption}{Assumption}{Assumptions}
\crefname{definition}{definition}{definitions}
\Crefname{definition}{Definition}{Definitions}
\crefname{lemma}{lemma}{lemmas}
\Crefname{lemma}{Lemma}{Lemmas}

\newcommand{\R}{\mathbb R}
\newcommand{\E}{\mathbb E}
\newcommand{\Prb}{\mathbb P}
\newcommand{\norm}[1]{\left\lVert#1\right\rVert}
\newcommand{\fro}[1]{\left\lVert#1\right\rVert_{\!F}}
\newcommand{\opnorm}[1]{\left\lVert#1\right\rVert_{\mathrm{op}}}
\newcommand{\ip}[2]{\left\langle #1,#2\right\rangle}
\newcommand{\rank}{\operatorname{rank}}
\newcommand{\range}{\operatorname{range}}
\newcommand{\tr}{\operatorname{tr}}
\newcommand{\diag}{\operatorname{diag}}
\newcommand{\vecop}{\operatorname{vec}}
\newcommand{\argmin}{\operatorname*{arg\,min}}

\newcommand{\KL}{D_{\mathrm{KL}}}
\newcommand{\green}{\textcolor{tucgreen}{\textbf{Green}}}
\newcommand{\yellow}{\textcolor{tucgold}{\textbf{Yellow}}}
\newcommand{\redstatus}{\textcolor{tucred}{\textbf{Red}}}

\title[Training Under Challenge]{Training Under Challenge: Executable Certificates and\\Challenge-Closed Optimality for Neural Networks}

\author[Yeganegi, Eamaz, and Soltanalian]{%
\Name{Farhang Yeganegi} \Email{fyegan2@uic.edu}\\
\addr Department of Electrical and Computer Engineering \\University of Illinois Chicago\\Chicago, Illinois, USA
\AND
\Name{Arian Eamaz} \Email{aeamaz2@uic.edu}\\
\addr Department of Electrical and Computer Engineering \\University of Illinois Chicago\\Chicago, Illinois, USA
\AND
\Name{Mojtaba Soltanalian} \Email{msol@uic.edu}\\
\addr Department of Electrical and Computer Engineering \\University of Illinois Chicago\\Chicago, Illinois, USA}
\jmlrauthors{Farhang Yeganegi, Arian Eamaz, and Mojtaba Soltanalian}

\begin{document}
\maketitle
\pagestyle{jmlrps}

\begin{abstract}
A flat training curve does not reveal whether a neural network has reached a global optimum, is locally trapped, is representation-limited, or is mismatched to its trainer. We introduce \emph{Training Under Challenge}, an executable-certificate framework in which predeclared, architecture-valid procedures construct complete alternatives in the same certified class and reevaluate the same objective. Any lower-valued candidate is a replayable witness that lower-bounds the checkpoint's empirical global-optimality gap. Passing a finite suite is only suite-relative; global-gap conclusions require a separately justified coverage mechanism. We define a resource-indexed challenge-power modulus that characterizes the largest gap compatible with passage. For squared loss, current block-decrease operators make coverage checkable and yield uniform and realized-residual bounds. We prove the converse frontier: without coverage, a first-order ReLU trainer can reach infinitely many exact conditional head optima while converging to a non-global point. On a channel-gated ResNet-18 distillation problem with known optimum, eight internal challenges cover all 240 audited output directions, and realized-residual bounds lie within factors of $1.74$--$3.02$ of the true gap. Paired predictive certificates separate decoder under-use from representation insufficiency, while quantized-denoising studies demonstrate diagnosis, repair, and current-state recertification.
\end{abstract}
\vspace{.4cm}



\clearpage
\begingroup
\thispagestyle{empty}
\setlength{\parindent}{0pt}
\setlength{\parskip}{0pt}

\begin{center}
{\Large\bfseries\color{DeepBlue}
At a Glance: The Executable Certificate Framework}\par
\vspace{2pt}
{\small
Complete alternative models establish failure constructively;
passage, coverage, tracking, and protected evidence determine
the strongest conclusion earned.}
\end{center}

\vspace{5pt}


\glancebox{DeepBlue}{SoftSlate}{%
\fontsize{9.2}{10.7}\selectfont
\glanceragged

\begin{center}

\textbf{\texttt{\MakeUppercase{Online status and post-passage examination}}}\par\vspace{10pt}
\end{center}

{\fontsize{8.9}{10.2}\selectfont

\begin{minipage}[t]{0.19\linewidth}
\centering
\textcolor{tucred}{\bfseries Red}\par
$J_t>J_{\mathrm{ref}}+\tau_G$\par
A feasible reference already beats the checkpoint.
\end{minipage}\hfill
\begin{minipage}[t]{0.19\linewidth}
\centering
\textcolor{tucgold}{\bfseries Yellow}\par
$G_t+\tau_G<J_t\le J_{\mathrm{ref}}+\tau_G$\par
The reference is passed, but executable headroom remains.
\end{minipage}\hfill
\begin{minipage}[t]{0.19\linewidth}
\centering
\textcolor{tucgreen}{\bfseries Green}\par
$J_t\le G_t+\tau_G$\par
The named current suite is passed at its tolerance and budget.
\end{minipage}\hfill
\begin{minipage}[t]{0.19\linewidth}
\centering
\textcolor{DeepBlue}{\bfseries Green-0}\par
The first current Green event launches the stronger policy.
\end{minipage}\hfill
\begin{minipage}[t]{0.19\linewidth}
\centering
\textcolor{tucgray}{\bfseries Gray}\par
A retained stronger target is feasible but not yet earned by the trainer.
\end{minipage}

\vspace{.5cm}

\textit{Notation:}
$J_t:=J(\theta_t)$ is the current checkpoint value;
$J_{\mathrm{ref}}:=J(\theta_{\mathrm{ref}})$ is the value of a fixed feasible reference;
$G_t$ is the best retained executable frontier through time $t$;
$\tau_G\ge0$ is the declared Green tolerance;
and $E_t\in\{0,1\}$ is the mandatory current-suite completion flag.
The Red, Yellow, and Green inequalities apply when $E_t=1$; if $E_t=0$, the current report is \emph{Inconclusive}.
All values use the same certified objective.

}\par\vspace{4pt}
}

\vspace{5pt}


\noindent
\begin{minipage}[t]{0.49\linewidth}

\glanceboxfixed{DeepBlue}{SoftBlue}{240pt}{%
\fontsize{9.1}{10.4}\selectfont
\glanceragged

\textbf{\texttt{\MakeUppercase{CENTRAL THEORETICAL CLAIMS}}}\par\vspace{7pt}

\textbf{Executable witnesses.}
If a complete challenger has value $B<J_t$, then
\[
J_t-J^\star\ge J_t-B.
\]
The model itself is retained as the replayable evidence.

\vspace{3pt}

\textbf{Budgeted finite passage.}
Challenge power $\Psi(\mathcal B,s)$ and its inverse $E(\mathcal B,\tau)$ connect
audit resources, detectable suboptimality, and the largest gap
compatible with passage.

\vspace{3pt}

\textbf{Current spectral coverage.}
Certified decrease operators yield uniform, nullspace-aware,
normal-residual, and realized-residual gap bounds.

\vspace{3pt}

\textbf{Exact globality frontier.}
Coverage upgrades passage to near-globality.
Without coverage, infinitely many reached exact conditional
head optima can remain non-global.

\vspace{3pt}

\textbf{Paired representation certificates.}
Outer and representation-preserving brackets separate decoder
under-use from representation insufficiency.

}

\end{minipage}\hfill
\begin{minipage}[t]{0.49\linewidth}

\glanceboxfixed{DeepBlue}{SoftBlue}{240pt}{%
\fontsize{9.1}{10.4}\selectfont
\glanceragged

\textbf{\texttt{\MakeUppercase{PRACTITIONER'S GUIDE}}}\par\vspace{6pt}

\textbf{Declare the problem.}
Freeze the certified class and objective, feasible reference model,
mandatory suite, budgets, tolerances, solver policies, permitted
information, optional null, and protected-task data.

\textbf{Audit the checkpoint.}
Materialize each alternative,
re-evaluate the fixed objective, and archive the best model with its  solver and 
replay record.

\textbf{Interpret the status.}
Red identifies a failed minimum standard.
Yellow exhibits executable headroom.
Green records current-suite passage.
Inconclusive records incomplete mandatory execution.

\textbf{After Green-0.}
Continue the original trainer to test attainability with stronger challengers, producing an optimality staircase.
Every accepted stair lies below the loss that creates it.

\textbf{Attach only earned evidence.}
Add spectral or proof-bearing coverage for gap meaning,
paired brackets for representation meaning, and protected evidence
for task meaning.

\textbf{Record null-relative rarity separately.} This quantifies how unusual the observed certificate depth is under
the declared null. 
}

\end{minipage}

\vspace{5pt}


\vspace{5pt}

\glancebox{SeaGreen}{SoftGreen}{%
\fontsize{9.0}{10.2}\selectfont
\glanceragged

\textbf{\texttt{\MakeUppercase{Certificate ladder:}}}
Stronger conclusions require additional evidence; passage alone is not globality.
\par\vspace{10pt}

\noindent
\begin{minipage}[c]{0.16\linewidth}
\centering
\textbf{\textcolor{tucgray}{1.\ Complete\\ witness}}\par
\vspace{2pt}
$B<J_t$\par
\vspace{2pt}
{\fontsize{8.3}{9.2}\selectfont
Constructive\\ suboptimality evidence}
\end{minipage}%
\begin{minipage}[c]{0.04\linewidth}
\centering{\Large\color{DeepBlue}$\Rightarrow$}
\end{minipage}%
\begin{minipage}[c]{0.16\linewidth}
\centering
\textbf{\textcolor{SlateBlue}{2.\ Suite\\ passage}}\par
\vspace{2pt}
Current suite\\ passes\par
\vspace{2pt}
{\fontsize{8.3}{9.2}\selectfont
Named budget\\ and tolerance}
\end{minipage}%
\begin{minipage}[c]{0.04\linewidth}
\centering{\Large\color{DeepBlue}$\Rightarrow$}
\end{minipage}%
\begin{minipage}[c]{0.16\linewidth}
\centering
\textbf{\textcolor{SeaGreen}{3.\ Global-gap\\ bound}}\par
\vspace{2pt}
$+$ coverage\par
\vspace{2pt}
{\fontsize{8.3}{9.2}\selectfont
Challenge power,\\ spectral, or proof-bearing}
\end{minipage}%
\begin{minipage}[c]{0.04\linewidth}
\centering{\Large\color{DeepBlue}$\Rightarrow$}
\end{minipage}%
\begin{minipage}[c]{0.16\linewidth}
\centering
\textbf{\textcolor{tucgold}{4.\ Representation\\ bound}}\par
\vspace{2pt}
$+$ paired\\ certificates\par
\vspace{2pt}
{\fontsize{8.3}{9.2}\selectfont
Decoder under-use\\ vs.\ rep.\ deficit}
\end{minipage}%
\begin{minipage}[c]{0.04\linewidth}
\centering{\Large\color{DeepBlue}$\Rightarrow$}
\end{minipage}%
\begin{minipage}[c]{0.16\linewidth}
\centering
\textbf{\textcolor{tucorange}{5.\ Task\\ evidence}}\par
\vspace{2pt}
$+$ valid bridge\par
\vspace{2pt}
{\fontsize{8.3}{9.2}\selectfont
Protected or\\ transferred task meaning}
\end{minipage}

\vspace{6pt}

\par
}

\endgroup
\clearpage

\section{Introduction}
\label{sec:introduction}

A flat training curve can accompany a global solution, a poor local basin, a blockwise bottleneck, a dead representation, a numerical plateau, or the limit of one optimizer and schedule. Gradient and update norms describe the local behavior of that trainer; held-out metrics describe another sample or task. The optimization question that motivates this paper is constructive:
\begin{quote}
\emph{Can a declared, architecture-valid procedure construct a materially better complete model for the same objective?}
\end{quote}

We propose a principled framework to construct feasible challenges to the training under investigation. If the candidate attains a lower loss value $B<J(\theta_t)$ in the same certified class, then feasibility alone gives
\begin{equation}
J(\theta_t)-J^\star\ge J(\theta_t)-B.
\label{eq:intro-witness}
\end{equation}
Figure~\ref{fig:framework-overview} shows how the proposed framework turns this primitive into a monitoring and certification system.

\begin{figure}[t]
\centering
\includegraphics[width=.98\textwidth]{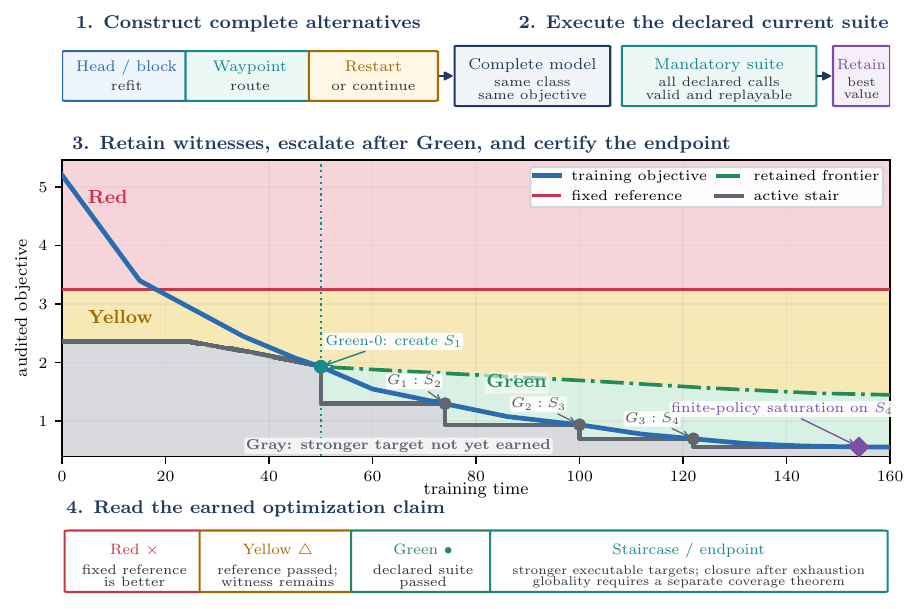}
\caption{The proposed framework turns complete executable alternatives into a retained evidence frontier. Red and Yellow preserve explicit lower-loss models. Current Green reports passage of the completed named suite and launches the stronger examination. Each later event creates a complete executable stair strictly below the triggering loss. The endpoint becomes globally meaningful only when a coverage theorem, exact solver, or proof-bearing lower floor connects the retained values to the complete certified class.}
\label{fig:framework-overview}
\end{figure}

\subsection{A constructive audit of a checkpoint}

The framework fixes one empirical optimization problem: model class, data, objective, regularization, precision, and every state needed for deterministic execution. A predeclared challenge suite then constructs complete alternative models through head or block refits, waypoint-guided routes, continuations, restarts, or proof-bearing subsolvers. Only the complete terminal model and its value enter the evidence record.

The online monitor answers a limited set of questions necessary for the trainer's \emph{certification}. Red means that a fixed feasible reference still beats the checkpoint. Yellow means that the reference has been passed but a lower-valued executable challenger remains. Green means that every mandatory current call completed and the checkpoint passed that named suite at its declared tolerance and budget. 
Passing Green opens a stronger examination: the original trainer may be continued toward a lower executable target to test attainability, or the challenger may be adopted and the current suite rerun to test actionability.


\subsection{The theorem spine and contributions}
\label{sec:contributions}

We develop the presentation around the questions in \cref{tab:result-spine}. The first four rows form the central optimization spine; representation and task evidence are downstream extensions.  The paper contributes four connected components:
\begin{enumerate}[leftmargin=1.65em,itemsep=4pt,topsep=3pt]
\item \textbf{Executable semantics and status.} Architecture-valid candidates, same-objective reevaluation, current-suite completion, retained evidence, and event-driven stronger targets define what is observed and what each status means.
\item \textbf{Finite passage and coverage.} Challenge power gives the sharp resource-indexed interpretation of suite passage. For squared loss, certified decrease operators turn current challenge executions into full, partial, normal-residual, and realized-residual gap bounds; E-optimal design selects complementary calls. Exact affine motifs, collective rank-deficient blocks, a stable-core ReLU theorem, and a primal--dual adapter bracket provide concrete coverage mechanisms.
\item \textbf{The globality boundary.} Challenge-closed optimality characterizes exhaustive endpoint families, while an explicit ReLU construction proves that infinitely many reached exact conditional head optima may still converge to a non-global point when one representation direction remains uncovered.
\item \textbf{Representation, statistical, and operational evidence.} Paired predictive certificates quantify representation insufficiency; finite-trial calibration controls misses of a frozen policy; a ResNet-18 study measures certificate tightness against known ground truth; and quantized-denoising experiments separate executable headroom, trainer attainability, structural coordination, and protected-task behavior.
\end{enumerate}

\begin{table}[t]
\centering
\small
\caption{The theorem spine. Each stronger conclusion requires the additional object named in the middle column.}
\label{tab:result-spine}
\begin{tabular}{>{\raggedright\arraybackslash}p{0.28\linewidth}>{\raggedright\arraybackslash}p{0.25\linewidth}>{\raggedright\arraybackslash}p{0.37\linewidth}}
\toprule
Question & Mathematical object & Conclusion earned \\
\midrule
Can the checkpoint be beaten? & complete executable challenger & a realized improvement lower-bounds empirical suboptimality \\
What does finite passage mean? & budgeted challenge power and its inverse & the largest gap compatible with passage by the named resource-qualified family \\
When does passage become near-global? & spectral, exact-affine, contraction, or proof-bearing coverage & a quantitative upper bound on the full-class empirical gap \\
Why is coverage necessary? & exact non-global ReLU staircase & repeated exact conditional progress can remain globally suboptimal \\
What is missing from the representation? & paired outer and representation-preserving brackets & decoder under-use is separated from representation insufficiency \\
What does the result mean for deployment? & population calibration, transfer theorem, or protected audit & a separately licensed population or task statement \\
\bottomrule
\end{tabular}
\end{table}

\subsection{Evidence and organization}

The experiments serve two different purposes. The ResNet-18 and exact-regression studies are controlled validations of theorem identities, scope, and numerical tightness in problems where ground truth is available. The quantized-denoising study tests the operational system: whether complete alternatives expose actionable headroom, whether the original trainer reaches stronger targets, whether adoption survives current-state recertification, and whether protected image quality changes. These are checkpoint-based empirical claims; population and intended-task conclusions remain separate evidence fields.

Sections~\ref{sec:framework}--\ref{sec:staircase} define the audit semantics and explain the challenge operations through a running graph-cut example. Sections~\ref{sec:challenge-power}--\ref{sec:closure-main} develop finite passage, coverage, necessity, and challenge-closed endpoints. Sections~\ref{sec:representation}--\ref{sec:coverage-plugins} develop representation, tracking, exact-solver, proof-bearing, population, and null-relative extensions. Section~\ref{sec:experiments} reports the controlled and neural studies, and Section~\ref{sec:protocol} gives the reporting, tolerance, and cost protocol. Full proofs and replay records appear in the appendices and artifact.

\section{Related Work and Positioning}
\label{sec:related-work}

\subsection{Optimization diagnostics, variable projection, and local training}

Classical optimization monitors objective descent, gradients, KKT residuals, line-search conditions, and duality gaps \citep{nesterov2004introductory,bottou2018optimization}. These quantities describe a chosen formulation and update process. Variable projection removes affine variables from separable nonlinear least squares \citep{golub1973variable}, and neural variants exploit exact or convex output-layer structure \citep{newman2021varpro,galashov2025closedform}. Block-coordinate, greedy layerwise, and local-error methods use structured subproblems as the training algorithm \citep{tseng2001bcd,belilovsky2019greedy,nokland2019local}.

Globally convergent trainers form an important boundary case rather than a contradiction to post-hoc challenge auditing. \citet{akiyama2025bcd} proves global convergence of a particular auxiliary-variable BCD procedure for fully connected squared-error regression under full-row-rank training data, a common hidden width, and strictly monotone activations whose slope is uniformly bounded away from zero. 
When these assumptions and the prescribed procedure match the exact certified problem, the theorem directly resolves the empirical optimization question, and an additional challenge audit is unnecessary for establishing globality. Outside such a theorem-backed regime, the distinction motivating this paper remains. A guarantee for one architecture--objective--data--algorithm tuple does not determine the status of an arbitrary realized checkpoint produced by an arbitrary trainer on a problem outside the theorem setting. Our proposed framework is solver-agnostic: heuristic, approximate, exact-conditional, and globally convergent solvers can all enter the challenge portfolio. 

Waypoint reconstruction is especially close to auxiliary-coordinate and local-target methods. The method of auxiliary coordinates introduces explicit intermediate states to remove deep nesting and alternates over local parameter and state subproblems \citep{carreira2014mac}; Akiyama's BCD construction likewise introduces auxiliary hidden representations and layer-local reconstruction losses, but uses them to define and analyze a trainer \citep{akiyama2025bcd}. The proposed framework does not claim local targets as a new training primitive. It uses such local objectives as \emph{challenge generators} for an independently realized checkpoint. The evidentiary object is always the complete assembled model reevaluated under the original objective; a local waypoint or auxiliary loss is never itself a certificate.

\subsection{Local search, error bounds, and resource-indexed challenge power}

Local search defines optimality relative to a neighborhood or move family \citep{johnson1988localsearch,kolda2003direct}. PL and error-bound conditions convert local residuals into global-gap and rate statements \citep{karimi2016pl}. KL analyses provide a complementary convergence language for alternating nonsmooth methods \citep{bolte2014palm}. Coordinate and subspace corrections quantify progress from structured directions \citep{nesterov2012coordinate,xuzikatanov2002}, while fusion frames and E-optimal design characterize collective subspace coverage \citep{casazza2008fusion,pukelsheim2006optimal}.

Challenge power places these ideas inside an executable audit. The move family is state dependent and includes solver policy, permitted information, complete materialization, and total budget. Its inverse states what finite passage can mean. It does not remove the hard coverage question: coverage must be established by a theorem, a current operator calculation, a proof-bearing floor, or another declared mechanism. The paper develops several such mechanisms.

\subsection{Neural geometry and exact conditional subproblems}

Wide-network analyses control Jacobian or neural-tangent spectra at initialization or within a neighborhood \citep{jacot2018ntk,allen2019convergence,oymak2019overparameterized,nguyen2021ntk,banerjee2023ntk,carvalho2024ntk,song2026ntk}. These results explain when a positive coverage coefficient can arise. The certificate here is checkpoint specific and execution weighted: it is built from the challenges actually run, their accepted steps, and their solver errors, and it reports uncovered current residual energy when the operator bank is rank deficient.

Selected neural classes admit global-optimality conditions, convex representations, exact reformulations, and relaxations \citep{haeffele2017global,bach2017convex,ergenpilanci2021geometry,pilanci2020convex,ergen2020convex,mishkin2022fast,kim2024convex,sahiner2022attention}. Deep-linear squared-loss structure is classical \citep{baldi1989linear}. In this paper, such results enter as exact or proof-bearing conditional challenges whose scope remains attached to the frozen components, regularizer, width, and data.

\subsection{Verification, repair, and algorithm portfolios}

Robustness verification and proof-of-learning certify prediction properties or computation histories \citep{anderson2025branching,jia2021proof,srivastava2024optimistic,mao2025ctbench}. Counterexample-guided synthesis and neural repair construct changes after a specification failure \citep{solarlezama2006sketching,boetius2023robustrepair,sotoudeh2021provable,tao2023architecture}. Algorithm-selection and variable-neighborhood methods allocate heterogeneous search effort \citep{rice1976algorithm,mladenovic1997vns}. The object certified here is different: the empirical optimization quality of one checkpoint under one fixed objective, together with the complete alternatives that establish failure and the named suite whose passage was observed.

\subsection{Representation probes, information, and control tasks}

Linear probes train independent decoders on frozen intermediate representations to diagnose what those representations make accessible \citep{alain2016probes}. Control tasks test whether probe performance reflects reusable representation structure rather than probe memorization \citep{hewitt2019control}. Blackwell comparison, statistical information, proper scoring, predictive $\mathcal V$-information, and decodable bottlenecks relate experiments, predictor classes, and Bayes risks \citep{blackwell1953,degroot1962,gneiting2007,reid2011information,williamson2024bridge,xu2020usable,dubois2020dib,banerjee2020deficiency}.

The representation-preserving challenge in this paper is probe-like, but it returns a complete model and an executable objective value. Its novelty is the paired certificate: an outer bracket and a frozen-representation bracket are combined so that decoder under-use is not misreported as information absent from the representation. Permutation selectivity is used as a diagnostic control, not as a substitute for the paired risk bounds.

\subsection{Statistical risk control and quantized training}

Learn-then-test and risk-control procedures calibrate frozen decision pipelines over declared populations \citep{bates2021rcps,angelopoulos2025ltt}. The finite-trial population result applies that discipline to heterogeneous challenge episodes. Its probability space is separate from deterministic pointwise coverage and from the randomness of the observed trainer.

Quantization-aware training includes binary networks, integer inference, learned clipping and step sizes, oscillation-aware optimization, and recent low-bit methods \citep{courbariaux2015binaryconnect,hubara2016binarized,jacob2018quantization,choi2018pact,esser2020lsq,nagel2022oscillation,jin2025parq,panferov2025quest}. The quantized-denoising experiments use this difficult regime to test diagnosis and intervention.

\subsection{What is inherited and what is new}

The individual optimization operations have important precedents. The contribution is their use inside one executable evidence system and the bridges that state what finite passage licenses. \Cref{tab:positioning} summarizes the relationship. 

This paper builds on our earlier data-aware training-monitoring program. In 
\citet{yeganegi2025dataaware}, we introduced training bounds and a color-coded monitoring system built from layerwise achievable reference solutions. \citet{eamaz2026trust} extended that construction to low-bit decoder-only transformers through layer peeling, local fits to intermediate representations, and multiple target assignments. Those studies established that structured alternatives can expose optimization inefficiencies hidden by an aggregate training curve. The present work turns that construction principle into a general executable-certificate theory. A route may select any subset of architecture-native graph cuts, from no internal waypoint to every available waypoint. More importantly, the framework distinguishes what a lower-valued complete model proves, what current-suite passage means, and what additional coverage is required for global-gap, representation, population, or task claims. The earlier layer-wise bound and transformer-peeling constructions are therefore concrete challenge generators within the broader hierarchy developed here.

\begin{table}[t]
\centering
\small
\caption{Relationship to neighboring methods. The middle column identifies the inherited idea; the right column identifies the additional role it plays in the proposed framework.}
\label{tab:positioning}
\begin{tabular}{>{\raggedright\arraybackslash}p{0.24\linewidth}>{\raggedright\arraybackslash}p{0.30\linewidth}>{\raggedright\arraybackslash}p{0.36\linewidth}}
\toprule
Neighboring line & Shared ingredient & Distinction in this paper \\
\midrule
Variable projection / head elimination & exact conditional optimization of affine variables & post-hoc complete candidate, same-objective replay, and a conditional-gap record \\
Auxiliary coordinates / local-target and layerwise training & intermediate targets and local subproblems & local objectives generate counterfactual routes; only the complete terminal model carries evidence \\
Local search and error bounds & move-family optimality and residual-to-gap conditions & state-dependent executable families, explicit resource budgets, and a passage-versus-globality distinction \\
Neural repair / counterexample synthesis & construct a modified model after detecting failure & optimize the same empirical objective and retain strict objective improvements as checkpoint evidence \\
Convex or exact neural subclasses & global or bracketed optimization in a restricted class & plug in as proof-bearing challenges with the class boundary recorded \\
Linear probes and control tasks & frozen-representation diagnosis and selectivity & pair inner and outer risk brackets to separate decoder under-use from representation insufficiency \\
\bottomrule
\end{tabular}
\end{table}

\section{The Proposed Framework}
\label{sec:framework}

\subsection{Problem identity: training, certification, and protected task}

A certificate begins by fixing the exact optimization problem. Let $D_{\mathrm{cert}}$ be the deterministic certification sample and let $\Theta=\Theta_{\mathrm{cert}}$ be the complete certified class. The certification objective is
\begin{equation}
J_{\mathrm{cert}}(\theta)
=
\frac{1}{n_{\mathrm{cert}}}
\sum_{(x_i,y_i)\in D_{\mathrm{cert}}}
\ell(f_\theta(x_i),y_i)
+
\Omega_{\mathrm{cert}}(\theta),
\qquad
J^\star=\inf_{\vartheta\in\Theta}J_{\mathrm{cert}}(\vartheta)\in\mathbb R.
\label{eq:objective}
\end{equation}
The class contains every state that affects deterministic execution: architecture, dimensions, parameter sharing, masks, routing, normalization mode and buffers, quantization, sparsity, constraints, and precision semantics. Choose a fixed feasible reference model $\theta_{\mathrm{ref}}\in\Theta$ and set $J_{\mathrm{ref}}:=J(\theta_{\mathrm{ref}})$. Throughout the paper $J$ denotes $J_{\mathrm{cert}}$ when no ambiguity remains.

Training and task quantities answer different questions. Table~\ref{tab:objective-roles} separates them.
\begin{table}[t]
\centering
\small
\caption{Three objective roles. Every executable inequality in the certificate uses the fixed certification objective. Training and protected-task quantities enter as separately named evidence.}
\label{tab:objective-roles}
\begin{tabular}{>{\raggedright\arraybackslash}p{0.17\linewidth}>{\raggedright\arraybackslash}p{0.32\linewidth}>{\raggedright\arraybackslash}p{0.39\linewidth}}
\toprule
Quantity & Role & Question answered \\
\midrule
$J_{\mathrm{train}}$ on $D_{\mathrm{train}}$ & objective and stochastic transformations used by the trainer & What update problem generated the checkpoint? \\
$J_{\mathrm{cert}}$ on $D_{\mathrm{cert}}$ & fixed complete objective used for every challenge and status comparison & Can the checkpoint be executably beaten for the audited empirical problem? \\
$L_{\mathrm{task}}$ on protected data or a population law & intended-use risk excluded from challenge construction & Does the optimization evidence transfer to the task of interest? \\
\bottomrule
\end{tabular}
\end{table}
When $J_{\mathrm{train}}=J_{\mathrm{cert}}$, the certificate audits the trainer's empirical problem directly. When they differ, the report identifies the fixed certification problem explicitly.

\subsection{Core notation and claim levels}

In this paper, a certificate is a replayable record supporting a precisely stated inequality or status for the fixed certification problem. The notation used throughout is collected in \cref{tab:core-notation}.

\begin{table}[t]
\centering
\small
\caption{Core notation. Current quantities are checkpoint specific; they must be recomputed before being transferred to another checkpoint.}
\label{tab:core-notation}
\begin{tabular}{>{\raggedright\arraybackslash}p{0.18\linewidth}>{\raggedright\arraybackslash}p{0.72\linewidth}}
\toprule
Symbol & Meaning \\
\midrule
$J_t=J(\theta_t)$ & value of the current checkpoint under the fixed certification objective \\
$J^\star$ & infimum of that objective over the complete certified class $\Theta$ \\
$J_{\mathrm{ref}}$ & value of a fixed feasible reference model \\
$B_{\mathrm{on}}(\theta_t)$ & best complete value returned by the mandatory current suite, including the fallback \\
$G_t$ & best executable value retained through time $t$ \\
$E_t\in\{0,1\}$ & flag indicating whether every mandatory current call completed validly \\
$S_j$ & active stronger target created after the preceding passage or hit event \\
$I_{\mathrm{suite}}(\theta_t)=J_t-B_{\mathrm{on}}(\theta_t)$ & largest improvement exposed by the mandatory current suite \\
\bottomrule
\end{tabular}
\end{table}

The claims form a ladder. A lower-valued complete model is a one-sided witness. Completed current execution can establish passage of a named suite. A coverage theorem or proof-bearing lower floor can turn that passage into a global-gap bound. Paired predictive certificates can then isolate representation insufficiency. Population and intended-task claims require an independent calibration or protected bridge. Each level retains the evidence supporting the levels below it.

\subsection{Executable challenges and one-sided evidence}

At checkpoint $t$, let $\mathcal I_t$ be the declared information available to a challenge. It may include the current model, optimizer state, stored graph-cut tensors, certification data, earlier retained models, and declared randomness. Future checkpoints and protected audit data remain outside $\mathcal I_t$ unless the protocol explicitly includes them.

\begin{definition}[Executable challenge]
\label{def:challenge}
An executable challenge $C$ maps permitted information and auxiliary randomness to a complete model
\[
\widehat\theta_{C,t}=C(\mathcal I_t,\omega)\in\Theta.
\]
Its realized value is $B_{C,t}=J(\widehat\theta_{C,t})$, recomputed through the complete forward pass and the same certification objective. External replay requires either the candidate state or a deterministic regeneration recipe containing the released starting state, code, configuration, randomness, and environment.
\end{definition}

\begin{proposition}[One-sided executable certificate]
\label{prop:witness}
Every finite sound challenge satisfies
\begin{equation}
J(\theta_t)-J^\star
\ge
J(\theta_t)-B_{C,t}.
\label{eq:witness}
\end{equation}
A positive realized gap is therefore a quantitative certificate of empirical suboptimality.
\end{proposition}
\begin{proof}
The materialized challenger belongs to $\Theta$, so $J^\star\le B_{C,t}$. Subtract from $J(\theta_t)$.
\end{proof}

Soundness and strength are separate. Complete materialization validates the inequality. Construction, solver status, budget, and coverage determine whether a challenge is likely to expose a meaningful defect and how much passage means.

\paragraph{Canonical example: an exactly solved or tolerance-certified head refit.}
Let $f_{\phi,w}=g_w\circ h_\phi$ and freeze the current representation $h_{\phi_t}$. Define
\[
B_{\mathrm{head}}^\star(\phi_t)=\inf_wJ(\phi_t,w).
\]
When compactness, coercivity, or an explicit regularizer gives attainment, an exact solver returns $\widehat w\in\argmin_wJ(\phi_t,w)$. Otherwise it returns a feasible $\varepsilon_{\mathrm{head}}$-optimal decoder satisfying
\[
J(\phi_t,\widehat w)\le B_{\mathrm{head}}^\star(\phi_t)+\varepsilon_{\mathrm{head}}.
\]
The complete candidate $\widehat\theta=(\phi_t,\widehat w)$ is evaluated by the original model from the original input. Its realized improvement always lower-bounds the full-class global gap; with an exact solve it is the exact conditional decoder gap, and with a tolerance-certified solve it approximates that gap to the declared additive tolerance. Later sections combine this representation-preserving value with an outer certificate to separate decoder under-use from representation insufficiency.

\subsection{Typed graph cuts}

Modern architectures are naturally represented as computation graphs. Choose ordered cuts
\[
Z_0,Z_1,\ldots,Z_M,
\]
where each $Z_r$ contains every live tensor and state required to continue exact execution beyond the cut. A cut can include residual streams, feature maps, masks, positions, recurrent state, routing state, normalization buffers, skip tensors, and caches. Parameters shared across graph regions remain shared inside every candidate.

\begin{table}[t]
\centering
\small
\caption{Representative architecture-native cuts. A sound cut stores every component of the continuation state required for exact downstream execution.}
\label{tab:graph-cuts}
\begin{tabular}{>{\raggedright\arraybackslash}p{0.23\linewidth}>{\raggedright\arraybackslash}p{0.29\linewidth}>{\raggedright\arraybackslash}p{0.39\linewidth}}
\toprule
Architecture & Natural cut & State retained for exact continuation \\
\midrule
MLP & hidden block boundary & activation tensor and normalization state \\
CNN / ResNet & stage or residual-group boundary & feature map, skip stream, buffers \\
Transformer & residual stream after a block & token states, masks, positions, cache policy \\
Graph network & message-passing round & node, edge, and global states \\
RNN / state-space model & recurrent stage or time window & hidden, recurrent, and auxiliary state \\
U-Net & encoder/decoder stage & stage activation and all live skip tensors \\
Mixture of experts & router--expert group & routing state, expert outputs, capacity state \\
Deep unfolding & algorithmic iteration & current iterate and algorithmic auxiliaries \\
\bottomrule
\end{tabular}
\end{table}

\paragraph{Running three-segment example.}
Consider
\[
x\xrightarrow{F_1}Z_1\xrightarrow{F_2}Z_2\xrightarrow{F_3}\widehat y
\]
at checkpoint $t$, with stored states $Z_1^{(t)}$ and $Z_2^{(t)}$. The operations below answer different questions:
\begin{itemize}[leftmargin=1.55em,itemsep=2pt,topsep=3pt]
\item a \emph{waypoint reconstruction} asks whether a perturbed copy of $F_2$ can reproduce the stored map $Z_1^{(t)}\mapsto Z_2^{(t)}$;
\item a \emph{frozen-context block refit} changes $F_2$ while evaluating the complete model from $x$ under the original objective;
\item a \emph{sequential route} rebuilds several segments and feeds the newly produced state $\widehat Z_1$, not the stored state $Z_1^{(t)}$, into the next rebuilt segment;
\item \emph{route-and-release} starts from the assembled route, unfreezes the permitted parameters, and resumes end-to-end optimization of the original objective.
\end{itemize}
Thus waypoint reconstruction and block refitting are local challenge objectives; sequential routing is a rule for composing several local proposals; route-and-release is an optional final end-to-end refinement. They are not four competing definitions of the same operation.

\subsection{Local challenge objectives}

A challenge configuration records the selected outer cuts, optimized blocks, within-segment decomposition and target-assignment rule, local and release budgets, solver portfolio, restart rule, waypoint metric, and cost origin.

\paragraph{Waypoint reconstruction.}
A segment between cuts $a$ and $b$ is reinitialized or perturbed and optimized to reproduce the stored checkpoint state:
\begin{equation}
\widehat\phi_{a:b}
\in
\operatorname{Opt}_{A,B}
\,d_b\!\left(F_{a:b}(Z_a^{(t)};\phi),Z_b^{(t)}\right).
\label{eq:waypoint-template}
\end{equation}
The stored input and target make this a local reconstruction problem. The restart or perturbation is part of the declaration; initializing at the current weights while targeting the current output would be a trivial identity exercise. Waypoint reconstruction tests whether an alternative local route can reproduce the checkpoint's interface state more effectively or with different conditioning. Its local loss is a construction device, not a certificate value.

\paragraph{Frozen-context block refit.}
All parameters outside block $r$ are fixed and the complete end-to-end objective is optimized:
\begin{equation}
\widehat\phi_r
\in
\operatorname{Opt}_{A,B}
J(\theta_{-r,t},\phi_r).
\label{eq:block-template}
\end{equation}
This directly asks whether block $r$ is underoptimized in its current upstream and downstream context. Exact heads, convex internal blocks, coordinate updates, and bounded-gap subsolvers fit this template.

\subsection{Multiple-waypoint routes form a nested challenge hierarchy}
\label{sec:multi-waypoint-hierarchy}

Waypoint reconstruction is not restricted to one intermediate state. For ordered cuts $Z_0,\ldots,Z_M$, let
\[
\mathcal I=\{1,\ldots,M-1\}
\]
index the available internal waypoints. Any ordered subset
\[
\mathcal S=\{i_1<\cdots<i_m\}\subseteq\mathcal I
\]
defines $m$ fixed waypoint targets. Together with the input and output cuts, the ordered sequence
\[
0=i_0<i_1<\cdots<i_m<i_{m+1}=M
\]
partitions the model into $m+1$ segments. Thus a network with $M-1$ internal cuts admits routes with zero, one, two, and up to all $M-1$ waypoints. At exactly $m$ waypoints there are $\binom{M-1}{m}$ possible location sets before accounting for solver, budget, restart, and composition choices.

In the running three-segment example, the internal set is $\mathcal I=\{1,2\}$. The admissible waypoint sets are
\[
\varnothing,
\qquad
\{1\},\ \{2\},
\qquad
\{1,2\}.
\]
The two singleton choices impose different intermediate targets, while $\{1,2\}$ fixes both stored interfaces. This is the combinatorial source of alternative training routes: one checkpoint supplies many architecture-valid counterfactual decompositions of the same end-to-end problem.

\paragraph{Outer waypoint choice and within-segment solution are separate.}
Selecting sparse checkpoint waypoints does \emph{not} force the span between them to be optimized as one deep subnetwork. The outer set $\mathcal S$ specifies which stored checkpoint states are fixed as route anchors; independently, the challenge declares how each induced segment is solved. If two selected anchors occur at cuts $a<b$, one challenger may optimize the whole map $F_{a:b}$ jointly through \Cref{eq:waypoint-template}. Another may peel the same span into layer- or block-local problems, introducing additional internal targets solely for the segment solver.

For example, if the selected checkpoint anchors are the outputs of layers 2 and 7, both of the following are admissible constructions:
\[
Z_2^{(t)}
\xrightarrow[\text{joint segment solve}]{F_{3:7}}
Z_7^{(t)},
\qquad\text{or}\qquad
Z_2^{(t)}\to\widehat Z_3\to\widehat Z_4\to\widehat Z_5\to\widehat Z_6\to Z_7^{(t)}.
\]
In the second route, $\widehat Z_3,\ldots,\widehat Z_6$ are states generated while solving the 2-to-7 span; they are not additional \emph{outer} checkpoint waypoints. The local targets used to fit those layers are themselves part of the declared solver. When interfaces are compatible, a solver can direct successive local layers toward the segment's terminal anchor; a checkpoint-peeling solver can instead use stored intermediate representations, and alternative target assignments or permutations give further challengers \citep{yeganegi2025dataaware,eamaz2026trust}. Auxiliary-variable BCD provides another within-segment decomposition; when the induced problem satisfies the assumptions of \citet{akiyama2025bcd}, its convergence theorem can make that challenger proof-backed. Thus, \emph{waypoint density} and \emph{within-segment decomposition} are independent axes. A sparse outer route can still be solved layer by layer, while a dense outer route may use a joint or exact solver inside each span. All such alternatives must ultimately be assembled into complete models and re-evaluated under the original objective.

Let $\Lambda_{\le m}$ denote the declared family of route configurations using at most $m$ outer waypoints and the declared within-segment solver/target-assignment portfolio, with all other resource restrictions held fixed. Then
\[
\Lambda_{\le0}\subseteq\Lambda_{\le1}\subseteq\cdots\subseteq\Lambda_{\le M-1}.
\]
If $B_{\le m,t}$ is the best complete-model objective among the executed and retained candidates in $\Lambda_{\le m}$, nestedness gives
\[
B_{\le m+1,t}\le B_{\le m,t}.
\]
This monotonicity belongs to the \emph{best retained family value}; an individual route can improve or worsen when another waypoint is imposed because local compatibility and downstream amplification change.

When $\Lambda_{\le m}$ is the named mandatory current suite, completed passage is an $m$-waypoint-qualified statement for the declared within-segment solver portfolio. A natural post-passage policy may next activate $\Lambda_{\le m+1}$ and retain a lower complete-model value as the next staircase target. Waypoint count is only one axis of strength: later levels may instead enlarge the within-segment solver or target-assignment portfolio---for example, adding peeling to a joint span solve---increase budget, switch from stored-state to sequential composition, or append route-and-release. In every case, the evidentiary rule is unchanged: only the materialized terminal model and its reevaluated original objective determine Green, a retained witness, or a stronger stair.

Our constructions in earlier work are special cases of this route principle \citep{yeganegi2025dataaware,eamaz2026trust}. The present formulation separates their two roles explicitly: selected checkpoint outputs define the outer route, while intermediate-target permutations provide within-segment construction rules. It then adds the full subset hierarchy, complete-model materialization, current-suite passage, and coverage-qualified interpretation. Exhaustive enumeration is not required: a declared search or dynamic-programming policy may select promising subsets, but the resulting claim remains relative to the routes and within-segment solvers actually executed. Appendix~\ref{app:route-resource} develops route selection and the nested retained portfolio formally.

\subsection{Composing local challenges and releasing the route}

Several local proposals can be assembled into one complete candidate. If each segment is optimized against its stored checkpoint input, the local problems are more separable, but the assembled route may encounter a distribution shift because an upstream replacement changes the state seen downstream.

\paragraph{Sequential route.}
A sequential route addresses that mismatch by propagating each newly constructed output into the next segment. In the running example, the rebuilt $F_2$ receives $\widehat Z_1=\widehat F_1(x)$ rather than $Z_1^{(t)}$. The route is less separable, but its local training states better match those encountered by the assembled candidate.

\paragraph{Route-and-release.}
After the route is assembled, all permitted parameters may be released and the original end-to-end objective optimized for a declared budget:
\begin{equation}
\widehat\theta_{\mathrm{route}}
\longmapsto
\operatorname{Opt}_{A,B_{\mathrm{rel}}}J(\theta).
\label{eq:route-release}
\end{equation}
This final phase tests whether the assembled route provides a better basin or initialization for the same end-to-end optimizer. Head refits, exact blocks, continuations, restarts, and solver-diverse proposals can be used independently or inside either composition rule. In every case, only the materialized complete terminal model is evaluated as evidence.

\subsection{Why intermediate targets can change conditioning}

Waypoint challenges are useful when they replace a long multiplicative credit-assignment path by shorter supervised subproblems. The following exact anchor isolates this conditioning mechanism for a disclosed product model.

\begin{theorem}[Waypoint advantage in a deep product model]
\label{thm:waypoint-advantage}
Let $K\ge3$ and
\[
f_w(1)=\prod_{r=1}^{K}w_r,
\qquad
J(w)=\frac12\left(\prod_{r=1}^{K}w_r-1\right)^2.
\]
Initialize end-to-end gradient flow at the balanced point $w_r(0)=\alpha\in(0,1/2]$. The balanced trajectory is preserved. If $T_{2\alpha}$ is the first time at which every coordinate reaches $2\alpha$, then
\begin{equation}
T_{2\alpha}
\ge
\frac{1-2^{2-K}}{K-2}\,\alpha^{2-K}.
\label{eq:waypoint-flow-time}
\end{equation}
Under a matched budget of $K$ scalar derivative evaluations, one unit-step full-gradient update satisfies
\begin{equation}
J(w^{(1)})
\ge
\frac12\left[1-\left(\frac32\alpha\right)^K\right]^2,
\label{eq:waypoint-one-step}
\end{equation}
whereas $K$ one-coordinate waypoint updates toward the disclosed ladder $Z_0=\cdots=Z_K=1$ attain $J=0$ exactly.
\end{theorem}

The theorem identifies a conditioning mechanism: the end-to-end derivative contains the product of every other small factor, while the local waypoint objectives remove that multiplicative coupling. The ladder is favorable and uses the target; the general certificate continues to depend on the complete materialized terminal value. The proof and boundary calculation appear in Appendix~\ref{app:waypoint-mechanism}. \Cref{fig:waypoint-advantage} visualizes the resulting depth-dependent separation at matched scalar-derivative budget.

\begin{figure}[t]
\centering
\includegraphics[width=\linewidth]{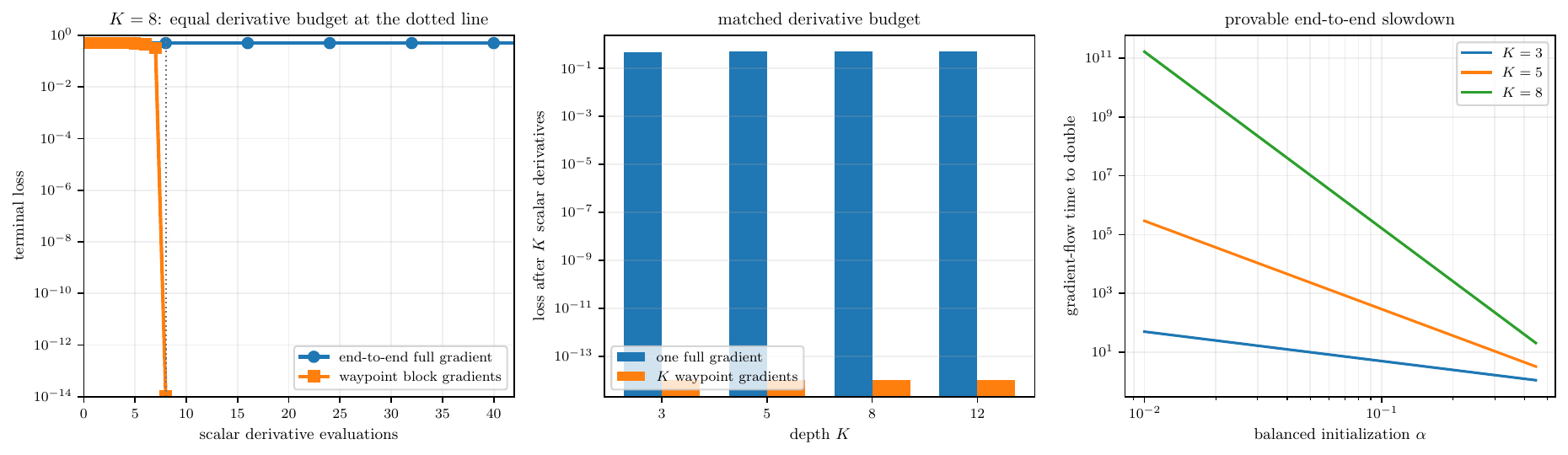}
\caption{Intermediate targets can remove depth-induced multiplicative flatness. At equal scalar-derivative budget, the disclosed waypoint ladder solves the product model while an end-to-end update barely moves; the separation grows with depth, and the proved gradient-flow time required merely to double a small balanced initialization scales as $\alpha^{2-K}$. The figure establishes the conditioning mechanism that motivates waypoint challengers; the complete terminal objective supplies the certificate for every assembled route.}
\label{fig:waypoint-advantage}
\end{figure}

\subsection{Terminal materialization is the certified value}

Local waypoint losses guide construction, while the complete terminal objective supplies the certificate value. An improved segment changes the inputs received by downstream segments, so local orderings can reverse after composition. For an assembled candidate $\widehat\theta_{t,\lambda}$, the certificate records only
\begin{equation}
B_{t,\lambda}=J(\widehat\theta_{t,\lambda})
\label{eq:terminal-value}
\end{equation}
after execution from the original model input.

The route-analysis appendix quantifies composition error. If segment $r$ is $L_r$-Lipschitz in its input and has local waypoint residual $\delta_r$, then
\begin{equation}
\|\widehat Z_R-Z_R\|
\le
\sum_{r=1}^{R}\delta_r\prod_{j=r+1}^{R}L_j.
\label{eq:route-propagation-preview}
\end{equation}
The bound guides route selection and budget allocation; the complete terminal objective remains the evidentiary output. Appendix~\ref{app:waypoint-mechanism} gives the fixed-waypoint separability theorem, the propagation proof, a finite-budget guarantee, a water-filling allocation rule, and an $O(M^2m)$ dynamic program for selecting up to $m$ segments among $M$ ordered cuts.

\subsection{Challenge records, solver status, and cost}

Every accepted candidate stores the exact problem identity, permitted information, candidate initialization, optimized parameters, optimizer and budget, restart policy, solver residual or primal--dual gap, complete terminal value, numerical and materiality tolerances, effective-rank controls, marginal correction cost, total-from-scratch cost, and replay artifact. A hash authenticates a state; replay additionally requires the state or a deterministic regeneration path.

The cost origin determines the claim. A checkpoint-dependent correction can be inexpensive marginally while inheriting the full cost of producing the checkpoint representation. Marginal cost measures audit burden and repair actionability. Total-from-scratch cost is required for Pareto or compute-efficiency comparisons. Nested portfolio values, resource-qualified Pareto statements, power profiles, and symmetry-safe challenge construction are developed in Appendix~\ref{app:route-resource}.

\subsection{Mandatory current execution, retention, and status}
\label{sec:current-status}

Partition the current suite into mandatory calls $\mathcal C_{\mathrm{mand}}$ and optional calls. Include the identity or a predeclared no-improvement fallback so that the best current value never exceeds the checkpoint value. Let $E_t=1$ exactly when every mandatory current call returns a valid finite replayable terminal result under its declared protocol. Let $B_{\mathrm{on}}(\theta_t)$ be the smallest value among completed current candidates, including the fallback. With the feasible reference value $J_{\mathrm{ref}}$, define
\begin{equation}
G_{-1}=J_{\mathrm{ref}},
\qquad
G_t=\min\{G_{t-1},B_{\mathrm{on}}(\theta_t)\}.
\label{eq:frontier}
\end{equation}
The attaining model is archived with the value.

\begin{lemma}[Monotone retained frontier]
\label{thm:frontier}
For every realized adaptive challenge history,
\[
J^\star\le G_t\le J_{\mathrm{ref}},
\qquad
G_{t+1}\le G_t.
\]
\end{lemma}
\begin{proof}
Each value entering the minimum is attained by a feasible stored model and is therefore at least $J^\star$. A running minimum cannot increase.
\end{proof}

Given display/materiality tolerance $\tau_G\ge0$, positive status is issued only when $E_t=1$:
\[
\redstatus:\ J_t>J_{\mathrm{ref}}+\tau_G,
\qquad
\yellow:\ G_t+\tau_G<J_t\le J_{\mathrm{ref}}+\tau_G,
\qquad
\green:\ J_t\le G_t+\tau_G.
\]
When $E_t=0$, the current report is \emph{Inconclusive}; all archived strict gaps remain valid. Because $G_t\le B_{\mathrm{on}}(\theta_t)$, current Green implies
\begin{equation}
I_{\mathrm{suite}}(\theta_t)
:=J(\theta_t)-B_{\mathrm{on}}(\theta_t)
\le\tau_G.
\label{eq:current-passage}
\end{equation}
This current-state inequality is the input to the power and spectral coverage theorems.

\paragraph{Evidence levels.}
A witness certifies constructive failure. Current Green certifies passage of the named mandatory suite. Coverage or a lower floor turns that passage into a global-gap statement. Paired nested certificates turn predictive optimization control into representation control. Protected-task claims enter through a held-out audit or transfer theorem. The certificate records these levels separately and reports the highest active one.

\section{The Event-Driven Stronger Examination}
\label{sec:staircase}

Current Green closes the inexpensive online examination and opens a stronger one. Define the first Green event
\begin{equation}
\tau_0
=
\inf\{t:E_t=1,\ J_t\le G_t+\tau_G\}.
\label{eq:green-zero}
\end{equation}
Before $\tau_0$, no post-Green stair is active. At $\tau_0$, the stronger policy constructs a complete candidate with proposed value $\widetilde S_1$. It is accepted only when
\begin{equation}
\widetilde S_1
<
\min\{G_{\tau_0},J_{\tau_0}\}
-
\tau_{\mathrm{acc}}.
\label{eq:first-stair}
\end{equation}
Set $S_1=\widetilde S_1$ when accepted. For $j\ge1$, define the next hit time
\begin{equation}
\tau_j
=
\inf\{t>\tau_{j-1}:J_t\le S_j+\tau_{\mathrm{hit}}\}.
\label{eq:stair-hit}
\end{equation}
At $\tau_j$, accept the next candidate only when
\begin{equation}
\widetilde S_{j+1}
<
\min\{S_j,J_{\tau_j}\}
-
\tau_{\mathrm{acc}}.
\label{eq:next-stair}
\end{equation}
Thus every active stair is a complete feasible value below the loss that creates it. Numerical uncertainty, materiality, solver error, hit tolerance, and acceptance margin remain distinct fields in the record. \Cref{fig:monitoring-story} shows the exact event timing and the roles of the retained frontier and active stair.

\begin{figure}[t]
\centering
\includegraphics[width=\linewidth]{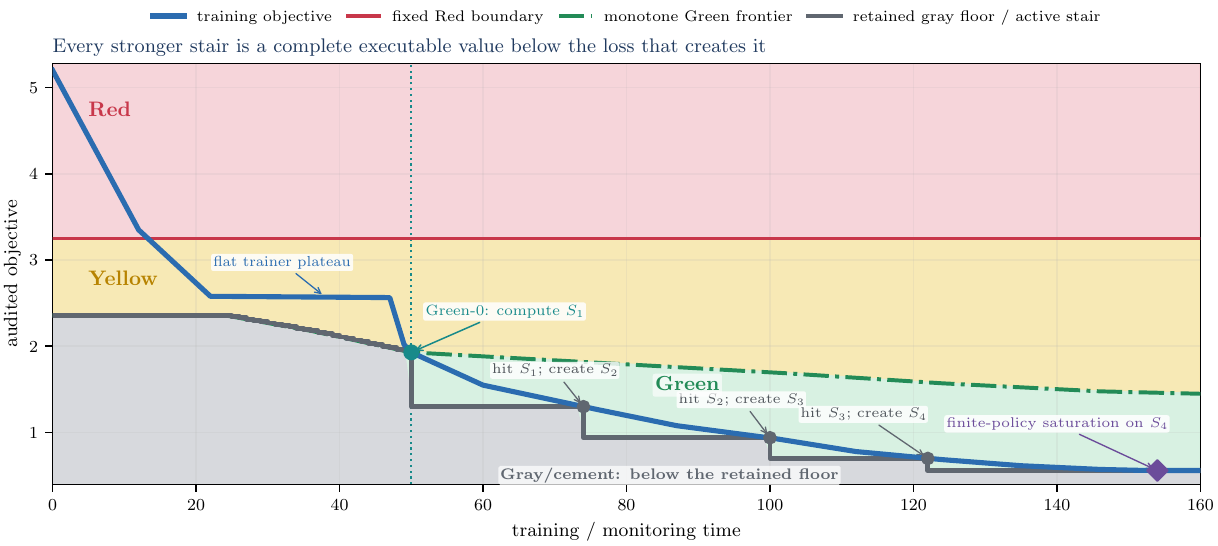}
\caption{The event-driven stronger examination. At Green-0, the stronger policy creates $S_1$ strictly below the triggering loss; every later stair is created only after the previous stair is reached. Before Green-0, the retained Gray ceiling equals the Green frontier, so the entire region below it is displayed as an unearned floor. The final diamond records finite-policy saturation after the loss lands on the last retained stair.}
\label{fig:monitoring-story}
\end{figure}

\subsection{Trainer-tracking and intervention branches}

The two post-Green branches answer different questions.

\paragraph{Trainer-tracking branch.}
The checkpoint is left unchanged and the declared trainer, or an explicitly labeled shadow continuation, is run toward the active stair. Reaching $S_j$ demonstrates attainability for that trainer and starting state. Convergence above $S_j$ exposes a trainer-specific barrier relative to an already feasible model.

\paragraph{Intervention branch.}
The challenger is adopted, and the complete current suite is rerun at the adopted state before current Green is issued again. The resulting intervention or adoption-closure path demonstrates actionability and renewed current-suite passage by the modified model.

\paragraph{Retrospective use of a later witness.}
If a complete candidate with value $S$ is discovered at time $t_0$ and the model class, data, and objective remain unchanged, then $J^\star\le S$ and
\[
J(\theta_t)-J^\star\ge J(\theta_t)-S
\]
for every earlier checkpoint. The archived model may therefore be shown as a retrospective comparator, provided its discovery time remains visible; it does not retroactively change the current status that was issued at the earlier time.

The staircase is a controlled experiment between challenge strength and trainer attainability. A lower stair proves that a better model exists. Reaching it proves that the trainer can attain that value. A coverage theorem determines how close the stair lies to the global value. A protected audit determines whether the intervention serves the intended task.

\section{Challenge Power: The Meaning of Finite Passage}
\label{sec:challenge-power}

A suite can pass for two very different reasons: the checkpoint may be strong, or the suite may be weak. Challenge power separates these explanations by asking how much improvement the declared family must expose at each level of true suboptimality. The answer is necessarily relative to a state region, budget, and target severity. Let $\mathcal R$ be a nonempty compact trainer-relevant region of the fixed problem. A challenge path $\pi$ from $x\in\mathcal R$ has total declared cost $c(\pi)$ and endpoint $T_\pi(x)$. Define
\begin{equation}
b_B(x)=\inf_{\pi:c(\pi)\le B}J(T_\pi(x)),
\qquad
I_x(B)=J(x)-b_B(x),
\qquad
\Delta(x)=J(x)-J^\star.
\label{eq:budgeted-improvement}
\end{equation}
Identity feasibility gives $0\le I_x(B)\le\Delta(x)$ even when the best endpoint is approached but not attained.

\begin{definition}[Budgeted challenge power and undetected-gap envelope]
\label{def:power}
For $B,s,\tau\ge0$, define the extended-real quantities
\begin{align}
\Psi(B,s)
&=\inf_{x\in\mathcal R:\,\Delta(x)\ge s} I_x(B),
\label{eq:psi}\\
E(B,\tau)
&=\sup_{x\in\mathcal R:\,I_x(B)\le\tau}\Delta(x).
\label{eq:envelope}
\end{align}
The conventions $\inf\varnothing=+\infty$ and $\sup\varnothing=-\infty$ are used; operationally, an empty passage set means that no checkpoint passes.
\end{definition}

$\Psi(B,s)$ is the improvement guaranteed at every checkpoint at least $s$ suboptimal. $E(B,\tau)$ is the largest global gap compatible with passage at tolerance $\tau$.

\begin{theorem}[Exact separation--certification inverse]
\label{thm:power-inverse}
Assume $\Delta$ and $I_{\cdot}(B)$ are continuous on compact $\mathcal R$, and endpoint families are nested in $B$. Then $\Psi$ is nondecreasing in $B$ and $s$, while $E$ is nonincreasing in $B$ and nondecreasing in $\tau$. Moreover,
\begin{equation}
\boxed{\Psi(B,s)\le\tau\quad\Longleftrightarrow\quad s\le E(B,\tau),}
\label{eq:inverse}
\end{equation}
so
\begin{equation}
\boxed{E(B,\tau)=\sup\{s\ge0:\Psi(B,s)\le\tau\}.}
\label{eq:geninverse}
\end{equation}
For every desired gap $\varepsilon$,
\begin{equation}
E(B,\tau)<\varepsilon
\quad\Longleftrightarrow\quad
\Psi(B,\varepsilon)>\tau.
\label{eq:strict-inverse}
\end{equation}
\end{theorem}
\begin{proof}
Compactness gives attainment. The event $\Psi(B,s)\le\tau$ holds exactly when there exists a passing state with gap at least $s$, which is exactly $s\le E(B,\tau)$. The remaining identities follow immediately; monotonicity follows from nesting of budget and severity level sets.
\end{proof}

\begin{corollary}[Certification-budget law and solver error]
\label{cor:budget-law}
The smallest budget certifying strict gap below $\varepsilon$ after passage at tolerance $\tau$ is
\[
\inf\{B:E(B,\tau)<\varepsilon\}
=
\inf\{B:\Psi(B,\varepsilon)>\tau\}.
\]
If the executed challenge is at most $\bar\epsilon_C$ above the best budget-$B$ value, observed improvement $\widehat I_x(B)\le\tau$ licenses
\begin{equation}
\Delta(x)\le E(B,\tau+\bar\epsilon_C),
\label{eq:approx-passage}
\end{equation}
not $E(B,\tau)$.
\end{corollary}

\Cref{fig:power-master} visualizes the sharp resource-indexed inverse, the resulting certification frontier, and representative rate envelopes on exact synthetic calibrations.

\begin{figure}[t]
\centering
\includegraphics[width=\linewidth]{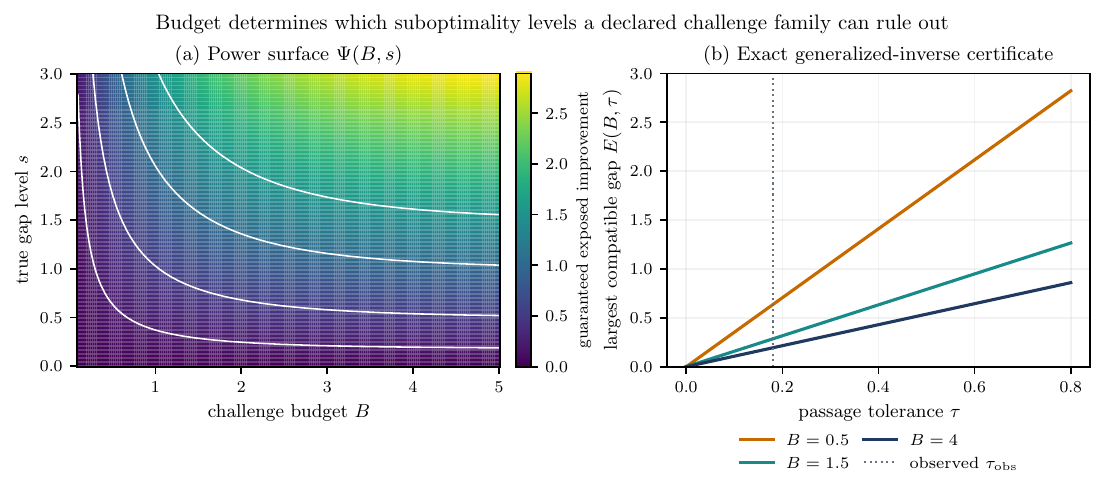}
\caption{Challenge power turns a budgeted suite into a quantitative certificate. The left panel shows the guaranteed improvement against each gap severity and budget. The right panel shows the exact generalized-inverse envelope: after passage at tolerance $\tau$, it is the largest gap still compatible with the declared suite. These synthetic surfaces visualize the theorem and the role of resource qualification.}
\label{fig:power-master}
\end{figure}

\subsection{Static power versus multistep dynamics}

The scalar surface exactly characterizes one-shot passage. Exact repeated intervention additionally depends on the endpoint transition map. Two systems can have identical $\Psi$ and $E$ at every budget while sending the same named state to different endpoints and requiring two versus eleven calls to reach the optimum. The transition-complete object is the Bellman endpoint operator
\[
(\mathsf T_BV)(x)=\inf_{y\in\mathcal C_B(x)}V(y).
\]
The formal counterexample and operator properties appear in Appendix~\ref{app:bellman}. The Bellman endpoint operator supplies the transition-level description required for exact adaptive challenge dynamics.

\paragraph{What this section establishes.}
Challenge power is an exact language for resource-qualified passage: it states the best improvement that a declared family must expose at a given gap severity and the largest gap still compatible with an observed pass. It does not by itself compute the unknown global gap. The next section supplies concrete coverage mechanisms that lower-bound this power from current neural geometry or exact conditional structure.

\section{Spectral Coverage for Current Neural Challenges}
\label{sec:spectral}

Challenge power states what must be shown; spectral coverage supplies a current, checkable route for showing it in empirical squared loss. The hard step is to establish that the executed block family covers every residual direction relevant to the claimed gap. Let
\begin{equation}
e(\theta)=\vecop(F_\theta-Y)\in\R^{nq},
\qquad
J(\theta)=\frac1{2n}\norm{e(\theta)}^2,
\label{eq:squared-loss}
\end{equation}
and partition challenged parameters into blocks $\theta_1,\ldots,\theta_m$. Let $G_r(\theta)=\partial e(\theta)/\partial\theta_r$ with dimensions $(nq)\times d_r$.

\begin{definition}[Certified decrease operator]
\label{def:decrease-operator}
A completed current block challenge has a certified decrease record $(Q_r,\epsilon_r)$ if $Q_r\succeq0$, $\epsilon_r\ge0$, and its completely materialized and reevaluated candidate value $\widehat B_r$ satisfies
\begin{equation}
J(\theta)-\widehat B_r(\theta)
\ge
\frac1{2n}e(\theta)^\top Q_r(\theta)e(\theta)-\epsilon_r.
\label{eq:decrease-record}
\end{equation}
The record is indexed by the current checkpoint and the exact complete objective. Transfer to changed data, class, regularizer, buffers, or state requires an explicit equivalence or regional-stability argument.
\end{definition}

Three canonical constructions are useful. If the block gradient is $L_r$-Lipschitz and the executable candidate takes the step $-\nabla_rJ/L_r$, then
\[
Q_r=\frac{G_rG_r^\top}{nL_r}.
\]
If a verified Armijo step of length $\eta_r$ satisfies
$J(\theta^+)\le J(\theta)-\sigma\eta_r\|\nabla_rJ(\theta)\|^2$, then
\[
Q_r=\frac{2\sigma\eta_r}{n}G_rG_r^\top.
\]
For an exact affine least-squares block, $Q_r=P_{\range(G_r)}$. In all cases the complete candidate value is the universal certificate; $Q_r$ records a proved portion of its strength.

\begin{theorem}[Certified spectral challenge coverage]
\label{thm:spectral}
For a completed current suite, let $\pi\in\Delta_m$ and define
\[
Q_\pi=\sum_{r=1}^m\pi_rQ_r,
\qquad
\bar\epsilon_\pi=\sum_{r=1}^m\pi_r\epsilon_r,
\qquad
I_{\mathrm{blk}}=\max_r[J-\widehat B_r].
\]
Then
\begin{equation}
I_{\mathrm{blk}}
\ge
\frac1{2n}e^\top Q_\pi e-\bar\epsilon_\pi.
\label{eq:spectral-average}
\end{equation}
If $Q_\pi\succeq cI_{nq}$ for $c>0$, then
\begin{equation}
J(\theta)\le\frac{I_{\mathrm{blk}}(\theta)+\bar\epsilon_\pi}{c},
\qquad
\boxed{J(\theta)-J^\star\le\frac{I_{\mathrm{blk}}(\theta)+\bar\epsilon_\pi}{c}.}
\label{eq:spectral-gap}
\end{equation}
This conclusion uses the nonnegative loss floor and therefore applies whether or not interpolation is attainable. When $J^\star>0$, the normal-residual refinement below can produce a substantially sharper excess-gap certificate.

More generally, let $U_\pi=\range(Q_\pi)$ and suppose $\lambda_{\min}^+(Q_\pi)\ge c$. For every valid lower floor $L\le J^\star$,
\begin{equation}
J(\theta)-J^\star
\le
\frac{I_{\mathrm{blk}}+\bar\epsilon_\pi}{c}
+
\left[\frac{\norm{P_{U_\pi^\perp}e}^2}{2n}-L\right]_+.
\label{eq:nullspace-gap}
\end{equation}
\end{theorem}
\begin{proof}
A maximum dominates every convex average, so \cref{eq:spectral-average} follows from \cref{eq:decrease-record}. Under full coverage,
$e^\top Q_\pi e\ge c\|e\|^2=2ncJ(\theta)$; rearrange and then use $J^\star\ge0$. For partial coverage, decompose $e=P_{U_\pi}e+P_{U_\pi^\perp}e$, control the covered component spectrally, and subtract the valid floor from the uncovered component.
\end{proof}

The zero-floor theorem identifies exactly which current residual directions the suite sees. A sharper excess-gap theorem is available when the optimum residual is compatible with the covered prediction defect.

\begin{theorem}[Normal-residual spectral coverage certificate]
\label{thm:normal-residual}
Assume a global optimum $\theta^\star$ is attained. Write
\[
e^\star=e(\theta^\star),
\qquad
d=\vecop(F_\theta-F_{\theta^\star}),
\qquad e=e^\star+d.
\]
Suppose, for a declared subspace $U$,
\begin{equation}
\ip{e^\star}{d}=0,
\qquad d\in U,
\qquad Q_\pi\succeq cP_U,
\qquad \xi\ge\norm{Q_\pi^{1/2}e^\star}.
\label{eq:normal-residual-assumptions}
\end{equation}
Then
\begin{equation}
\boxed{
J(\theta)-J^\star
\le
\frac{\bigl(\sqrt{2n[I_{\mathrm{blk}}(\theta)+\bar\epsilon_\pi]}+\xi\bigr)^2}{2nc}.}
\label{eq:normal-residual-gap}
\end{equation}
If $Q_\pi e^\star=0$, this simplifies to
\begin{equation}
\boxed{J(\theta)-J^\star\le\frac{I_{\mathrm{blk}}(\theta)+\bar\epsilon_\pi}{c}.}
\label{eq:normal-residual-simple}
\end{equation}
\end{theorem}
\begin{proof}
Orthogonality gives $J(\theta)-J^\star=\|d\|^2/(2n)$. Since $d\in U$ and $Q_\pi\succeq cP_U$,
$\sqrt c\|d\|\le\|Q_\pi^{1/2}d\|$. The triangle inequality and \cref{eq:spectral-average} give
\[
\|Q_\pi^{1/2}d\|
\le \|Q_\pi^{1/2}e\|+\|Q_\pi^{1/2}e^\star\|
\le \sqrt{2n[I_{\mathrm{blk}}+\bar\epsilon_\pi]}+\xi.
\]
Square and divide by $2nc$.
\end{proof}

The normal-residual condition is exact for affine least squares: the optimum residual is orthogonal to the attainable prediction subspace, while the current prediction difference lies inside it. In nonlinear models the theorem is a conditional compatibility result unless a separate argument or proof-bearing lower problem supplies a computable upper bound $\xi$ on $\|Q_\pi^{1/2}e^\star\|$.

\begin{corollary}[Green-to-gap certificate]
\label{cor:green-gap}
Suppose every block with $\pi_r>0$ is mandatory, every mandatory current call completes, and Green is issued at $\theta_t$. Then
\[
I_{\mathrm{blk}}(\theta_t)\le\tau_G.
\]
Under full residual coverage,
\begin{equation}
J(\theta_t)-J^\star\le\frac{\tau_G+\bar\epsilon_\pi}{c}.
\label{eq:green-gap}
\end{equation}
Under \cref{eq:normal-residual-assumptions}, replace $I_{\mathrm{blk}}$ by $\tau_G$ in \cref{eq:normal-residual-gap}; if $Q_\pi e^\star=0$, the same simple ratio applies to the true excess gap. These conclusions attach to the checkpoint at which the required current coverage calls completed; the retained historical frontier continues to support its one-sided witness statements.
\end{corollary}

On a region where a genuine gap inequality
$I_{\mathrm{blk}}(\theta)\ge c[J(\theta)-J^\star]-\bar\epsilon$
holds uniformly, the challenge-power objects satisfy
\begin{equation}
\Psi(B,s)\ge cs-\bar\epsilon,
\qquad
E(B,\tau)\le\frac{\tau+\bar\epsilon}{c},
\label{eq:spectral-power-law}
\end{equation}
where $B$ is the total budget of the complete suite.

\paragraph{Certificate record.}
A current-state spectral certificate records the covered residual subspace, the current decrease operators, every regularization contribution, and any gate-stability region used by the proof. An initialization spectrum enters only through a stated regional-stability theorem. These quantities determine the validity and tightness of the bound and are archived with the executable challenge values. \Cref{fig:neural-green-pipeline} summarizes the chain from executed block values to the coverage-qualified global-gap conclusion.

\begin{figure}[t]
\centering
\includegraphics[width=\linewidth]{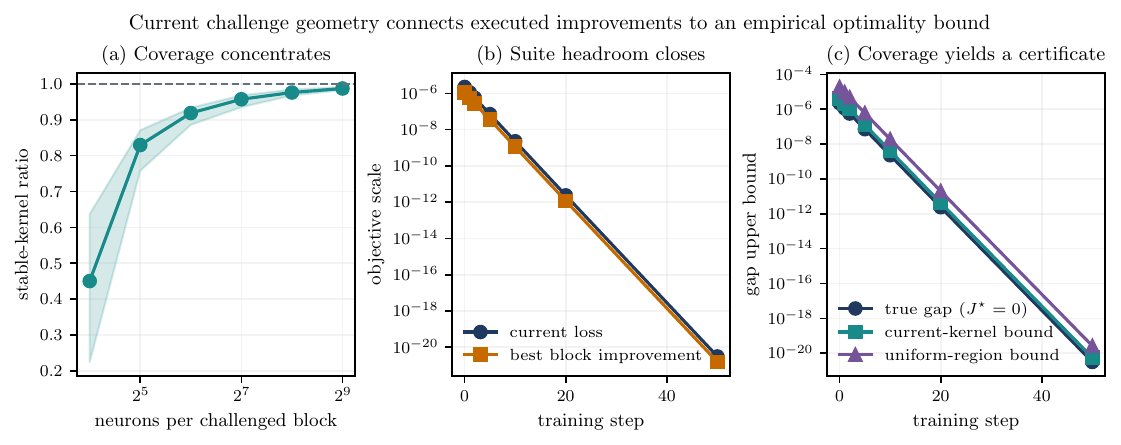}
\caption{The spectral coverage certificate chain. Finite-width geometry produces a stable coverage spectrum; executed current challenges quantify the residual directions that the suite can reduce; the certified current kernel converts that directional coverage into a global-gap upper bound. The panels show the theorem mechanisms and the quantities recorded by the certificate.}
\label{fig:neural-green-pipeline}
\end{figure}

\begin{corollary}[Realized-residual current-state certificate]
\label{cor:realized-residual}
Assume $J^\star=0$ for the fixed certification problem and let $e=e(\theta)\ne0$. For any completed suite mixture $Q_\pi$, define
\begin{equation}
\kappa_{\mathrm{cur}}(\theta,\pi)
=
\frac{e^\top Q_\pi e}{\|e\|^2}.
\label{eq:kappa-current}
\end{equation}
Then
\begin{equation}
I_{\mathrm{blk}}(\theta)+\bar\epsilon_\pi
\ge
\kappa_{\mathrm{cur}}(\theta,\pi)J(\theta),
\qquad
\boxed{J(\theta)-J^\star\le
\frac{I_{\mathrm{blk}}(\theta)+\bar\epsilon_\pi}{\kappa_{\mathrm{cur}}(\theta,\pi)}}
\label{eq:realized-residual-certificate}
\end{equation}
whenever $\kappa_{\mathrm{cur}}>0$.
\end{corollary}
\begin{proof}
Insert $e^\top Q_\pi e=\kappa_{\mathrm{cur}}\|e\|^2=2n\kappa_{\mathrm{cur}}J(\theta)$ into \cref{eq:spectral-average} and use $J^\star=0$.
\end{proof}

The uniform coefficient $\lambda_{\min}(Q_\pi)$ protects every possible residual direction. The realized coefficient \cref{eq:kappa-current} measures how strongly the same certified operator covers the residual actually present at the checkpoint. It is computed from the current residual and current operators already stored in the certificate; it does not transfer to another checkpoint without recomputation.

\subsection{Modern nonlinear nonvacuity: a current ResNet-18 certificate}
\label{sec:resnet-current-certificate}

A channel-gated CIFAR-10 ResNet-18 student is distilled from a teacher under a fixed empirical squared-logit objective on 24 certification examples. The teacher belongs to the declared student class, so $J^\star=0$ is known. This controlled construction is a ground-truth test of certificate validity and tightness. 
Eight architecture-valid gate challenges are run from each of five saved student checkpoints. Every call uses a verified Armijo step, materializes the complete network, restores the checkpoint before the next challenge, and recomputes the same objective. Exact current Jacobians are evaluated in memory-bounded reverse-mode chunks; the complete operator identity is numerically checked for every block.

All eight challenges are accepted at all five checkpoints. The E-optimal mixture has rank $240/240$ on the audited output space, and the residual fraction outside its range is zero throughout. The uniform worst-direction coefficient is positive but conservative: the corresponding bound is $480$--$6449$ times the known gap. The realized-residual coefficient is much larger along the current residual and yields a certificate only $1.74$--$3.02$ times the true gap. \Cref{tab:resnet-current,fig:resnet-current} report the complete result.

\begin{table}[t]
\centering
\small
\caption{Current internal spectral certificate on the channel-gated ResNet-18. The operator mixture has full rank and zero uncovered residual at every audited checkpoint. Uniform uses $\lambda_{\min}(Q_\pi)$; realized uses $\kappa_{\mathrm{cur}}$.}
\label{tab:resnet-current}
\begin{tabular}{rrrrrr}
\toprule
Epoch & True gap & $I_{\mathrm{blk}}$ & $\lambda_{\min}$ & $\kappa_{\mathrm{cur}}$ & Realized ratio\\
\midrule
0  & 55.555 & 1.887 & $7.07\times10^{-5}$ & 0.01125 & 3.020\\
5  & 16.836 & 3.678 & $1.10\times10^{-4}$ & 0.09582 & 2.280\\
10 & 10.751 & 3.233 & $7.09\times10^{-5}$ & 0.15819 & 1.901\\
20 & 6.475  & 1.784 & $4.27\times10^{-5}$ & 0.15812 & 1.743\\
40 & 2.958  & 0.635 & $3.41\times10^{-5}$ & 0.10891 & 1.972\\
\bottomrule
\end{tabular}
\end{table}

\begin{figure}[t]
\centering
\makebox[\linewidth][c]{%
  \includegraphics[width=1.03\linewidth]{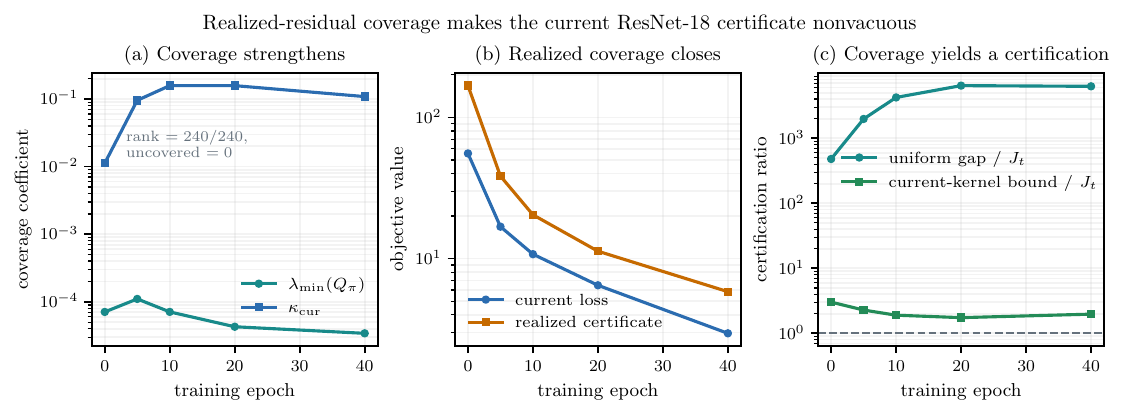}%
}
\caption{Uniform and realized-residual coverage on a channel-gated ResNet-18. The current eight-block suite covers all $240$ audited output directions at every checkpoint. The minimum-eigenvalue certificate protects a worst-case direction and is conservative; evaluating the same certified operator along the realized residual gives a nonvacuous current gap bound within $1.74$--$3.02\times$ of the known optimum gap.}
\label{fig:resnet-current}
\end{figure}

\subsection{E-optimal challenge bases}
\label{sec:eoptimal}

A sparse challenge basis can cover a declared residual subspace $U$ of dimension $d_U$. Represent each $Q_r$ on an orthonormal basis of $U$ and define
\begin{equation}
\chi_U=\max_{\pi\in\Delta_m}\lambda_{\min}\!\left(\sum_r\pi_rQ_r\big|_U\right).
\label{eq:eoptimal}
\end{equation}

\begin{proposition}[E-optimal minimax design and sparse support]
\label{prop:eoptimal}
The value $\chi_U$ is the semidefinite program
\[
\max_{c,\pi}\left\{c:\sum_r\pi_rQ_r\succeq cI_U,\ \pi\in\Delta_m\right\}
\]
and has minimax form
\begin{equation}
\chi_U
=
\min_{Z\succeq0,\ \tr Z=1}\max_r\tr(Q_rZ).
\label{eq:eoptimal-minimax}
\end{equation}
An optimal design exists with support at most $d_U(d_U+1)/2+1$.
\end{proposition}
\begin{proof}
Use $\lambda_{\min}(A)=\min_{Z\succeq0,\tr Z=1}\tr(AZ)$, apply Sion's minimax theorem to the bilinear payoff on compact convex sets, and apply Carath\'eodory's theorem in the $d_U(d_U+1)/2$-dimensional space of symmetric operators on $U$ \citep{sion1958,caratheodory1911}.
\end{proof}

The design exposes the least-covered residual density matrix $Z$ and selects the block most useful against it. In the registered eight-dimensional example, four two-coordinate blocks attain the information-theoretic maximum $\chi_U=0.25$, whereas random subsets are commonly singular through eight blocks. \Cref{fig:challenge-basis} shows how E-optimal design protects the weakest covered direction while using a sparse executable basis.

\begin{figure}[t]
\centering
\includegraphics[width=\linewidth]{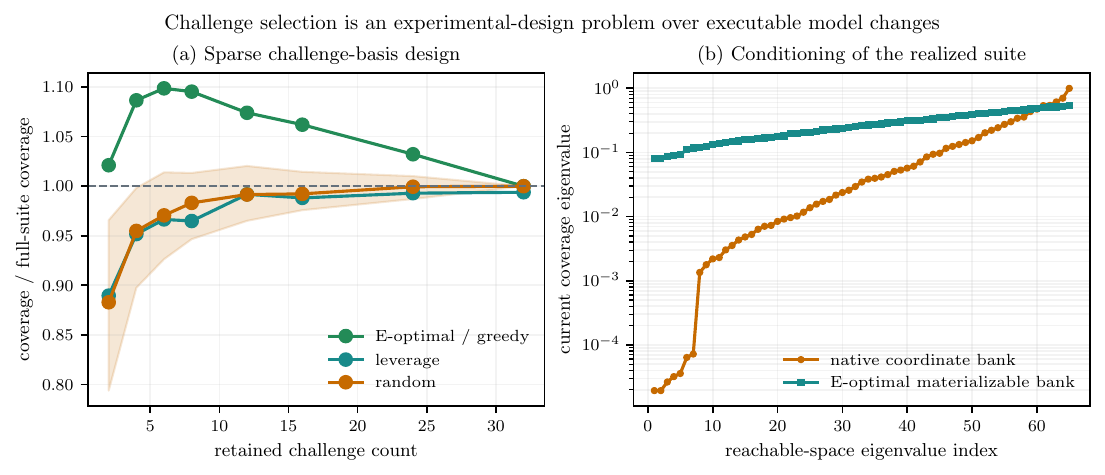}
\caption{Challenge-basis design. E-optimal weighting preserves the weakest residual direction and can be much sparser than the available portfolio. Random subsets can remain singular even after using more blocks. Only the executed and materialized selected calls enter status.}
\label{fig:challenge-basis}
\end{figure}

\subsection{Exact affine and collective neural instantiations}
\label{sec:exact-affine}

The cleanest exact challenge arises when all upstream nonlinear computation is frozen and the complete prediction is affine in an internal block. For certification matrices $T\in\R^{q\times n}$, $W\in\R^{q\times d}$, $H\in\R^{p\times n}$, define
\[
L(U)=\frac1{2nq}\fro{T-WUH}^2,
\qquad
P_W=WW^\dagger,
\qquad
P_H=H^\dagger H.
\]

\begin{theorem}[Exact residual-output challenge]
\label{thm:affine-block}
The minimum-Frobenius-norm solution is $U^\star=W^\dagger T H^\dagger$, the exact conditional optimum is
\begin{equation}
B_{\mathrm{cond}}=\frac1{2nq}\fro{T-P_WTP_H}^2,
\label{eq:affine-opt}
\end{equation}
and every current $U_0$ satisfies
\begin{equation}
L(U_0)-B_{\mathrm{cond}}
=
\frac1{2nq}\fro{P_WTP_H-WU_0H}^2.
\label{eq:affine-gap}
\end{equation}
The vectorized design has rank $\rank(H)\rank(W)$. If $W$ has full row rank and $H$ full column rank, then $B_{\mathrm{cond}}=0$, so the materialized replacement is a global empirical minimizer of the declared nonnegative unregularized squared-loss model.
\end{theorem}
\begin{proof}
Vectorization gives design $H^\top\otimes W$ and pseudoinverse $(H^\top)^\dagger\otimes W^\dagger$. The attainable prediction space is the matrices with columns in $\range(W)$ and rows in $\range(H)$; its orthogonal projection is $T\mapsto P_WTP_H$. Pythagoras gives \cref{eq:affine-opt,eq:affine-gap}.
\end{proof}

The same motif appears in a residual CNN's final internal convolution before linear pooling/readout and in a pre-LayerNorm transformer's MLP down-projection before the final residual addition, provided there is no post-addition nonlinearity or terminal LayerNorm. The exact architecture statement and terminal-normalization counterexample are in Appendix~\ref{app:architecture}.

A stronger collective result applies when several affine branches merge additively.

\begin{theorem}[Collective spectral coverage by rank-deficient blocks]
\label{thm:collective}
Suppose the complete predictions on the certification set are
\[
\widehat y(u_1,\ldots,u_m)=b+\sum_{r=1}^mA_ru_r\in\R^N,
\qquad
J(u)=\frac1{2N}\norm{y-\widehat y(u)}^2.
\]
Let $S_r=\range(A_r)$, $P_r=P_{S_r}$, $S=\sum_rS_r$, $P_S=P_S$, and
\[
F=\sum_{r=1}^mP_r,
\qquad
\lambda_{\mathrm{coll}}=\lambda_{\min}(F|_S)>0.
\]
For current residual $e=y-\widehat y(u)$,
\begin{align}
J(u)-J^\star&=\frac1{2N}\norm{P_Se}^2,\\
I_r(u)&=\frac1{2N}\norm{P_re}^2
\end{align}
for the exact frozen-context refit of block $r$. Consequently,
\begin{equation}
\sum_r I_r(u)\ge\lambda_{\mathrm{coll}}[J(u)-J^\star],
\qquad
\max_r I_r(u)\ge\frac{\lambda_{\mathrm{coll}}}m[J(u)-J^\star].
\label{eq:collective}
\end{equation}
Thus passage of all exact current block challenges at tolerance $\tau$ gives
\begin{equation}
J(u)-J^\star\le\frac{m\tau}{\lambda_{\mathrm{coll}}},
\label{eq:collective-passage}
\end{equation}
with $\sum_r\epsilon_r$ added to the numerator for approximate solves.
\end{theorem}
\begin{proof}
The irreducible residual is $(I-P_S)(y-b)$, yielding the gap identity. Exact refitting of block $r$ subtracts $P_re$. Since $P_re=P_rP_Se$,
\[
\sum_rI_r=\frac1{2N}\ip{P_Se}{FP_Se}\ge\frac{\lambda_{\mathrm{coll}}}{2N}\norm{P_Se}^2.
\]
The best block is at least the average.
\end{proof}

Collective coverage permits every individual block to be rank deficient. The theorem covers parallel or additive residual adapters and branches that merge linearly before a fixed readout. Sequential modifications use state-dependent decrease operators because later nonlinear features or buffers change. \Cref{fig:collective} makes the collective-versus-individual rank distinction explicit.

\begin{figure}[t]
\centering
\includegraphics[width=\linewidth]{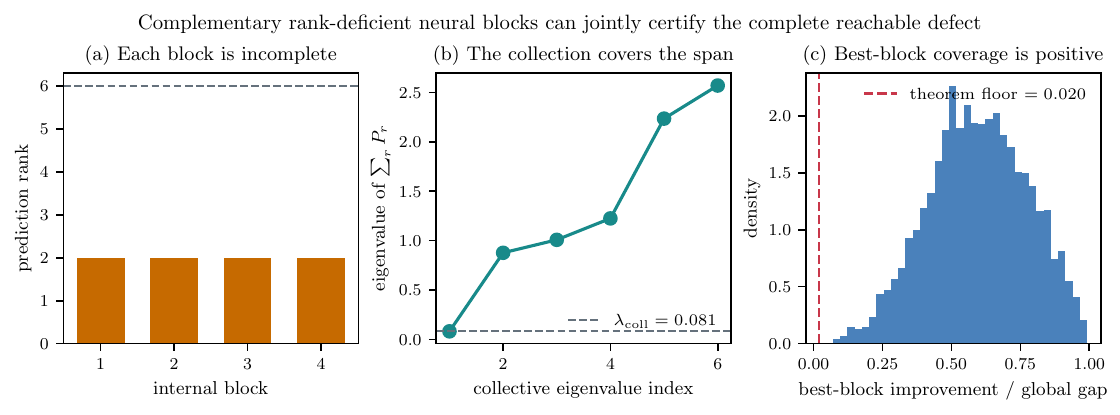}
\caption{Collective exact coverage. Four rank-two blocks jointly span a six-dimensional prediction space. The smallest fusion-frame eigenvalue gives a certified worst-case best-block fraction; observed random residuals are typically better. Approximate passage includes each solver error explicitly.}
\label{fig:collective}
\end{figure}

\subsection{A nonlinear finite-width ReLU theorem}
\label{sec:relu-theorem}

Once decrease operators are certified, the preceding results are deterministic. The next theorem establishes the spectral premise for hidden-weight challenges in a genuinely nonlinear network.

Consider a two-layer ReLU model with fixed output signs
\begin{equation}
f_W(x)=\frac1{\sqrt p}\sum_{j=1}^p a_j[w_j^\top x]_+,
\qquad a_j\in\{-1,+1\},
\label{eq:relu-model}
\end{equation}
unit-norm samples $x_i$, data matrix $X$ with the $x_i^\top$ as rows, and squared loss \cref{eq:squared-loss}. At Gaussian initialization $w_j^0\stackrel{\mathrm{iid}}{\sim}N(0,I)$ define
\[
D(w)=\diag(\mathbf1\{x_i^\top w>0\})_{i=1}^n,
\qquad
H^\infty=\E[D(w)XX^\top D(w)].
\]
Set
\[
\lambda_0=\lambda_{\min}(H^\infty)>0,
\qquad
R=\opnorm{X}^2.
\]
Choose
\[
0<\gamma\le\frac{\lambda_0\sqrt{2\pi}}{8nR}
\]
and call neuron $j$ stable when $\min_i|x_i^\top w_j^0|\ge2\gamma$. Prepartition the neurons into $m$ \emph{balanced} bins, each of size at most $\lceil p/m\rceil$, and let every bin challenge modify only its stable neurons while leaving unstable neurons and every other parameter unchanged. Let
\[
s_j=\mathbf1\!\left\{\min_i|x_i^\top w_j^0|\ge2\gamma\right\}.
\]
Define the stable-core empirical kernel
\[
H_{\mathrm{st}}=\frac1p\sum_{j=1}^p
s_jD(w_j^0)XX^\top D(w_j^0).
\]

\begin{theorem}[Stable-core hidden-block coverage certificate]
\label{thm:relu-green}
Fix $\delta\in(0,1)$. If
\begin{equation}
p\ge\frac{16R}{\lambda_0}\log\frac n\delta,
\label{eq:relu-width1}
\end{equation}
then with probability at least $1-\delta$,
\begin{equation}
\lambda_{\min}(H_{\mathrm{st}})\ge\frac{\lambda_0}{4}.
\label{eq:relu-kernel}
\end{equation}
On this event, simultaneously for every current checkpoint satisfying
\[
\norm{w_j-w_j^0}\le\gamma/2\quad\text{for stable neurons},
\qquad
\norm{f_W(X)-y}\le E_{\max},
\]
assume $p\ge m$ and
\begin{equation}
\sqrt p\ge\frac{mE_{\max}}{\gamma\sqrt R}.
\label{eq:relu-width2}
\end{equation}
Then the executable stable-bin gradient candidates with step
$\eta=mn/(2R)$ preserve every stable gate and satisfy
\begin{equation}
\boxed{\max_r[J(W)-J(W^{(r,+)})]\ge\frac{\lambda_0}{8R}J(W).}
\label{eq:relu-coverage}
\end{equation}
If these verified calls are mandatory and current Green is issued,
\begin{equation}
J(W)\le\frac{8R}{\lambda_0}\tau_G,
\qquad
\boxed{J(W)-J^\star\le\frac{8R}{\lambda_0}\tau_G.}
\label{eq:relu-green}
\end{equation}
A stronger heuristic block solver may be used, but the verified gradient candidate must be retained as a fallback unless a declared solver error is added to the passage tolerance.
\end{theorem}

The proof combines Gaussian anti-concentration, a lower-tail matrix Chernoff bound \citep{tropp2012matrix}, uniform gate preservation, and fixed-cell block smoothness; it appears in Appendix~\ref{app:relu-proof}. The width condition is sufficient and conservative. The theorem certifies the fixed empirical problem for the declared local hidden-weight suite. Its assumptions specify the gate-stable region, finite-width probability, and zero-floor bound used by the certificate.

\section{Why Coverage Is Necessary}
\label{sec:necessity}

A challenge-relative endpoint is the strongest universally valid positive statement when coverage is absent. The following theorem isolates the logical obstruction.

\begin{proposition}[A proper closure need not be global]
\label{thm:no-free-globality}
Let $\Theta$ be a compact metric space, let $\theta_0\in\Theta$, and let $\mathcal R\subsetneq\Theta$ be a closed proper subset containing $\theta_0$. There exists a continuous objective $J$ such that $\theta_0$ minimizes $J$ over $\mathcal R$ but is not globally optimal over $\Theta$.
\end{proposition}
\begin{proof}
Choose $\phi^\star\notin\mathcal R$ and a continuous bump supported in a ball disjoint from $\mathcal R$. Subtract a sufficiently large multiple of the bump from $d(\cdot,\theta_0)^2$.
\end{proof}

The next result is stronger: it uses a finite ReLU problem, an actual first-order trainer, and exact conditional head challenges.

\begin{theorem}[An exact infinite executable staircase can remain non-global]
\label{thm:non-global}
For every $m\ge3$, there is a two-unit ReLU regression problem on $m+1$ coordinate inputs with a zero-loss model, a nonempty open set of initializations, and a declared full-batch first-order trainer such that:
\begin{enumerate}[leftmargin=1.5em,itemsep=2pt]
\item one unit remains dead on every sample and the other remains inactive on the final sample;
\item the trainer converges to the positive floor $d^2/[2(m+1)]$;
\item the exact head value computed at initialization is retained as the online frontier until Green-0; mandatory identity/replay calls complete at every current issuance;
\item the first stronger exact head stair is computed only at Green-0, and every later stair only when the trainer reaches the preceding active target; every stair is materialized in the complete model;
\item with exact crossing/acceptance (or a vanishing tolerance schedule), the trainer starts Red, enters Yellow and Green, and generates and reaches infinitely many strictly decreasing exact stairs;
\item nevertheless,
\[
J(\theta_j)\longrightarrow\frac{d^2}{2(m+1)}>J^\star=0.
\]
\end{enumerate}
\end{theorem}

The construction and closed-form stair recurrence are proved in Appendix~\ref{app:staircase-proof}. The theorem uses an event policy that retains the initialization comparator until Green-0 and recomputes the exact head challenge at each subsequent hit event. The challenge is exact, every displayed stair is feasible, every mandatory issuance completes, and every active stair is reached. Failure comes from the preserved representation obstruction: the policy never revives the dead unit that can fit the final sample. Under a fixed positive materiality margin only a finite prefix is displayed, but any prescribed finite prefix can be scaled above that margin. \Cref{fig:positive-negative} places this exact obstruction beside the positive coverage regime and identifies the assumption separating them.

\begin{figure}[t]
\centering
\includegraphics[width=\linewidth]{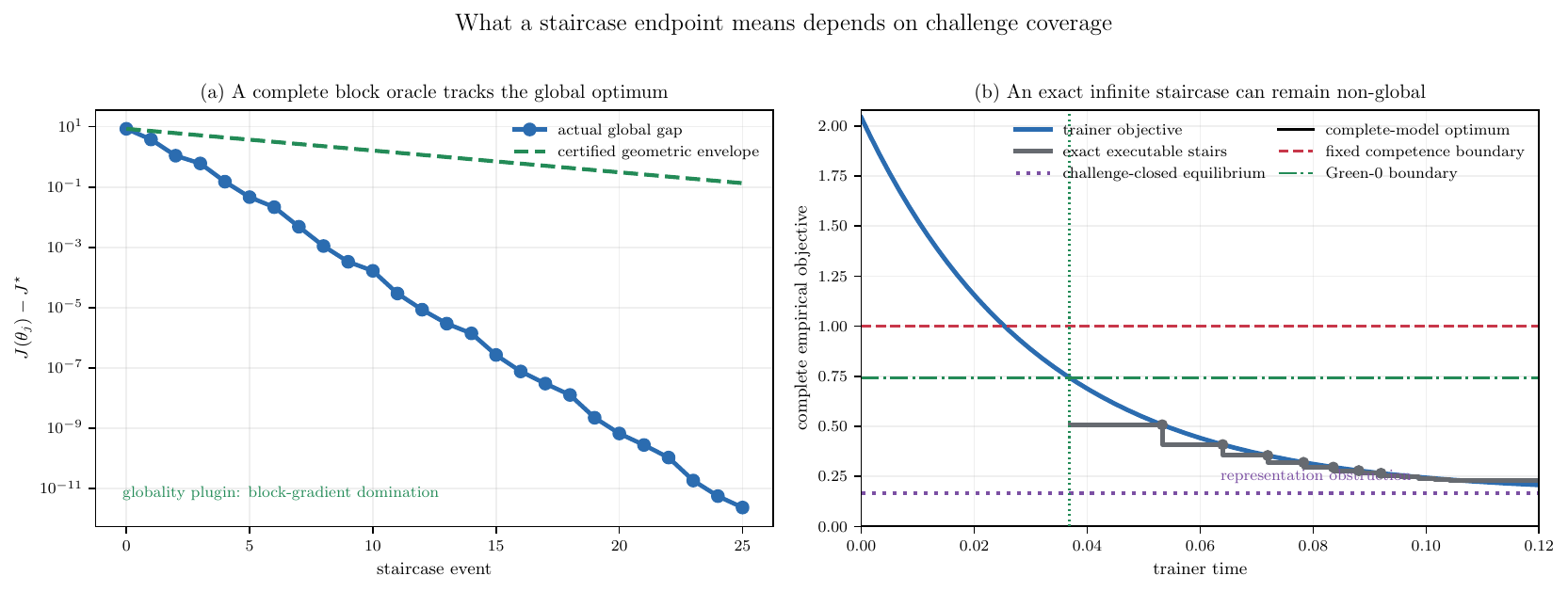}
\caption{Coverage, not stair count, determines global meaning. Left: under a valid block-coverage coefficient, exact reached stairs contract the true gap. Right: the support-preserving ReLU trainer reaches infinitely many exact current-representation head stairs but converges to a positive non-global floor. The positive and negative panels use different declared regimes; the figure identifies the assumption that separates them.}
\label{fig:positive-negative}
\end{figure}

Challenge-closed optimality characterizes the endpoint earned by an exhaustive declared operation family. Under compactness, regularity, and vanishing solve-hit error, a convergent event process is optimal over the closure of those executable operations. Completeness, power, spectral coverage, duality, or another lower-bound theorem then upgrades that endpoint to a global statement over the full certified class. The formal closure theorem appears in Appendix~\ref{app:closure}.

\section{Challenge-Closed Optimality and the Globality Frontier}
\label{sec:closure-main}

Coverage provides a global interpretation when the challenge family controls every relevant defect. The general endpoint is challenge-closed optimality: optimality over the closure of the declared executable operations. This concept remains useful across architectures and solver policies because the policy, information state, and replay memory are part of the state being certified.

Let $\mathcal K$ be a compact metric state space containing the model, optimizer state, buffers, finite policy memory, and realized auxiliary randomness. Let $p_\Theta:\mathcal K\to\Theta$ be the continuous projection to the certified model parameters and let $J_\Theta:\Theta\to\mathbb R$ be the continuous certified objective; write $J(x)=J_\Theta(p_\Theta(x))$. For hierarchy level $m$, let $\mathcal C_m(x)\subseteq\mathcal K$ contain the endpoints of every admissible composition of at most $m$ declared primitive challenge operations from state $x$, including the identity. Assume $\mathcal C_m(x)$ is nonempty, compact-valued, and nested in $m$. Define
\begin{equation}
b_m(x)=\min_{y\in\mathcal C_m(x)}J(y),
\qquad
g_m(x)=J(x)-b_m(x),
\label{eq:closure-level}
\end{equation}
and
\begin{equation}
\mathcal C_\infty(x)
=
\operatorname{cl}\!\left(\bigcup_{m\ge0}\mathcal C_m(x)\right),
\qquad
b_\infty(x)=\min_{y\in\mathcal C_\infty(x)}J(y),
\qquad
g_\infty(x)=J(x)-b_\infty(x).
\label{eq:closure-infty}
\end{equation}
A state is an $(m,\tau)$ challenge equilibrium when $g_m(x)\le\tau$, and a $\tau$ challenge-closed optimum when $g_\infty(x)\le\tau$.

Nestedness yields
\[
b_{m+1}(x)\le b_m(x)\le J(x),
\qquad
g_{m+1}(x)\ge g_m(x).
\]
The hierarchy therefore records progressively stronger passage and optimality statements. Exact one-block coverage gives blockwise global optimality; passage of a complete restart family establishes qualification relative to that family; equality $p_\Theta(\mathcal C_\infty(x))=\Theta$ gives full-class global optimality.

\begin{theorem}[Challenge-closure limit characterization]
\label{thm:closure-limit-main}
Assume $x_k\to\bar x$, $m_k\to\infty$, and $b_m$ is upper semicontinuous at $\bar x$ for every finite $m$. At event $k$, suppose a feasible returned stair satisfies
\[
S_{k+1}\le b_{m_k}(x_k)+\epsilon_k,
\qquad
J(x_{k+1})\le S_{k+1}+h_{k+1},
\]
with $\limsup_k(\epsilon_k+h_{k+1})\le e$. Then
\[
0\le J(\bar x)-b_m(\bar x)\le e
\quad\text{for every fixed }m,
\]
and hence
\begin{equation}
0\le g_\infty(\bar x)\le e.
\label{eq:closure-main-bound}
\end{equation}
When $e=0$, the convergent event sequence has an exact challenge-closed limit.
\end{theorem}

The proof appears in Appendix~\ref{app:closure}; its structure is simple. For any fixed $m$, hierarchy exhaustion eventually gives $m_k\ge m$, so the returned value is no larger than $b_m(x_k)$ up to solve error. Continuity of $J$ and upper semicontinuity of $b_m$ pass the inequality to the limit. Taking the closure over all finite levels yields \cref{eq:closure-main-bound}.

\subsection{Trainer--challenger equilibria}

Let $r_A(x)$ be a declared residual for the audited trainer. An $(\eta,\tau,m)$ trainer--challenger equilibrium satisfies
\[
r_A(x)\le\eta,
\qquad
g_m(x)\le\tau.
\]
The two coordinates answer independent questions: the trainer residual measures motion under one update rule, while $g_m$ measures executable improvement under the declared challenge family. A stationary Red or Yellow checkpoint is therefore a constructive trainer--challenger disagreement.

Full parameter convergence requires a separate trajectory theorem. Gradient flow and fixed-step gradient descent under the usual smoothness and boundedness assumptions drive the gradient residual to zero; KL or finite-length hypotheses can promote subsequential convergence to a single state. Once convergence is established, \cref{thm:closure-limit-main} identifies the executable endpoint reached by the challenge hierarchy.

\subsection{Finite-budget improvement exclusion}

For declared paths $\pi$ with total cost $c(\pi)$, define
\[
I_x(B)=J(x)-\min_{c(\pi)\le B}J(T_\pi(x)),
\qquad
\chi_x(\varepsilon)=\min\{c(\pi):J(T_\pi(x))\le J(x)-\varepsilon\}.
\]
Whenever the minima are attained,
\begin{equation}
I_x(B)\ge\varepsilon
\quad\Longleftrightarrow\quad
\chi_x(\varepsilon)\le B.
\label{eq:improvement-escape}
\end{equation}
Thus $I_x(B)\le\tau$ certifies that every declared route improving the objective by more than $\tau$ lies beyond budget $B$. This is the operational endpoint when exhaustive closure is unaffordable.

\subsection{Globality upgrades}

Challenge closure becomes near-global through a completeness defect
\[
\delta_m(\mathcal R)
=
\sup_{x\in\mathcal R}[b_m(x)-J^\star].
\]
If $g_m(x)\le\tau$, then
\begin{equation}
J(x)-J^\star\le\tau+\delta_m(\mathcal R).
\label{eq:closure-completeness}
\end{equation}
A contractive challenge law
\[
b_m(x)-J^\star
\le
\alpha_m[J(x)-J^\star]+\zeta_m,
\qquad 0\le\alpha_m<1,
\]
yields
\begin{equation}
J(x)-J^\star
\le
\frac{\tau+\zeta_m}{1-\alpha_m}.
\label{eq:closure-contraction}
\end{equation}
The spectral coverage certificate, exact affine-block coverage, proof-bearing dual floors, and complete activation-pattern solvers are concrete mechanisms for controlling these terms. The no-free-globality theorem and exact ReLU staircase in Section~\ref{sec:necessity} establish the converse boundary: a proper uncovered closure can contain a strict non-global minimum.

\paragraph{Road map from optimization evidence to interpretation.}
The central optimization argument is now complete: witness, passage, coverage, and globality. The next two sections collect exact, proof-bearing, and population-level mechanisms that instantiate or supplement that argument. Only afterward do we ask what the resulting evidence says about representation quality, trainer attainability, and protected-task performance.

\section{Exact Conditional Challenge Regimes}
\label{sec:exact-solvers}

One of the principal advantages of the proposed framework is that certification can be decomposed into layerwise, blockwise, or subnetwork problems for which global or near-global optimization is often available. These conditional problems are typically smaller than full-network training, can exploit convexity, affine structure, activation-pattern structure, or exact algebra, and return especially strong challenge values. For a fixed checkpoint, objective, and conditional feasible set, the exact optimum value is a deterministic scalar: different solver initializations may return different minimizers, but a correct exact solver returns the same value and therefore the same pass/fail conclusion. Certification is consequently reproducible rather than restart-dependent.

Table~\ref{tab:exact-solver-islands} summarizes the exact-solver islands used by the framework. Their assumptions differ, but their evidentiary role is the same: an exact conditional optimum yields the best possible challenge of its declared type, while a primal--dual interval states exactly how much solver uncertainty remains.

\begin{table}[t]
\centering
\small
\caption{Architecture-native exact and proof-bearing challenge islands. The conditional problem is substantially smaller than unrestricted network training, and its exact value is a canonical certificate for the fixed checkpoint and declared challenge class.}
\label{tab:exact-solver-islands}
\begin{adjustbox}{max width=\textwidth}
\begin{tabular}{@{}>{\raggedright\arraybackslash}p{3.15cm}>{\raggedright\arraybackslash}p{4.0cm}>{\raggedright\arraybackslash}p{4.5cm}>{\raggedright\arraybackslash}p{3.25cm}@{}}
\toprule
Challenge regime & Structural reason & Solver output & Certificate contribution \\
\midrule
Affine or ridge head & Convex least squares / strongly convex quadratic & Closed form or primal--dual optimum & Exact decoder under-use and complete-model witness \\
Affine residual-output CNN/transformer block & Complete prediction is $B+WUH$ & Orthogonal projection and canonical minimum-norm block & Exact internal-block gap; global empirical optimum under full-rank zero-loss conditions \\
Fixed-pattern ReLU segment & Activation mask converts the segment to a convex quadratic program & Exact pattern value; enumeration, cutting planes, or branch-and-bound & Exact conditional segment optimum and explicit residual solver gap \\
Frozen-backbone homogeneous ReLU adapter & Positive homogeneity yields an atomic convex formulation & Executable primal ceiling and certified polar-dual floor & Full-class global bracket in the declared adapter class \\
\bottomrule
\end{tabular}
\end{adjustbox}
\end{table}

\subsection{Exact conditional block certificates}

Partition the parameters as $\theta=(\theta_{-r},\phi_r)$ and define
\[
B_r^\star(\theta_{-r})
=
\inf_{\phi_r}J(\theta_{-r},\phi_r).
\]
When the infimum is attained by an exact conditional solver, materialized in the original architecture, and reevaluated under the complete objective,
\begin{equation}
J(\theta)-B_r^\star(\theta_{-r})
\label{eq:conditional-gap}
\end{equation}
is the exact conditional gap of block $r$ and a lower bound on the full empirical global gap. Sequential exact block replacements create a monotone executable challenger even when a full sweep is not globally optimal.

\subsection{A shared residual-output corollary for CNNs and transformers}

\Cref{thm:affine-block} applies directly when a graph cut freezes sample-specific offsets $b_i\in\mathbb R^q$ and features $h_i\in\mathbb R^p$, while an internal block $U\in\mathbb R^{d\times p}$ enters the complete prediction as
\[
\widehat y_i(U)=b_i+WU h_i.
\]
The same conditional projection theorem therefore covers a residual CNN's final internal convolution before linear pooling/readout and a pre-LayerNorm transformer's MLP down-projection before the final residual addition. The conditional optimum is exact for the declared frozen context, and the materialized complete model yields the executable certificate. A terminal nonlinearity or post-addition LayerNorm breaks the affine superposition and requires a different challenge or a bounded-gap solver. The architecture diagrams and boundary proof appear in Appendix~\ref{app:architecture}.

\begin{figure}[t]
\centering
\includegraphics[width=0.96\linewidth]{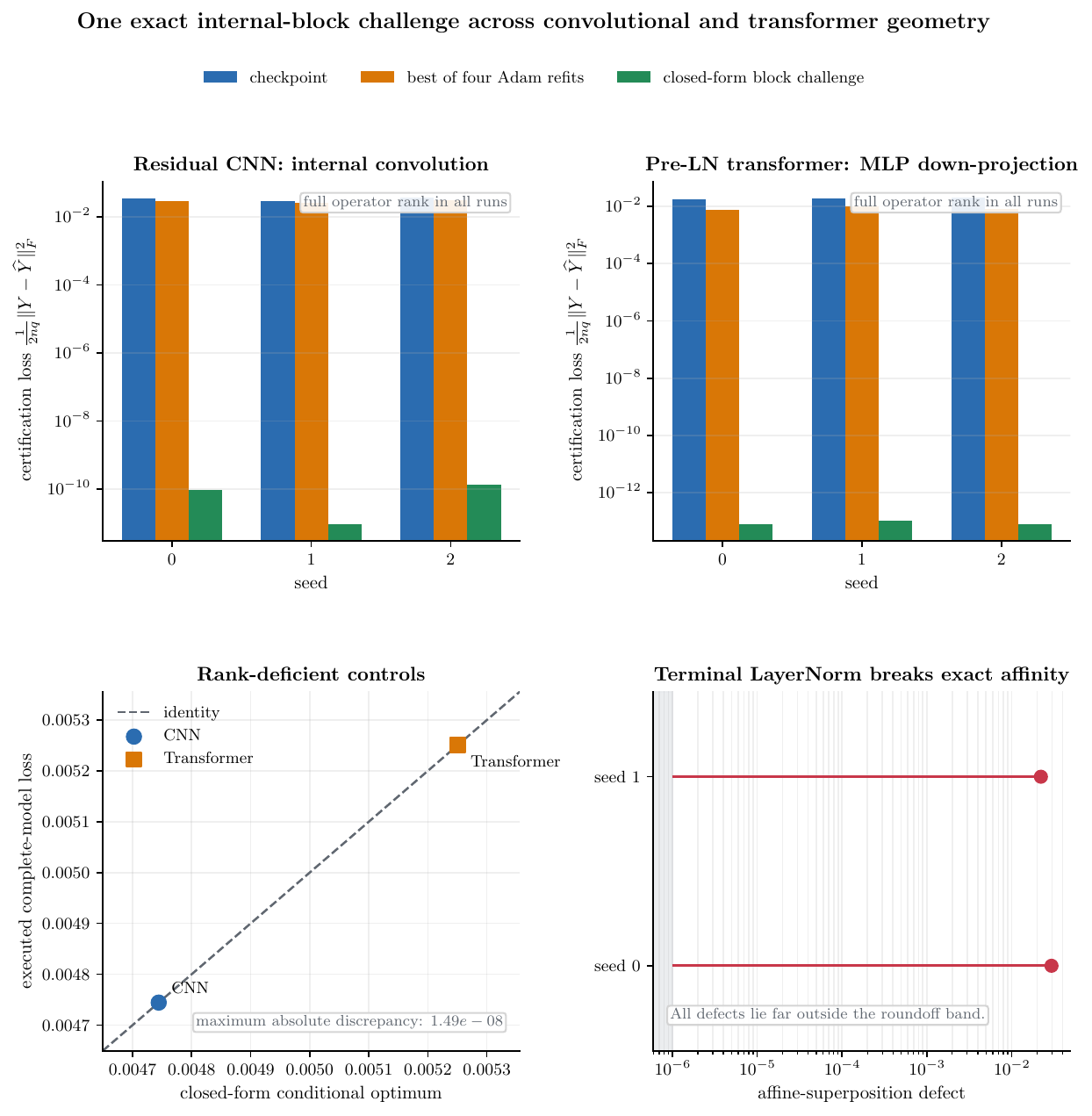}
\caption{One exact internal-block certificate across residual CNN and pre-LayerNorm transformer motifs. The closed-form replacement reaches the formula-predicted optimum in full-rank and rank-deficient controls. A terminal LayerNorm creates a clear affine-superposition defect, identifying the exact motif boundary.}
\label{fig:exact-architecture-main}
\end{figure}

\subsection{Projection, deep-linear, and activation-pattern anchors}

For frozen representation $H$, an exact linear head satisfies
\begin{equation}
\|Y-AH\|_F^2
=
\underbrace{\|Y(I-P_H)\|_F^2}_{\text{representation insufficiency}}
+
\underbrace{\|YP_H-AH\|_F^2}_{\text{head under-use}}.
\label{eq:projection-anchor}
\end{equation}
This identity is the additive anchor for the paired representation certificate.

For a deep linear network with widths $d_0,\ldots,d_K$, input matrix $X$, and bottleneck rank $r=\min_{0\le k\le K}d_k$, the exact squared-loss optimum is
\begin{equation}
\|Y(I-P_X)\|_F^2
+
\sum_{j>r}\sigma_j^2(YP_X).
\label{eq:deep-linear-anchor}
\end{equation}
The first term is inaccessible from the inputs; the second is the architectural bottleneck cost. It provides an exact regression test for class-preserving materialization.
The complete statement, attainability argument, rank-constrained-head analogue, target-code result, and class-preserving padding construction appear in Appendix~\ref{app:projection-anchors}.

For a one-layer ReLU segment with fixed input $H$, each realizable activation mask defines a convex quadratic problem. Enumerating all realizable masks recovers the global conditional value. The naive mask space is exponential; fixed-rank hyperplane-arrangement bounds, cutting planes, and branch-and-bound determine when enumeration is practical. Every primal candidate remains an executable witness, and every valid lower bound sharpens the conditional interval. The full conditional-solver intervals, activation-pattern theorem, shared-model feasibility result, stochastic attainability result, and task-calibrated translations are collected in Appendix~\ref{app:strong-solvers}.

\section{Proof-Bearing and Population-Level Extensions}
\label{sec:coverage-plugins}

The spectral theorem converts current executable improvements into an optimality bound through residual-defect coverage. Two additional mechanisms extend the certificate beyond settings with directly verified spectral geometry. A proof-bearing lower floor gives a deterministic pointwise bracket for one exact neural class. Independent calibration instead controls the miss behavior of a frozen randomized policy over a declared checkpoint population. The deterministic bracket and the population card answer complementary questions and are reported in separate fields.

\subsection{A proof-bearing neural bracket}
\label{sec:adapter-bracket}

Let $Z\in\R^{n\times d}$ be a frozen feature matrix and $y\in\R^n$ the residual target after subtracting a frozen backbone prediction. For width $M$ and $\lambda>0$, consider the homogeneous two-layer ReLU adapter
\begin{equation}
J_{\mathrm{ad,M}}(a,u)
=
\frac12\left\|y-\sum_{j=1}^{M}a_j(Zu_j)_+\right\|^2
+
\frac{\lambda}{2}\sum_{j=1}^{M}(a_j^2+\|u_j\|^2),
\qquad J_{\mathrm{ad,M}}^\star=\inf J_{\mathrm{ad,M}}.
\label{eq:adapter-objective}
\end{equation}
Define the symmetric atom set $\mathcal A_Z=\{\pm(Zv)_+:\|v\|=1\}$, its gauge $\|\cdot\|_{\mathcal A_Z}$, and the atomic primal--dual pair
\begin{align}
P_Z^\star
&=\min_{z\in\R^n}\frac12\|y-z\|^2+\lambda\|z\|_{\mathcal A_Z},
\label{eq:adapter-primal}\\
D_Z^\star
&=\max_{\nu\in\R^n}
\left\{y^\top\nu-\frac12\|\nu\|^2:
\sigma_Z(\nu)\le\lambda\right\},
\quad
\sigma_Z(\nu)=\sup_{\|v\|\le1}|\nu^\top(Zv)_+|.
\label{eq:adapter-dual}
\end{align}
Positive homogeneity and balancing map finite networks to signed atoms, Fenchel--Rockafellar duality gives $P_Z^\star=D_Z^\star$ \citep{rockafellar1970}, and $M\ge n+1$ is sufficient for equality with the finite-width neural optimum by Carath\'eodory \citep{caratheodory1911}.

At event $s$, let $\overline U_s$ be an outward-rounded complete adapter value. For a proposed dual vector $\nu_s$, let a complete separator return $\overline\sigma_s\ge\sigma_Z(\nu_s)$, set
\[
\rho_s=
\begin{cases}
1, & \overline\sigma_s\le\lambda,\\
\lambda/\overline\sigma_s, & \overline\sigma_s>\lambda,
\end{cases}
\qquad
\underline D_s\le y^\top(\rho_s\nu_s)-\tfrac12\|\rho_s\nu_s\|^2,
\]
and retain
\begin{equation}
L_t=\max_{s\le t}\underline D_s,
\qquad
U_t=\min_{s\le t}\overline U_s.
\label{eq:adapter-retained}
\end{equation}

\begin{theorem}[Executable primal--dual neural coverage]
\label{thm:adapter-coverage}
Assume every model entering $U_t$ uses the same frozen backbone, data, objective, regularizer, and width bound $M$, and is completely materialized and reevaluated. Assume every separator and numerical enclosure is valid. Then
\begin{equation}
\boxed{L_t\le P_Z^\star\le J_{\mathrm{ad,M}}^\star\le U_t.}
\label{eq:adapter-bracket}
\end{equation}
Consequently, for any width-$M$ checkpoint $\theta$,
\begin{equation}
J_{\mathrm{ad,M}}(\theta)-J_{\mathrm{ad,M}}^\star
\le J_{\mathrm{ad,M}}(\theta)-L_t.
\label{eq:adapter-gap-upper}
\end{equation}
If $U_t\le J_{\mathrm{ad,M}}(\theta)$ and
$\Delta_t=J_{\mathrm{ad,M}}(\theta)-U_t$, then
\begin{equation}
\boxed{
\Delta_t
\le J_{\mathrm{ad,M}}(\theta)-J_{\mathrm{ad,M}}^\star
\le\Delta_t+(U_t-L_t).}
\label{eq:adapter-sandwich}
\end{equation}
Thus the executable witness estimates the entire unknown gap to additive error at most the bracket width.
\end{theorem}
\begin{proof}
The scaled vector $\rho_s\nu_s$ is dual feasible because
$\rho_s\sigma_Z(\nu_s)\le\rho_s\overline\sigma_s\le\lambda$. Weak duality gives $\underline D_s\le P_Z^\star$, while every materialized adapter is feasible for the width-$M$ problem. Retention preserves both inequalities; subtraction gives \cref{eq:adapter-gap-upper,eq:adapter-sandwich}.
\end{proof}

The complete separator has two roles. A violated polar constraint supplies the next quantitatively improving ReLU atom; complete nonviolation supplies the floor. In the exact-history rank-two audit, the bracket closes to $3.9941\times10^{-7}$. Across 36 randomized rank-two instances, the median relative width is $7.44\times10^{-5}$ and the largest ratio of the computable gap upper bound to the realized witness is $1.000006$. The proof and outward-rounding details are in Appendix~\ref{app:coverage-details}. The bracket is a pointwise global certificate for the exactly declared frozen-backbone adapter class. \Cref{fig:adapter-bracket} shows the retained executable ceiling and certified floor closing from opposite directions.

\begin{figure}[t]
\centering
\includegraphics[width=\linewidth]{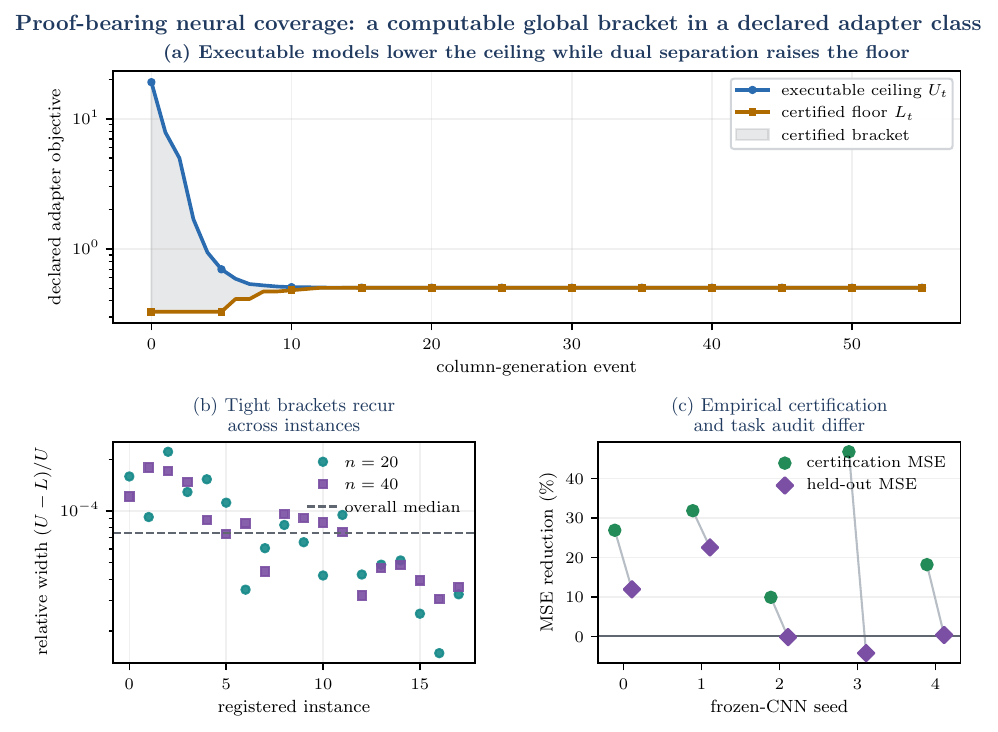}
\caption{Proof-bearing pointwise coverage in a declared neural class. Complete executable adapters lower the primal ceiling while complete polar separation raises the certified floor. The resulting bracket nearly identifies the entire empirical adapter gap, whereas held-out behavior remains a separate audit.}
\label{fig:adapter-bracket}
\end{figure}

\subsection{Finite-sample population coverage of a frozen policy}
\label{sec:statistical-coverage}

A randomized or learned policy may be too broad for deterministic pointwise coverage but can still be calibrated over a named failure population. Freeze the complete policy after development. Let
\[
q(\theta;\delta)
=\Prb\{\text{one fresh policy call constructs a sound improvement of at least }\delta\mid\theta\}.
\]
Let $\theta_1,\ldots,\theta_n\stackrel{\mathrm{iid}}{\sim}P_{\mathcal F}$ with $n\ge1$. At each checkpoint execute $m\ge1$ calls using fresh policy randomness that is independent across calls and calibration episodes conditional on the sampled checkpoints, and observe
$S_i\mid\theta_i\sim\operatorname{Binomial}(m,q_i)$. Choose a cutoff $c\in\{0,\ldots,m-1\}$ using development data and set $T=\sum_i\mathbf1\{S_i\le c\}$. For confidence error $\alpha\in(0,1)$ and permitted uncovered mass $\beta\in(0,1)$, define
\begin{align}
a_\alpha(T)
&=
\begin{cases}
1,&T=n,\\
1-F^{-1}_{\operatorname{Beta}(n-T,T+1)}(\alpha),&0\le T<n,
\end{cases}\\
p_L
&=
\begin{cases}
0,&T=n\text{ or }a_\alpha(T)\ge\beta,\\
1-F^{-1}_{\operatorname{Beta}(m-c,c+1)}\!\left(a_\alpha(T)/\beta\right),
& a_\alpha(T)<\beta.
\end{cases}
\label{eq:population-pl}
\end{align}
This exact finite-trial inversion is related to confidence procedures for binomial mixing distributions \citep{basu2026mixing}.

\begin{theorem}[Finite-sample valid finite-trial quantile coverage]
\label{thm:population-coverage}
Assume $\theta_1,\ldots,\theta_n\stackrel{\mathrm{iid}}{\sim}P_{\mathcal F}$, with $n,m\ge1$, $c\in\{0,\ldots,m-1\}$, and $\alpha,\beta\in(0,1)$. The policy, failure population, threshold $\delta$, trial count, cutoff, and confidence levels are fixed before calibration counts are inspected. Conditional on the sampled checkpoints, the policy-call randomness is independent across all calibration episodes and calls. Then, with probability at least $1-\alpha$,
\begin{equation}
\boxed{
\Prb_{\theta\sim P_{\mathcal F}}\{q(\theta;\delta)\ge p_L\}\ge1-\beta.}
\label{eq:population-coverage}
\end{equation}
For $k$ conditionally independent fresh calls at a future failure checkpoint,
\begin{equation}
\Prb(\text{miss all $k$ calls})
\le\beta+(1-\beta)(1-p_L)^k
\label{eq:population-miss}
\end{equation}
with the same calibration confidence.
\end{theorem}

The outer confidence probability is over the joint draw of the iid calibration checkpoints and their fresh policy-call randomness. Conditional on realized checkpoints, the count is Poisson--binomial; after marginalizing over the iid checkpoint draw, the low-count indicators are iid Bernoulli with the common marginal probability inverted in Appendix~\ref{app:coverage-details}.

The proof, given in Appendix~\ref{app:coverage-details}, inverts nested composite binomial-mixture tests while allowing heterogeneous per-checkpoint detection probabilities. Policy selection must be separated from calibration or corrected. In the registered least-favorable stress test, a policy frozen before calibration attains $95.3\%$ confidence coverage; selecting the best of 20 policies on the same calibration data drops coverage to $40.5\%$, while Bonferroni correction restores $98.1\%$. The card controls population false negatives for the frozen policy. Deterministic coverage mechanisms supply pointwise floors on~$J^\star$.

\begin{corollary}[Hybrid deterministic--statistical coverage]
\label{cor:hybrid-coverage}
Suppose a proof-bearing branch has zero misses on its certified domain and at most a fraction $\pi_{\mathrm{res}}$ of failures are routed to a residual branch satisfying \cref{eq:population-coverage}. Then
\[
\Prb(\text{portfolio misses a future failure})
\le
\pi_{\mathrm{res}}\bigl[\beta+(1-\beta)(1-p_L)^k\bigr].
\]
\end{corollary}

The factor $\pi_{\mathrm{res}}$ multiplies the entire residual-branch miss probability because failures handled by the proof-bearing branch cannot be missed under the stated routing model.  \Cref{fig:population-coverage} reports both finite-trial strengthening and the failure caused by reusing calibration data for policy selection.

\begin{figure}[t]
\centering
\includegraphics[width=\linewidth]{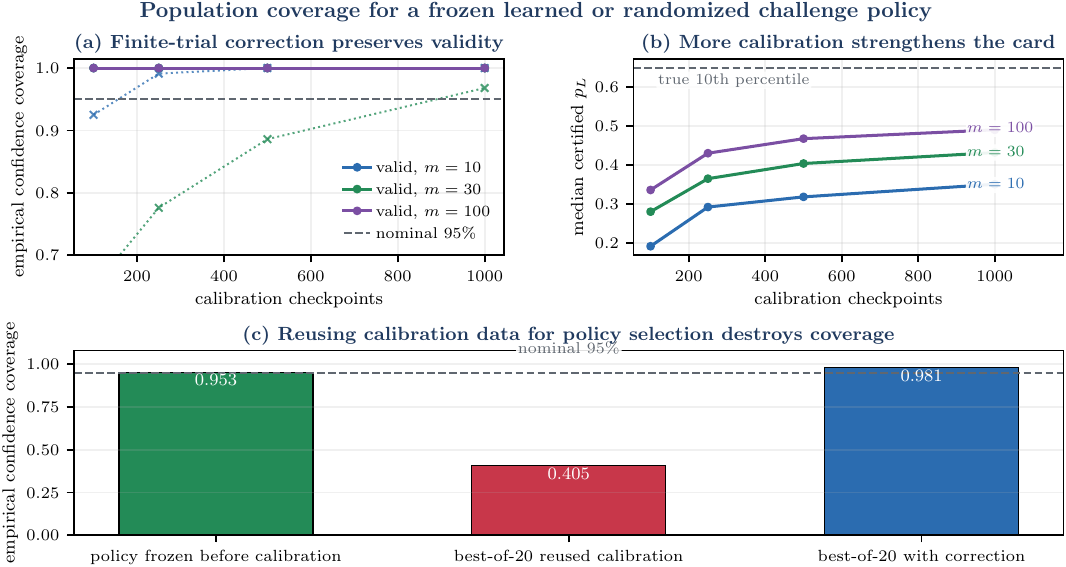}
\caption{Calibrated population coverage for a frozen randomized challenge policy. Finite-trial correction preserves nominal confidence and strengthens with additional checkpoints and calls. Development-data selection, fresh calibration, or a declared multiplicity correction preserves the confidence guarantee. The card bounds misses over the named checkpoint population; deterministic coverage supplies pointwise global-gap evidence.}
\label{fig:population-coverage}
\end{figure}

\subsection{Null-relative rarity under a declared fresh-null law}
\label{sec:green-rarity}

Deterministic certificates and null-relative rarity answer different questions. A complete lower-valued model proves an inequality for the audited empirical problem. A null calculation asks how unusual a frozen observed threshold would be under a declared random mechanism.

Let $\mathcal H_t$ contain the history used to select the reported checkpoint, suite, and threshold. Conditional on $\mathcal H_t$, draw a fresh complete model $W\sim\mathcal Q_t(\cdot\mid\mathcal H_t)$ independently of the observed training episode and define
\begin{equation}
F_{\mathcal Q_t}(b\mid\mathcal H_t)
=
\Prb\!\left(J(W)\le b\mid\mathcal H_t\right).
\label{eq:null-cdf}
\end{equation}
Three useful frozen thresholds are
\[
b_t^{\mathrm{gate}}=G_t+\tau_G,
\qquad
b_t^{\mathrm{depth}}=G_t,
\qquad
b_j^{\mathrm{stair}}=S_j+\tau_{\mathrm{hit}}.
\]
The retained depth below the reference is equivalently $g_t^{\mathrm{depth}}=J_{\mathrm{ref}}-G_t\ge0$. If $\bar p_t$ is a valid conditional upper bound on $F_{\mathcal Q_t}(b\mid\mathcal H_t)$ at one of these frozen thresholds, define
\begin{equation}
\mathcal R_{\mathcal Q,t}=-\log_{10}\bar p_t.
\label{eq:green-surprisal-score}
\end{equation}
Then an independent null replicate reaches at least that loss depth with probability at most $10^{-\mathcal R_{\mathcal Q,t}}$. This is a due-diligence statement under the named null, not a global-optimality certificate and not automatically a sequential p-value for the trajectory that selected the threshold.

A conditional block null fixes the remainder of the checkpoint and randomizes one predeclared block:
\begin{equation}
p_{r,t}
=
\Prb\!\left[
J(\theta_{-r,t},W_r)\le J(\theta_t)
\mid\mathcal H_t,\theta_{-r,t}
\right].
\label{eq:block-surprisal}
\end{equation}
The construction applies to a head, adapter, LoRA module, attention block, CNN stage, or router when the null law is fixed before the fresh replicate is drawn.

\begin{proposition}[Gaussian small-ball null]
\label{thm:gaussian-small-ball}
Let $z\sim\mathcal N(\mu,\Sigma)$, let $y$ be fixed, and set $D_z=\|z-y\|^2$. For every $b>0$,
\begin{equation}
\Prb(D_z\le b)
\le
\inf_{s>0}
\frac{\exp\!\left[sb-s(\mu-y)^\top(I+2s\Sigma)^{-1}(\mu-y)\right]}
{\sqrt{\det(I+2s\Sigma)}}.
\label{eq:gaussian-small-ball}
\end{equation}
If $\Sigma=\sigma^2I_d$ with $\sigma>0$, then $D_z/\sigma^2$ has a noncentral $\chi^2_d$ law with noncentrality $\|\mu-y\|^2/\sigma^2$, giving the exact probability.
\end{proposition}
\begin{proof}
For $s>0$, Markov's inequality gives $\Prb(D_z\le b)\le e^{sb}\mathbb E e^{-sD_z}$. Completing the square in the Gaussian integral yields the displayed expression; optimize over $s$.
\end{proof}

\begin{proposition}[Exact isotropic alignment null]
\label{prop:beta-alignment-null}
Let $Y\in\R^{q\times n}$ have independent standard Gaussian entries and be independent of a fixed rank-$r$ orthogonal projector $P$, with $0<r<n$. Then
\[
\alpha(Y,P)=\frac{\|YP\|_F^2}{\|Y\|_F^2}
\sim
\operatorname{Beta}\!\left(\frac{qr}{2},\frac{q(n-r)}{2}\right).
\]
\end{proposition}
\begin{proof}
Rotate coordinates so that $P=\operatorname{diag}(I_r,0)$. The projected and orthogonal energies are independent chi-square variables with $qr$ and $q(n-r)$ degrees of freedom; their ratio to the sum has the beta law.
\end{proof}

\Cref{fig:green-rarity} gives two worked null calculations: a Gaussian small-ball probability and an exact beta alignment tail. 
In consequential use, the null should represent the alternative explanation under examination, such as a matched-budget restart or randomized block policy.

\begin{figure}[t]
\centering
\includegraphics[width=\linewidth]{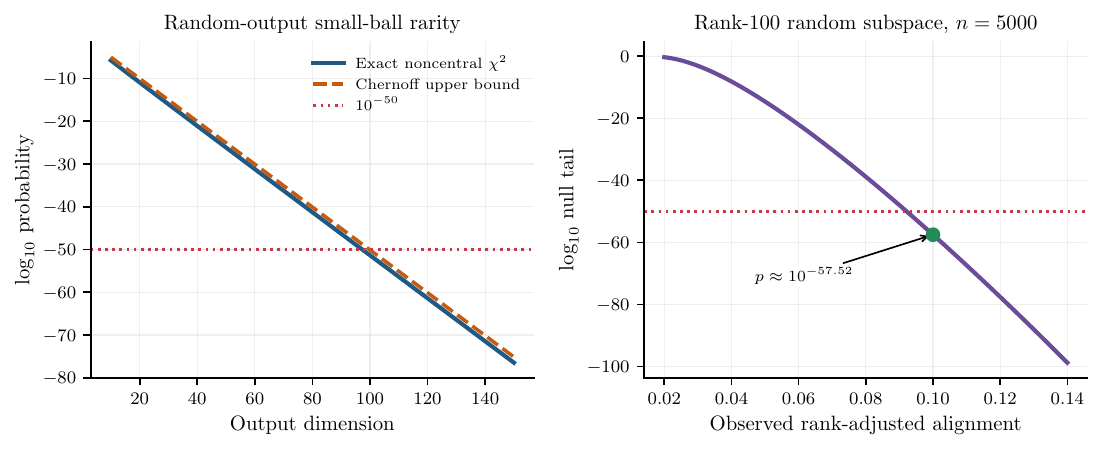}
\caption{Two worked null-relative rarity calculations. Left: the exact noncentral-$\chi^2$ lower-tail probability and its valid Gaussian Chernoff upper bound for the declared small-ball problem. Right: the exact beta tail for the declared rank-adjusted Gaussian alignment null. The model and threshold determine the statistical interpretation; the executable candidate and objective value carry the deterministic certificate.}
\label{fig:green-rarity}
\end{figure}

Repeated inspection across epochs, seeds, architectures, or policies requires a declared multiplicity treatment. A finite look set admits a family-wise correction; confidence sequences or e-processes give time-uniform alternatives when their premises hold \citep{howard2021confidence,ramdas2023game}. Permutation references apply only when a complete permutation null is executed \citep{ojala2010permutation}. The pathwise inequalities $J^\star\le G_t$ and $J^\star\le S_j$ require no statistical correction because every displayed value belongs to a stored feasible model.

\section{Representation Sufficiency from Paired Certificates}
\label{sec:representation}

An optimized decoder and an informative representation are different properties. If a new head improves the checkpoint while the representation is frozen, the missing value was already present in that representation but unused by the current decoder. Representation insufficiency is what remains after the best representation-preserving decoder has been accounted for. This section formalizes that subtraction with two nested predictive certificates.

Let a predictor factor as $f_{\phi,w}=g_w\circ h_\phi$ and write $Z=h_\phi(X)$. Fix a predictive risk $R$. Let
\begin{equation}
R_X^\star=\inf_{f\in\mathcal F_X}R(f(X)),
\qquad
R_Z^\star=\inf_{g\in\mathcal F_Z}R(g(Z)),
\label{eq:nested-risks}
\end{equation}
where every pullback $g\circ h_\phi$ belongs to $\mathcal F_X$. Thus $R_X^\star\le R_Z^\star$. At current decoder risk $r=R(g_w(Z))$,
\begin{equation}
r-R_X^\star
=
\underbrace{r-R_Z^\star}_{\text{decoder under-use}}
+
\underbrace{R_Z^\star-R_X^\star}_{D_{\mathrm{rep}}:\text{ representation deficit}}.
\label{eq:risk-decomp}
\end{equation}
The two optima become certifiable through paired bounds
\begin{equation}
L_X\le R_X^\star\le U_X,
\qquad
L_Z\le R_Z^\star\le U_Z,
\label{eq:paired-brackets}
\end{equation}
where upper endpoints are complete executable models and lower endpoints are valid for exactly the named classes, objective, data law, and regularizer.

\begin{theorem}[Sharp nested-certificate representation interval]
\label{thm:nested-rep}
Assume \cref{eq:paired-brackets} is compatible with $R_X^\star\le R_Z^\star$. Then
\begin{equation}
\boxed{
\max\{0,L_Z-U_X\}
\le D_{\mathrm{rep}}
\le U_Z-L_X.}
\label{eq:rep-interval}
\end{equation}
Both endpoints are sharp from the two marginal brackets and nesting alone.

Because the current decoder belongs to the representation-preserving class, tighten $U_Z\leftarrow\min\{U_Z,r\}$. Define the outer gap upper bound $\Gamma_X=r-L_X$ and the nonnegative executable representation-preserving improvement $\Delta_Z=r-U_Z$. Then
\begin{equation}
\boxed{D_{\mathrm{rep}}\le\Gamma_X-\Delta_Z.}
\label{eq:headroom-subtract}
\end{equation}
Thus demonstrated decoder headroom is subtracted from total predictive uncertainty, leaving the representation-deficit interval.
\end{theorem}
\begin{proof}
The upper bound uses $R_Z^\star\le U_Z$ and $R_X^\star\ge L_X$; the lower bound uses $R_Z^\star\ge L_Z$, $R_X^\star\le U_X$, and nesting. If the intervals overlap, $R_X^\star=R_Z^\star$ attains the zero lower endpoint; otherwise $(U_X,L_Z)$ attains $L_Z-U_X$. The pair $(L_X,U_Z)$ attains the upper endpoint. Substitution gives \cref{eq:headroom-subtract}.
\end{proof}

Let $R_{\mathrm{ref}}>R_X^\star$ be a fixed reference feasible in the representation-preserving class and define
\[
M_{\mathrm{rep}}=\frac{R_{\mathrm{ref}}-R_Z^\star}{R_{\mathrm{ref}}-R_X^\star}.
\]

\begin{corollary}[Representation-maturity interval]
\label{prop:maturity-interval}
Tighten $U_Z\leftarrow\min\{U_Z,R_{\mathrm{ref}}\}$ and assume $U_X<R_{\mathrm{ref}}$. Then
\begin{equation}
\boxed{
\frac{R_{\mathrm{ref}}-U_Z}{R_{\mathrm{ref}}-L_X}
\le M_{\mathrm{rep}}
\le
\min\left\{1,\frac{R_{\mathrm{ref}}-L_Z}{R_{\mathrm{ref}}-U_X}\right\}.}
\label{eq:maturity-interval}
\end{equation}
Both endpoints are sharp from the two marginal brackets and nesting alone.
\end{corollary}
\begin{proof}
The ratio is decreasing in $R_Z^\star$ and increasing in $R_X^\star$ whenever $R_X^\star<R_{\mathrm{ref}}$. The lower endpoint uses $(R_X^\star,R_Z^\star)=(L_X,U_Z)$. For the upper endpoint, $(U_X,L_Z)$ is feasible when $U_X\le L_Z$; otherwise nesting permits equality $R_X^\star=R_Z^\star$ and hence $M_{\mathrm{rep}}=1$. These choices also establish sharpness.
\end{proof}

\Cref{fig:information-sufficiency} visualizes the two independent brackets and the cross-difference that yields the representation-deficit interval.

\begin{figure}[t]
\centering
\includegraphics[width=0.96\linewidth]{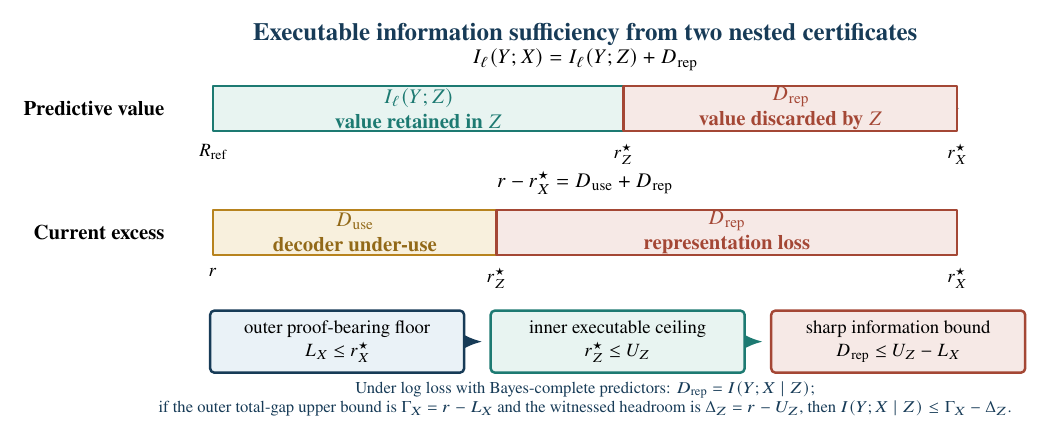}
\caption{Paired certificates isolate representation insufficiency. An outer bracket controls the best full-input predictor, while an inner representation-preserving bracket controls the best decoder on $Z$. Their cross-difference gives the sharp deficit interval. Under Bayes-complete log loss the same quantity is conditional mutual information; under squared loss it is conditional-mean predictive loss.}
\label{fig:information-sufficiency}
\end{figure}

\subsection{Log-loss information certificate}
\label{sec:log-info}

Assume $\mathcal Y$ is finite, $Y-X-Z$, risk is population natural-log loss, and both predictor classes are Bayes-complete. Then
\[
R_X^\star=H(Y\mid X),
\qquad
R_Z^\star=H(Y\mid Z),
\qquad
D_{\mathrm{rep}}=I(Y;X\mid Z).
\]
These identities are classical consequences of proper scoring and the Markov chain \citep{gneiting2007,reid2011information,williamson2024bridge}.

\begin{corollary}[Executable information-sufficiency certificate]
\label{cor:log-info}
Under the preceding assumptions,
\begin{equation}
\max\{0,L_Z-U_X\}
\le I(Y;X\mid Z)
\le U_Z-L_X.
\label{eq:cmi-interval}
\end{equation}
Suppose a valid outer population-risk certificate gives $r-H(Y\mid X)\le\Gamma_X$. If a population-valid representation-preserving challenger demonstrates improvement $\Delta_Z$ for the same risk, predictor classes, and event, then
\begin{equation}
I(Y;X\mid Z)\le\Gamma_X-\Delta_Z.
\label{eq:cmi-subtraction}
\end{equation}
With an input-ignorant Bayes reference and $I(Y;X)>0$, \cref{eq:maturity-interval} also bounds the retained information fraction
\[
\frac{I(Y;Z)}{I(Y;X)}.
\]
\end{corollary}

For restricted but nested decoder families, \cref{thm:nested-rep} remains exact for class-accessible predictive value. A Shannon interpretation adds explicit approximation errors between restricted optima and Bayes risks, and an empirical certificate acquires population meaning through a stated population lift. Finite-sample log-loss bounds use a declared probability floor or an appropriate tail theorem. Proper-score and finite-sample details appear in Appendix~\ref{app:information-details}; \cref{fig:crossfit} shows why cross-fitting and selectivity controls are essential when a representation reaches full sample rank.

\subsection{Squared-loss information analogue}
\label{sec:squared-info}

Assume $Y$ is square-integrable in a Hilbert space and predictor classes are unrestricted under squared loss. Let
\[
\mu_X=\E[Y\mid X],
\qquad
\mu_Z=\E[Y\mid Z].
\]

\begin{corollary}[Conditional-mean representation deficit]
\label{cor:squared}
For every measurable decoder $g(Z)$,
\begin{equation}
\E\norm{Y-g(Z)}^2-\E\norm{Y-\mu_X}^2
=
\underbrace{\E\norm{g(Z)-\mu_Z}^2}_{\text{decoder under-use}}
+
\underbrace{\E\norm{\mu_X-\mu_Z}^2}_{D_{\mathrm{rep}}}.
\label{eq:squared-decomp}
\end{equation}
Consequently, the sharp nested interval \cref{eq:rep-interval} bounds
\begin{equation}
D_{\mathrm{rep}}=\E\norm{\E[Y\mid X]-\E[Y\mid Z]}^2.
\label{eq:conditional-mean-deficit}
\end{equation}
For an objective using one-half squared loss, every term is multiplied by one-half.
\end{corollary}
\begin{proof}
Write $Y-g=(Y-\mu_X)+(\mu_X-\mu_Z)+(\mu_Z-g)$. Conditional-expectation orthogonality makes the three components pairwise orthogonal in $L_2$, giving Pythagoras.
\end{proof}

For the denoising experiments, this quantity measures the target-relevant conditional-mean value absent from the representation.

\subsection{Same-target bounded decisions}

The log-loss information certificate also controls any bounded decision problem about the same target under the same population.

\begin{theorem}[Bounded same-target decision transfer]
\label{thm:decision-transfer}
Assume regular conditional laws $P_{Y\mid X}$ and $P_{Y\mid Z}$ exist, $Y-X-Z$, and
\[
I(Y;X\mid Z)\le\Gamma<\infty.
\]
Let the action space be standard Borel, let $c(a,y)\in[0,B]$ be jointly measurable, and assume measurable $\varepsilon$-optimal decision rules exist (equivalently, the conditional Bayes envelope is measurable). Let $\mathcal R_c(X)$ and $\mathcal R_c(Z)$ be the corresponding Bayes risks when decisions may use $X$ and $Z$, respectively. With total variation defined by
$\operatorname{TV}(P,Q)=\sup_A|P(A)-Q(A)|$,
\begin{equation}
0\le\mathcal R_c(Z)-\mathcal R_c(X)
\le B\sqrt{\Gamma/2}.
\label{eq:decision-transfer}
\end{equation}
\end{theorem}
\begin{proof}
More information cannot increase Bayes risk, giving the left inequality. The Bayes envelope of a $[0,B]$-valued loss is $B$-Lipschitz in total variation. Hence
\[
\mathcal R_c(Z)-\mathcal R_c(X)
\le B\E\operatorname{TV}(P_{Y\mid X},P_{Y\mid Z}).
\]
Pinsker's inequality and Jensen's inequality give \citep{pinsker1964}
\[
\E\operatorname{TV}(P_{Y\mid X},P_{Y\mid Z})
\le\sqrt{\tfrac12\E\KL(P_{Y\mid X}\|P_{Y\mid Z})}
=\sqrt{I(Y;X\mid Z)/2}.
\]
\end{proof}

The theorem applies to the same target and distribution used in its conditional-information premise. A protected audit supplies the separate intended-task channel.

\section{Trainer Attainability, Closed-Loop Bounds, and Task Adequacy}
\label{sec:tracking}

After current Green, a stronger policy may construct a feasible target value $S<J(\theta_{\mathrm{G}})$. Two branches answer different questions. Direct adoption followed by current-state recertification tests actionability and renewed current-suite passage along the intervention path. Continuing the declared trainer toward $S$ tests attainability. Neither should be mislabeled as the other.

Every feasible stair yields the exact decomposition
\begin{equation}
J(\theta)-J^\star
\le
\underbrace{[J(\theta)-S]_+}_{\epsilon_T:\text{ tracking}}
+
\underbrace{S-J^\star}_{\epsilon_S:\text{ challenge strength}}.
\label{eq:tracking-strength}
\end{equation}
The new power and spectral results control $\epsilon_S$; a trainer theorem or measured continuation controls~$\epsilon_T$.

\begin{theorem}[Finite-time and target-relative PL stair attainment]
\label{thm:pl-attainment}
Suppose $J$ has $L$-Lipschitz gradient and gradient descent uses
\[
\theta_{k+1}=\theta_k-\gamma\nabla J(\theta_k),
\qquad 0<\gamma<2/L.
\]
Let $J_0>S+\varepsilon$, and suppose a compact pre-hit region $\mathcal R_{S,\varepsilon}$ contains every iterate before $J\le S+\varepsilon$. Define
\[
g_{S,\varepsilon}
=
\inf_{\theta\in\mathcal R_{S,\varepsilon}}\|\nabla J(\theta)\|,
\qquad
 a_\gamma=\gamma(1-L\gamma/2).
\]
If $g_{S,\varepsilon}>0$, then the stair is reached within
\begin{equation}
N_{S,\varepsilon}
\le
\left\lceil
\frac{J_0-S-\varepsilon}{a_\gamma g_{S,\varepsilon}^2}
\right\rceil
\label{eq:finite-hit}
\end{equation}
additional iterations. If instead
\[
\|\nabla J(\theta)\|^2
\ge
2\mu_S[J(\theta)-S]
\]
holds on the pre-hit band, $0<\gamma\le1/L$, and $0<\mu_S\gamma\le1$, then
\begin{equation}
J(\theta_n)-S
\le
(1-\mu_S\gamma)^n[J_0-S]
\label{eq:pl-hit}
\end{equation}
until the first hit.
\end{theorem}
\begin{proof}
Smoothness gives $J(\theta_{k+1})\le J(\theta_k)-a_\gamma\|\nabla J(\theta_k)\|^2$. The compact-region lower bound yields a fixed decrease before the hit and hence \cref{eq:finite-hit}. Under the target-relative PL inequality and $\gamma\le1/L$, the descent lemma gives a decrease of at least $(\gamma/2)\|\nabla J\|^2$; substitution and iteration yield \cref{eq:pl-hit}.
\end{proof}

The hypotheses concern the region traversed before the active stair is reached. Structured regimes can verify them analytically; Appendix~\ref{app:attainability-plugins} gives a stopped conditional-drift theorem and a noisy-SGD instantiation with an explicit tracking floor. A measured residual or a finite-budget barrier record provides the operational alternative when an analytic rate is unavailable.

\begin{proposition}[Closed-loop strength--tracking recurrence]
\label{thm:strength-tracking}
At event $j$, let $x_j=J(\theta_j)-J^\star$. Suppose the challenge returns a feasible stair satisfying
\[
S_{j+1}-J^\star\le\alpha x_j+\epsilon^C_j,
\qquad 0\le\alpha<1,
\]
and the next attained or terminal checkpoint satisfies
\[
J(\theta_{j+1})\le S_{j+1}+\epsilon^T_{j+1}.
\]
Then
\begin{equation}
x_{j+1}\le\alpha x_j+\epsilon^C_j+\epsilon^T_{j+1},
\label{eq:tracking-recurrence}
\end{equation}
and
\[
x_j\le\alpha^jx_0+
\sum_{i=0}^{j-1}\alpha^{j-1-i}(\epsilon^C_i+\epsilon^T_{i+1}).
\]
If both error sequences vanish, $x_j\to0$; if they are bounded by $\bar\epsilon_C$ and $\bar\epsilon_T$, the limsup is at most $(\bar\epsilon_C+\bar\epsilon_T)/(1-\alpha)$.
\end{proposition}
\begin{proof}
Combine the two assumed inequalities and iterate the affine recurrence.
\end{proof}

A power law gives the premise directly. If the budgeted challenge achieves $J(y_j)\le b_{B_j}(\theta_j)+\epsilon_j^C$, then
\begin{equation}
x_{j+1}
\le x_j-\Psi(B_j,x_j)+\epsilon_j^C+\epsilon_{j+1}^T.
\label{eq:power-dynamics}
\end{equation}
Linear power $\Psi(B,s)\ge cs$ gives geometric contraction to an error floor; polynomial power gives the corresponding sublinear rate. Exact trajectory prediction still depends on the Bellman endpoint operator.

\subsection{Task adequacy remains a separate bridge}

Let $L_{\mathrm{task}}$ be a population or protected-task risk on the same comparison class $\Theta$. Suppose, on an event of probability at least $1-\delta$,
\begin{equation}
\sup_{\theta\in\Theta}|L_{\mathrm{task}}(\theta)-J(\theta)|\le\epsilon_G.
\label{eq:uniform-transfer}
\end{equation}

\begin{theorem}[Optimization strength, tracking, and task adequacy]
\label{thm:task-transfer}
If a feasible stair satisfies $S-J^\star\le\epsilon_S$ and a checkpoint satisfies $[J(\theta)-S]_+\le\epsilon_T$, then on \cref{eq:uniform-transfer},
\begin{equation}
L_{\mathrm{task}}(\theta)-\inf_{\vartheta\in\Theta}L_{\mathrm{task}}(\vartheta)
\le\epsilon_S+\epsilon_T+2\epsilon_G.
\label{eq:task-bound}
\end{equation}
The infima need not be attained.
\end{theorem}
\begin{proof}
The uniform event gives $L_{\mathrm{task}}(\theta)\le J(\theta)+\epsilon_G$ and $\inf L_{\mathrm{task}}\ge J^\star-\epsilon_G$. Apply \cref{eq:tracking-strength}.
\end{proof}

A finite protected audit supplies direct finite-sample evidence; population performance follows from an additional population argument. In the QNN study, PSNR and SSIM are reported as protected task metrics and are never used to construct challengers.

\section{Experiments}
\label{sec:experiments}

The experiments are organized around the scientific questions posed by the certificate ladder. Exact or high-precision instances test identities and constants; current-state studies test whether the central bounds are numerically meaningful; QNN studies test the operational semantics. Numerical checks test constants, implementation, and nonvacuity alongside the analytic proofs. Full protocol and replay boundaries are recorded in Appendix~\ref{app:experiments}.

\subsection{A current-state nonzero-optimum certificate}
\label{sec:current-state-certificate}

A current-state nonvacuity study asks whether challenge-kernel geometry yields an informative \emph{excess-gap} certificate. A two-hidden-layer tanh feature extractor is trained on one split of the handwritten-digits data; its 65-dimensional augmented representation is then frozen on an independent 256-sample certification split, and scalar squared-loss heads are audited. The full head optimum is exact and has positive residual,
\[
J^\star=0.026182.
\]
The optimum residual is orthogonal to the reachable prediction subspace to numerical precision $7.64\times10^{-14}$, so \cref{thm:normal-residual} applies with $\xi=0$.

Two predeclared architecture-valid current suites are compared. The first partitions native feature coordinates into eight exact frozen-context head refits. The second uses twelve rank-18 prediction subspaces chosen by E-optimal design and maps each subspace back to an ordinary head-weight update through the representation pseudoinverse. The maximum relative materialization error is $2.20\times10^{-14}$; all candidate values are recomputed through the complete head. The native suite has coverage coefficient $1.94\times10^{-5}$. The designed suite has coefficient $0.08025$, a factor of $4.14\times10^3$ larger, without changing the complete head class. Across the fixed-objective continuation checkpoints, the designed certificate is within a median factor $4.78$ and a worst observed factor $6.06$ of the exact global head gap; the native certificate is roughly $2.1\times10^4$ times the gap. \Cref{fig:current-state-certificate} shows the exact gap, the two certificates, and the spectra that explain their radically different tightness.

\begin{figure}[t]
\centering
\includegraphics[width=\linewidth]{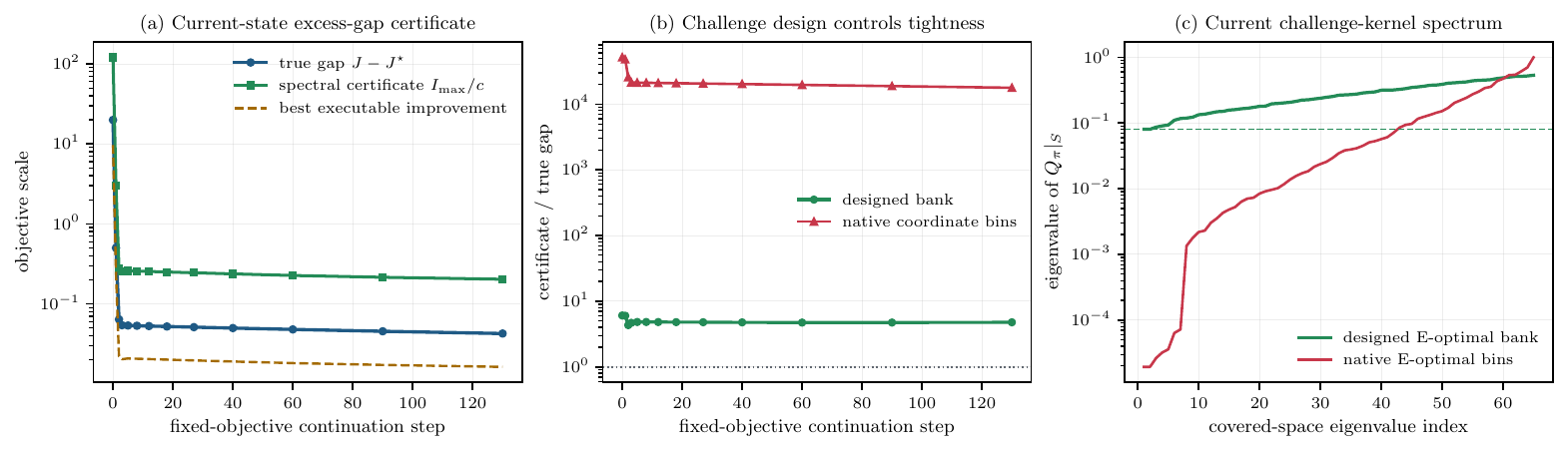}
\caption{Current-state excess-gap certification on a frozen nonlinear representation. (a) Exact materializable subspace challenges bound the true nonzero-optimum head gap throughout continuation. (b) The same head class can yield either a useful or a vacuous certificate depending on challenge-basis conditioning. (c) E-optimal design lifts the least-covered current prediction direction by more than three orders of magnitude. The study validates the normal-residual theorem and suite-design principle for the frozen nonlinear representation and declared head class.}
\label{fig:current-state-certificate}
\end{figure}

\subsection{Power, spectral, and information regressions}
\label{sec:theory-regressions}

\Cref{tab:regression-summary} summarizes the registered tests. The power suite checks the generalized inverse on 500 random finite systems, the two-versus-eleven-step scalar-incompleteness counterexample, and linear/polynomial rate envelopes. The exact collective neural example uses four rank-two blocks in a six-dimensional prediction space; the collective rank is six and $\lambda_{\mathrm{coll}}=0.08131$.

The nonlinear spectral suite includes a 12-sample tanh normalization test, a width-12,000 stable-core ReLU instance, width scaling, E-optimal basis design, and a trained width-1,024 tanh MLP on a deterministic handwritten-digits subset. In the ReLU registration, $\lambda_0=0.3129$, the observed stable-core eigenvalue is $0.2861$ versus theorem threshold $0.0782$, and the best actual block improvement is $11.05$ times the theorem floor. Across the trained MLP checkpoints, the smallest actual-to-certified ratio is $3.229$.

The information suite verifies exact log, Brier, and squared-Gaussian decompositions; 10,000 random sharp-interval instances; a finite-class population lift; and witness subtraction. In the discrete log-loss model, $I(Y;X\mid Z)=0.04228$ nats and $M_{\mathrm{rep}}=0.8648$. A deliberately imperfect paired certificate still proves $I(Y;X\mid Z)\le0.06728$ and $M_{\mathrm{rep}}\ge0.7947$.

\begin{table}[t]
\centering
\small
\caption{Registered theorem-regression results. ``Passed'' means every programmed inequality or identity held at the declared numerical precision, providing an independent numerical regression of the analytic statement.}
\label{tab:regression-summary}
\begin{tabular}{p{0.25\linewidth}p{0.55\linewidth}p{0.10\linewidth}}
\toprule
Program & Headline registered result & Status \\
\midrule
Challenge power & 500 generalized-inverse trials; zero violations; identical scalar surfaces with 2 vs. 11 intervention steps & Passed \\
Collective affine coverage & Four rank-2 blocks; collective rank 6; $\lambda_{\mathrm{coll}}=0.08131$ & Passed \\
Stable-core ReLU & Actual/theorem improvement ratio $11.05$; gate and spectrum margins satisfied & Passed \\
Trained tanh MLP & Seven checkpoints; minimum actual/theorem ratio $3.229$ & Passed \\
Information sufficiency & Exact log/Brier/squared identities; 10,000 sharp-interval trials; zero violations & Passed \\
Witness subtraction & True CMI $0.04228$; certified upper bound $0.06728$ & Passed \\
\bottomrule
\end{tabular}
\end{table}

\subsection{Quantized denoising: executable diagnosis and repair}
\label{sec:qnn}

The principal neural case study uses an eight-block residual denoiser at FP32, W8A8, W4A4, W2A4, and W1A2, with three seeds per regime. Hidden residual blocks use the declared quantization; the first and final affine maps remain full precision. Training and certification share the same fixed empirical mean-squared objective on 4,096 cached noisy/clean patches. A disjoint 512-patch audit and three crops from entirely held-out source images provide protected evidence.

The Core current suite contains: (i) an exact current-representation output projection materialized in the exact architecture; (ii) a late checkpoint-waypoint route followed by exact projection; and (iii) ordinary continuation matched to the route's gradient-batch count. A reinitialized-block no-waypoint control is withheld from live status and included in the stricter Full suite. After every adopted witness, the complete current suite is rerun before Green is issued.
This experiment instantiates selected levels of the general hierarchy rather than enumerating all internal subsets: the current suite contains a one-waypoint route, while the stronger post-passage policy adds two-block and two-waypoint routes, including route-and-release.

\Cref{tab:qnn-summary,fig:qnn-status} show a precision-dependent transition in the registered study. FP32 and W8A8 endpoints are Green with median known headroom below $0.003\%$. W4A4 and W2A4 are Yellow with median headroom $38.5\%$ and $10.8\%$; W1A2 is Red with median headroom $23.9\%$. Protocol-correct intervention improves the certified objective and protected image metrics in all nine low-precision runs. All 15 endpoints pass the Core suite after intervention; under the Full suite, only one of three W1A2 endpoints passes.
These five-regime, three-seed summaries are descriptive rather than population estimates: the artifact reports every seed, including the divergent binary trajectory, and no population-level confidence interval is inferred from three runs.
\Cref{fig:qnn-repair} separates the protected task gain from the stricter current-state suite outcome.

\begin{table}[t]
\centering
\small
\caption{Quantized-denoising centerpiece. Values are medians over three seeds; Core/Full counts are current-state passage after intervention. Compression reports parameter storage at the declared precision; runtime remains hardware- and kernel-dependent.}
\label{tab:qnn-summary}
\begin{tabular}{lccccc}
\toprule
Precision & Standard status & Headroom & $\Delta$PSNR & Core/Full & Storage \\
\midrule
FP32  & Green  & $0.003\%$ & $0.00$ dB & $3/3$ / $3/3$ & $1.0\times$ \\
W8A8  & Green  & $0.003\%$ & $0.00$ dB & $3/3$ / $3/3$ & $3.9\times$ \\
W4A4  & Yellow & $38.5\%$ & $+1.88$ dB & $3/3$ / $3/3$ & $7.6\times$ \\
W2A4  & Yellow & $10.8\%$ & $+0.80$ dB & $3/3$ / $3/3$ & $14.3\times$ \\
W1A2  & Red    & $23.9\%$ & $+1.20$ dB & $3/3$ / $1/3$ & $25.5\times$ \\
\bottomrule
\end{tabular}
\end{table}

\begin{figure}[t]
\centering
\includegraphics[width=\linewidth]{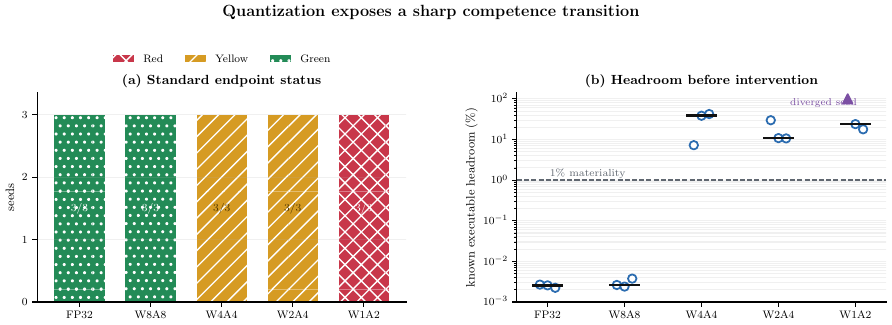}
\caption{Standard-training status and executable headroom. High-precision endpoints pass the declared Core suite, while every low-precision run contains material complete-model headroom. The bars report the registered medians over all three declared seeds; the complete per-seed ledger, including the divergent binary trajectory, is retained in the reproducibility artifact.}
\label{fig:qnn-status}
\end{figure}

\begin{figure}[t]
\centering
\includegraphics[width=\linewidth]{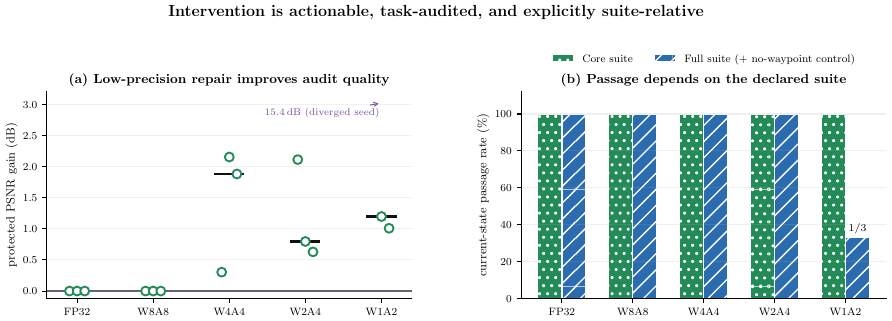}
\caption{Intervention is actionable, task-audited, and suite-relative. Every low-precision repair improves held-out PSNR. Every endpoint passes the predeclared Core suite after current-state recertification, whereas two binary endpoints fail the Full suite containing the protected no-waypoint control.}
\label{fig:qnn-repair}
\end{figure}

\subsection{Identical-checkpoint controls: transient versus structural failure}
\label{sec:qnn-controls}

A 75-endpoint control restores the identical saved checkpoint, optimizer state, scheduler prefix, deterministic data order, and random-number state for every branch. Same-step continuation, exact-head continuation, and alternating backbone-head training use the same 20-epoch learning-rate tail and 1,280 backbone-gradient batches; matched-wall continuation starts from the same checkpoint and uses a median 2,688 batches.

The precision regimes exhibit three distinct mechanisms. At W4A4, ordinary continuation earns Core and Full Green in all five seeds. At W2A4, same-step and matched-wall continuation each pass in only one of five seeds; exact-head continuation passes Core in all five and Full in four, while alternating training passes both in all five. At W1A2, exact-head and alternating updates stabilize catastrophic trajectories, but every endpoint remains Red relative to the fixed executable reference. Thus the certificate distinguishes transient undertraining, structural backbone-head coordination failure, and a deeper binary-precision failure to meet the declared competence standard; the identical-checkpoint comparison is summarized in \cref{fig:qnn-control}.

\begin{figure}[t]
\centering
\includegraphics[width=\linewidth]{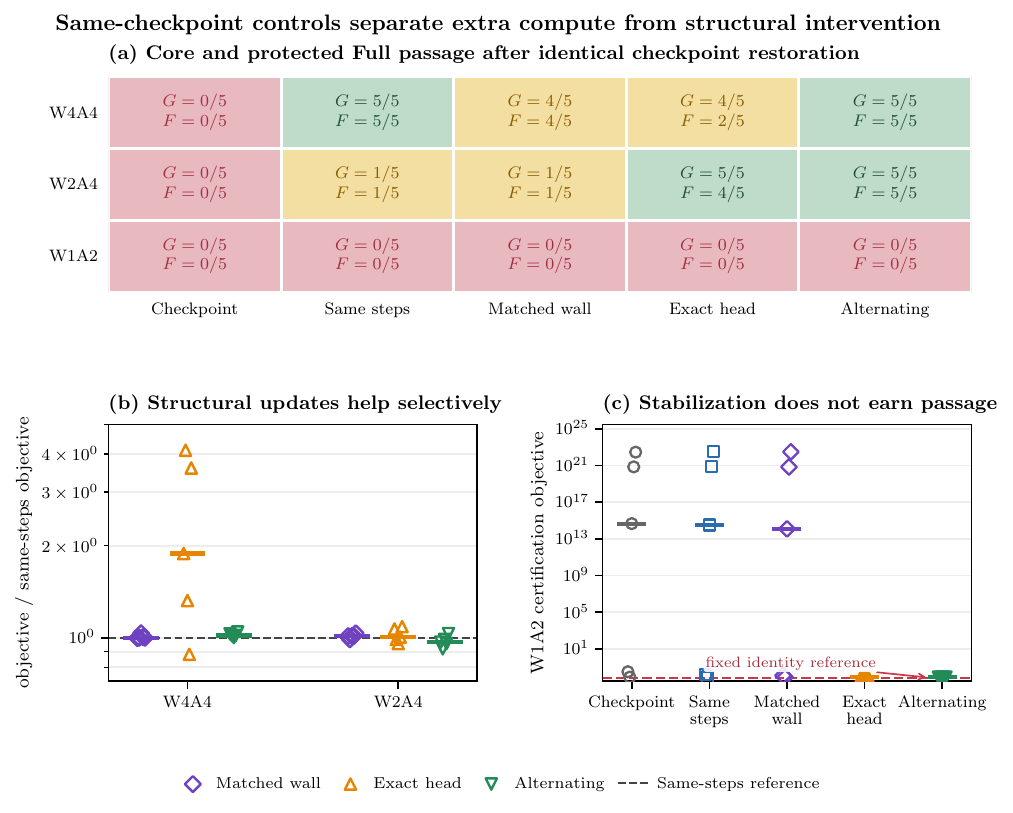}
\caption{Same-checkpoint controls separate extra compute from structural intervention. Ordinary continuation resolves W4A4. W2A4 is most reliably repaired by alternating backbone-head training despite wall-matched continuation using more backbone-gradient batches. W1A2 stabilization leaves every endpoint above the fixed executable reference.}
\label{fig:qnn-control}
\end{figure}

Two post-Green cases separately test attainability. A W2A4 matched-continuation target $3.94\%$ below the Green checkpoint is reached by a fresh same-family continuation in three epochs. A W4A4 two-waypoint route-and-release target $1.77\%$ lower remains unattained over the declared ten-epoch continuation budget, although direct adoption passes current-state recertification. The latter records a trainer barrier over the declared finite tracking budget.
\Cref{fig:qnn-attainability} displays the reached and unreached targets under the same evidentiary conventions.

\begin{figure}[t]
\centering
\includegraphics[width=\linewidth]{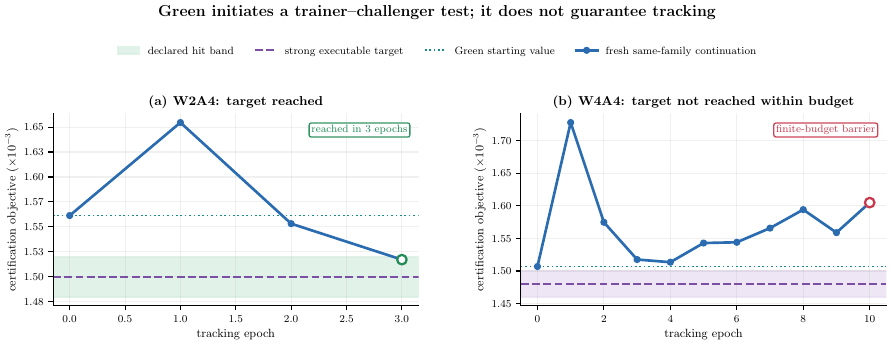}
\caption{Green begins a trainer--challenger experiment in which tracking is measured directly. One executable target is reached by fresh same-family continuation, while another remains unattained within budget. Both are actionable complete models; only the trainer dynamics differ.}
\label{fig:qnn-attainability}
\end{figure}

The study provides a diagnostic comparison within the declared architecture, trainer, and suites: under the declared architecture, objective, trainer, and suites, different precision regimes fail for different reasons.

\section{Operational Protocol, Reporting, and Scaling}
\label{sec:protocol}

A reported status is interpretable only together with the problem, the calls that were run, and the additional bridge supporting the claimed level. \Cref{tab:certificate-record} gives the minimum interoperable record.

\begin{table}[t]
\centering
\small
\caption{Minimum certificate record. These fields prevent evidence from being silently transferred across objectives, model classes, checkpoints, or probability spaces.}
\label{tab:certificate-record}
\begin{tabular}{>{\raggedright\arraybackslash}p{0.20\linewidth}>{\raggedright\arraybackslash}p{0.35\linewidth}>{\raggedright\arraybackslash}p{0.35\linewidth}}
\toprule
Layer & Required record & Scientific role \\
\midrule
Problem identity & exact class, architecture state, data, objective, regularizer, precision, reference & fixes the optimization problem to which every inequality refers \\
Current execution & mandatory calls, completion flag, candidate values, tolerances, solver and numerical status & distinguishes current passage from historical evidence \\
Executable evidence & candidate states, replay recipe, hashes, retained frontier, strict improvements & makes every failure witness independently checkable \\
Coverage & total budget, current decrease operators or proof-bearing floor, covered subspace, solver error, calibration population & licenses resource-qualified gap or population language \\
Representation and task & nested predictor classes, brackets, decoder headroom, transfer theorem or protected audit & separates optimization, representation, and intended-use claims \\
Resources & marginal and total-from-scratch compute, memory, wall time, data access & prevents a cheap correction from being mislabeled total-resource dominance \\
\bottomrule
\end{tabular}
\end{table}

A practical implementation follows seven steps:
\begin{enumerate}[leftmargin=1.65em,itemsep=3pt]
\item Declare the certified class and objective, training objective if different, feasible reference, protected task metric, and every state needed for deterministic execution.
\item Declare challenge information access, mandatory and optional calls, solver budgets, restart policies, tolerances, and marginal versus total cost conventions.
\item Materialize every candidate in the complete architecture and reevaluate it under the same objective. Reject mismatched, nonfinite, or unreplayable outputs.
\item Issue current status only after every mandatory call completes. Retain earlier strict witnesses independently of the current status.
\item Attach global-gap language only to current verified coverage quantities or a proof-bearing lower floor. Treat initialization spectra and unverified linearizations as descriptive geometry.
\item After adoption, rerun the current suite. Label intervention paths separately from trainer-driven staircases.
\item Report optimization, representation, population, and protected-task evidence in separate fields.
\end{enumerate}

\paragraph{Tolerance selection.}
The Green tolerance $\tau_G$ should dominate deterministic reevaluation and numerical uncertainty and then add the predeclared materiality threshold relevant to the use case. The acceptance margin $\tau_{\mathrm{acc}}$ should exceed solver and comparison uncertainty so that every new stair is a strict executable improvement. The hit tolerance $\tau_{\mathrm{hit}}$ describes trainer attainment and need not equal either quantity. Consequential reports should show status and bound sensitivity over a small predeclared grid of tolerances rather than selecting a favorable threshold after inspection.

\paragraph{Tiered cost and scaling.}
Coverage is not free. The intended deployment is tiered: inexpensive mandatory calls run online, stronger routes are triggered by passage events, and proof-bearing or exhaustive modules are reserved for high-consequence checkpoints. In the exhaustive QNN study, standard training used $0.54$ accelerator-hours and certification/intervention used $0.60$ accelerator-hours on an A100; the audit was therefore comparable to training cost, not negligible. Sparse E-optimal bases, cached graph-cut states, low-dimensional exact motifs, and event-triggered execution are the principal scaling mechanisms. Every report should give both marginal audit cost and total-from-scratch cost.

\paragraph{Scope.}
The primary certificates concern a fixed empirical objective. The squared-loss global-gap theorems apply only to declared classes satisfying their coverage premises. The stable-core ReLU theorem covers a local gate-stable region of a two-layer network. Representation statements are class relative; Shannon/Bayes interpretations require population-complete predictor classes. Robustness, fairness, safety, calibration, legal conformity, and downstream utility require their own protected audits or transfer theorems.

\section{Discussion and Conclusion}
\label{sec:conclusion}

Training completion becomes scientifically testable when a checkpoint is compared with complete alternatives that the same certified architecture can actually execute. A lower-valued challenger is a constructive counterexample whose complete model state can be replayed. Retaining such witnesses creates monotone pathwise evidence, while mandatory current execution distinguishes present suite passage from historical success. Green-0 then turns passive monitoring into a controlled trainer--challenger experiment.

The central theoretical bridge is coverage. Budgeted challenge power states which levels of suboptimality a resource-qualified family must expose, and its generalized inverse $E(B,\tau)$ gives the largest global gap compatible with passage. For squared loss, current decrease operators make this bridge computable. Uniform spectral coverage protects all relevant residual directions, whereas the realized-residual coefficient
\[
\kappa_{\mathrm{cur}}
=
\frac{e^\top Q_\pi e}{\|e\|^2}
\]
measures coverage along the error actually present at the checkpoint. In the ResNet-18 study, eight current internal challenges cover all $240$ audited output directions, and the realized-residual certificate bounds the known true gap within factors of $1.74$--$3.02$. For nonzero optima, the normal-residual theorem provides the corresponding excess-gap extension under its stated compatibility condition. E-optimal design identifies complementary challenge bases, while exact affine motifs and stable-core ReLU geometry provide concrete neural regimes.

The exact non-global staircase fixes the converse boundary. Conditional solver exactness, repeated stair attainment, and trajectory convergence can all hold while an uncovered direction preserves a positive global gap. Challenge-closed optimality is therefore the correct generic endpoint of an increasingly rich audit. Global-gap language requires an explicit mechanism---such as spectral coverage, contraction, or a proof-bearing lower floor---that controls the declared certified class.

Once predictive optimization uncertainty is controlled, paired certificates separate decoder under-use from representation insufficiency. Representation-preserving headroom can be separated from the outer predictive uncertainty to bound the class-relative representation deficit. Under Bayes-complete population log loss this deficit is $I(Y;X\mid Z)$; under squared loss it is the conditional-mean predictive value available from $X$ but absent from $Z$. The certificate ladder therefore connects optimization, representation quality, and intended-task evidence without conflating the assumptions needed at each level.

The quantized-denoising studies illustrate the operational value of this hierarchy. Complete candidates expose low-precision headroom, adoption followed by current-state recertification tests actionability, and fresh continuation tests attainability. Same-checkpoint controls distinguish additional training from structural backbone--head coordination: ordinary continuation resolves W4A4, structured coordination is decisive for W2A4, and W1A2 remains below the declared competence standard after stabilization.

The reusable object is an executable certificate record: problem identity, complete challenger states, strict improvements, current execution status, challenge budgets and solver evidence, coverage or lower-floor mechanisms, trainer-tracking evidence, representation brackets, and protected-task audit. Such a record turns a training plateau from an ambiguous visual pattern into a sequence of verifiable scientific statements. The resulting framework opens several directions, including scalable challenge-basis design, coverage certificates for broader neural objectives and architectures, and independent challenge-based training assurance.

\section{Disclosures and Reproducibility}
\label{sec:reproducibility}

A general-purpose large language model was used as an assistive tool for drafting and editing, code generation and debugging, and proof exploration. The accompanying artifact contains the complete LaTeX source, figure-generation code and data, theorem-regression outputs, proof and reference audits, manifests, checksums, and resumable notebooks. Separate directories collect the challenge-power, spectral-coverage, information-sufficiency, exact-block, ResNet current-certificate, and QNN reproducibility materials. The ResNet integration directory contains the corrected resumable notebook, validation record, certified operators, and derived uniform-versus-realized certificate table used in Section~\ref{sec:resnet-current-certificate}. Every numerical claim is linked to an included result record or to the identified external replay archive. The historical QNN state archive is identified by its canonical filename and checksum protocol; the manuscript package contains the code, metadata, compact summaries, and exact replay ledgers needed for interpretation.

\acks{
This work was supported in part by the Office of Naval Research under Award
N00014-22-1-2666 through the Mathematical and Resource Optimization Program.
Any opinions, findings, conclusions, or recommendations expressed in this
material are those of the authors and do not necessarily reflect the views of
the Office of Naval Research.

Mojtaba Soltanalian is the founder of AI-Certified.org. This work was not
funded by AI-Certified.org. Farhang Yeganegi and Arian Eamaz declare no
competing interests.
}

\appendix
\section{Waypoint Conditioning, Route Mechanics, and Resource Design}
\label{app:waypoint-mechanism}
\label{app:route-resource}

This appendix develops the mathematical tools behind architecture-native waypoint and route challenges. The local objectives are construction devices; the complete terminal model remains the certificate. The results identify when fixed-waypoint optimization separates, how local errors propagate through an assembled route, and how challenge compute can be allocated and routes selected before full materialization.

\subsection{Proof of the deep-product separation}

\begin{proof}[Proof of \cref{thm:waypoint-advantage}]
Permutation symmetry preserves the balanced trajectory $w_r(t)=s(t)$. Since
\[
\frac{\partial J}{\partial w_r}
=
\left(\prod_{\ell=1}^{K}w_\ell-1\right)
\prod_{\ell\neq r}w_\ell,
\]
gradient flow gives
\[
\dot s=(1-s^K)s^{K-1}.
\]
Until $s$ reaches $2\alpha\le1$,
$\dot s\le s^{K-1}$, and therefore
\[
T_{2\alpha}
\ge
\int_\alpha^{2\alpha}s^{1-K}\,ds
=
\frac{1-2^{2-K}}{K-2}\alpha^{2-K}.
\]
One full gradient evaluation computes $K$ scalar partial derivatives. A unit-step balanced update is
\[
s_1=\alpha+(1-\alpha^K)\alpha^{K-1}.
\]
For $K\ge3$ and $\alpha\le1/2$, $\alpha^{K-1}\le\alpha^2\le\alpha/2$, so $s_1\le3\alpha/2$ and
\[
J(w^{(1)})
=
\frac12(1-s_1^K)^2
\ge
\frac12\left[1-\left(\frac32\alpha\right)^K\right]^2.
\]
For the waypoint ladder, each local objective is $\phi_r(w_r)=\tfrac12(w_r-1)^2$, and one unit-step update gives $w_r\leftarrow1$. After $K$ scalar derivatives the assembled product equals one and its terminal loss is zero.
\end{proof}

\subsection{Fixed-waypoint separability and graph-cut propagation}

For fixed states $Z_0,\ldots,Z_R$, segment maps $F_r(\cdot;\phi_r)$, and local discrepancies $d_r$, define
\[
\Phi_Z(\phi_1,\ldots,\phi_R)
=
\sum_{r=1}^{R}d_r(F_r(Z_{r-1};\phi_r),Z_r).
\]

\begin{proposition}[Fixed-waypoint separability]
\label{prop:fixed-waypoint-separable}
Assume the segment feasible sets are nonempty, the local infima are finite, and the joint feasible set is a Cartesian product after all parameter sharing and global constraints have been encoded. Then
\[
\inf_{\phi_1,\ldots,\phi_R}\Phi_Z
=
\sum_{r=1}^{R}\inf_{\phi_r}d_r(F_r(Z_{r-1};\phi_r),Z_r).
\]
\end{proposition}
\begin{proof}
Every feasible joint choice is bounded below by the sum of the separate infima. Conversely, select each block within $\varepsilon/R$ of its own infimum and combine the selections; Cartesian feasibility gives the reverse inequality as $\varepsilon\downarrow0$.
\end{proof}

Global solution of each segment therefore solves the anchored surrogate. The assembled route receives a separate terminal evaluation. Its state deviation satisfies the following recursion.

\begin{theorem}[Graph-cut error propagation]
\label{thm:graph-propagation}
Let each materialized segment be $L_r$-Lipschitz in its input and define
\[
\delta_r
=
\|F_r(Z_{r-1};\widehat\phi_r)-Z_r\|.
\]
Execute the route from $\widehat Z_0=Z_0$ and set $\widehat Z_r=F_r(\widehat Z_{r-1};\widehat\phi_r)$. Then
\begin{equation}
\|\widehat Z_R-Z_R\|
\le
\sum_{r=1}^{R}\delta_r\prod_{j=r+1}^{R}L_j.
\label{eq:graph-propagation}
\end{equation}
The same statement holds for a computation graph when each cut is represented by the product-space state of every tensor crossing it.
\end{theorem}
\begin{proof}
Let $e_r=\|\widehat Z_r-Z_r\|$. Adding and subtracting $F_r(Z_{r-1};\widehat\phi_r)$ gives $e_r\le L_re_{r-1}+\delta_r$. Since $e_0=0$, induction yields \cref{eq:graph-propagation}.
\end{proof}

An early residual is amplified by every downstream Lipschitz factor. This weighting explains why route quality depends jointly on local fit, cut location, and continuation stability.

\subsection{Finite solver budgets and route selection}

\begin{corollary}[Finite-budget route guarantee]
\label{cor:finite-budget-route}
If budget $B_r$ guarantees $\delta_r(B_r)\le\rho_r(B_r)$, then
\[
\|\widehat Z_R-Z_R\|
\le
\sum_{r=1}^{R}\rho_r(B_r)\prod_{j=r+1}^{R}L_j.
\]
\end{corollary}

When the weighted contribution of segment $r$ is bounded by $a_re^{-\gamma_rB_r}$, continuous challenge compute has an exact allocation.

\begin{proposition}[Water-filling allocation of challenge compute]
\label{prop:budget-waterfill}
For $a_r,\gamma_r>0$, the convex program
\[
\min_{B_r\ge0}\sum_{r=1}^{R}a_re^{-\gamma_rB_r}
\quad\text{subject to}\quad
\sum_{r=1}^{R}B_r\le B
\]
has solution
\[
B_r^\star
=
\frac1{\gamma_r}\bigl[\log(a_r\gamma_r)-\log\lambda\bigr]_+,
\]
where, for $B>0$, the unique $\lambda>0$ makes the budget constraint active. At $B=0$, every $B_r^\star=0$.
\end{proposition}
\begin{proof}
KKT stationarity on an active coordinate gives $-a_r\gamma_re^{-\gamma_rB_r}+\lambda=0$. Complementary slackness truncates inactive coordinates at zero. The total assigned budget is strictly decreasing in $\lambda$ on the active set.
\end{proof}

For $M$ ordered cuts, an additive proposal surrogate permits exact route selection by dynamic programming.

\begin{theorem}[Optimal fixed-waypoint route by dynamic programming]
\label{thm:route-dp}
Let $c_{ij}$ be the declared fixed-waypoint surrogate cost of the segment from cut $i$ to cut $j$, $0\le i<j\le M$. Define
\[
D(j,\ell)
=
\min_{0=i_0<\cdots<i_\ell=j}
\sum_{r=1}^{\ell}c_{i_{r-1},i_r}.
\]
Then
\[
D(j,\ell)=\min_{i<j}\{D(i,\ell-1)+c_{ij}\},
\]
with $D(0,0)=0$ and unreachable states equal to $+\infty$. Best routes using at most $m$ segments---equivalently, at most $m-1$ internal waypoints---are recovered in $O(M^2m)$ time after the edge costs are available.
\end{theorem}
\begin{proof}
Every $\ell$-segment route to $j$ has a unique penultimate cut $i$. Its prefix must solve $(i,\ell-1)$; otherwise replacement by a better prefix lowers the route cost. Conversely, appending $(i,j)$ to an optimal prefix gives a feasible route.
\end{proof}

The selected proposals are assembled and reevaluated through the complete objective because the additive surrogate orders proposals rather than certificate values.

\subsection{Nested portfolios and resource-qualified competence}

A richer individual route can improve or worsen after composition. Monotonicity belongs to retained nested families. Let $\mathcal C_{m,b}(t)$ contain every valid challenger discovered by time $t$ with at most $m$ internal waypoints and budget at most $b$, and define
\[
V_{m,b,t}=\min_{C\in\mathcal C_{m,b}(t)}J(\theta_C).
\]
When the families are nested in $m$, $b$, and $t$, the retained values are nonincreasing in every index. Figure~\ref{fig:challenge-lattice} separates this memory property from the behavior of individual routes.

\begin{figure}[t]
\centering
\includegraphics[width=\textwidth]{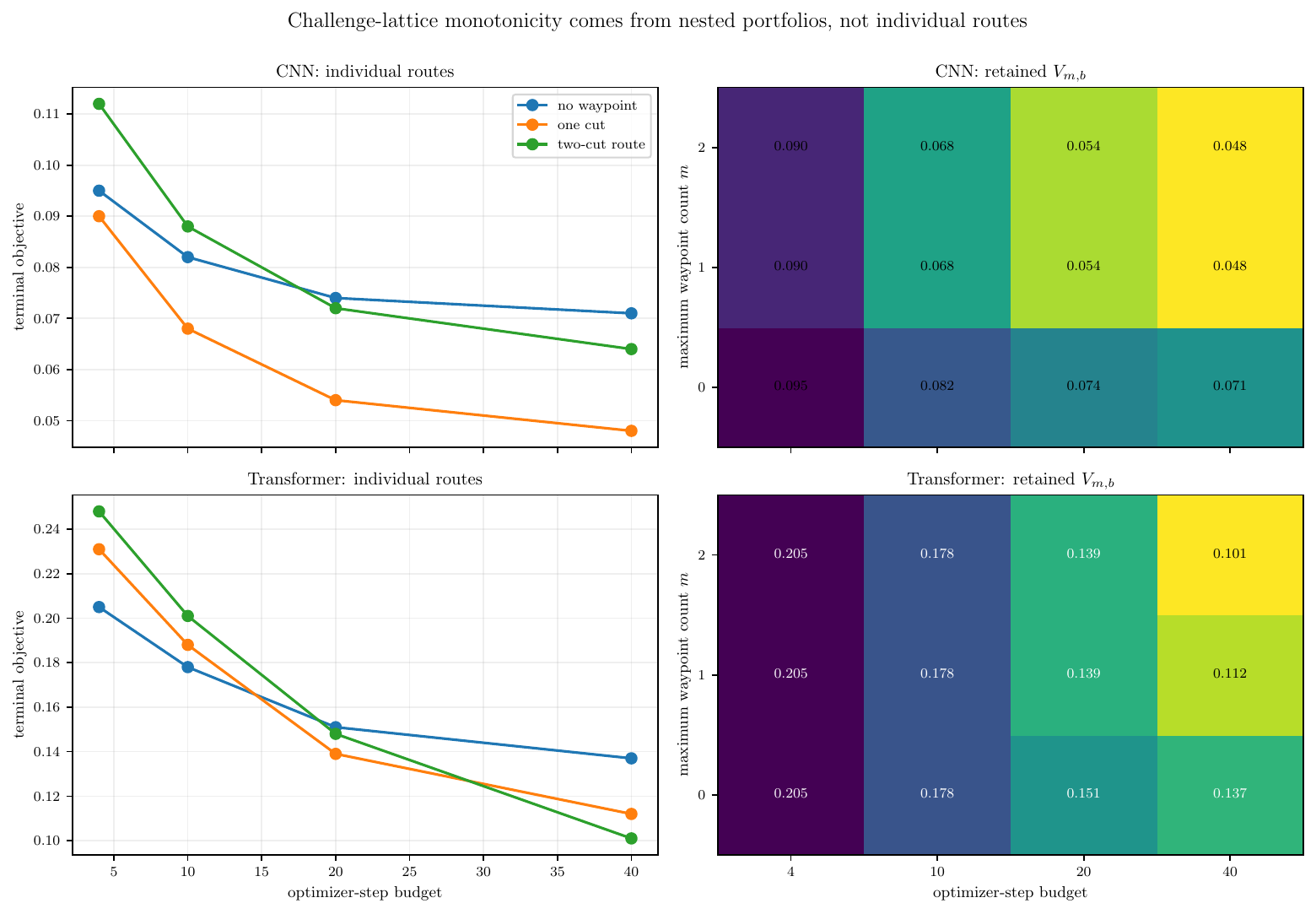}
\caption{Nested retained portfolios are monotone even when individual routes are not. Additional waypoints can improve or worsen one executed CNN or transformer candidate, while the minimum over all no-richer configurations yields the monotone value hierarchy used by the monitor. The figure distinguishes route behavior from retained evidence memory.}
\label{fig:challenge-lattice}
\end{figure}

Let $c_t\in\mathbb R_+^d$ be the audited run's cumulative resource vector and let $\operatorname{cost}_{\mathrm{tot}}(C)$ use the same total-from-scratch origin, including checkpoint production and challenge overhead.

\begin{proposition}[Resource-matched Pareto dominance]
\label{prop:resource-pareto}
If
\[
\operatorname{cost}_{\mathrm{tot}}(C)\preceq c_t,
\qquad
J(\theta_C)<J(\theta_t),
\]
and both resource vectors include every attributed cost from the common origin, then the audited checkpoint is Pareto-dominated in the declared resource coordinates and empirical objective.
\end{proposition}
\begin{proof}
The challenger uses no more of any declared resource and achieves a strictly lower objective.
\end{proof}

A same-checkpoint correction usually has total cost $c_t+\operatorname{cost}_{\mathrm{marg}}(C)$. Its marginal cost measures the added audit or repair burden; total-resource dominance uses the common from-scratch origin.

\subsection{Challenge profiles and executable symmetries}

A challenge family is summarized by a power profile rather than by its best returned value alone. For declared failure populations $Q_1,\ldots,Q_s$, an adequate-state population $Q_{\mathrm{ok}}$, materiality $\delta$, solver-tightness error $\bar\epsilon_C$, and total cost $T_C$, define
\[
\Pi(C)=\bigl(
\operatorname{Power}_{C}(Q_1,\delta),\ldots,
\operatorname{Power}_{C}(Q_s,\delta),
1-\operatorname{Power}_{C}(Q_{\mathrm{ok}},\delta),
\bar\epsilon_C,T_C
\bigr).
\]
The coordinates record failure detection, non-disruption of adequate states, solver tightness, and resource demand. Their partial order permits a cheap broad heuristic, an expensive exact block, and a targeted route to occupy different useful positions in the portfolio.

Waypoint comparison also respects proved executable symmetries. If $\mathcal G_r$ is a symmetry group at cut $r$, define
\[
d_r^{\mathcal G}(U,V)=\inf_{g\in\mathcal G_r}\|U-g(V)\|^2.
\]
Neuron, channel, and attention-head permutations, together with selected compensated positive scalings, are valid when the downstream parameters receive the matching inverse transformation. An arbitrary orthogonal rotation across a coordinatewise nonlinearity generally changes the computation. A safe transformation therefore lies in the activation symmetry group
\[
\mathcal G_\Omega=\{Q:\Omega(Qz)=Q\Omega(z)\ \text{for every }z\},
\]
occurs between consecutive linear maps with the compensating inverse inserted, or is represented by an explicit executable encoder--decoder pair whose distortion enters the ledger. Similarity measures can rank proposals; exact materialization determines certificate validity.

The distinction between feasibility and diagnostic power is visible in \cref{fig:challenge-power-calibration}. Purposeful exact or descent-certified challenges expose early failure states, while arbitrary feasible proposals rarely do. Near a true optimum their improvement probability falls to zero, as a well-calibrated challenge should.

\begin{figure}[t]
\centering
\includegraphics[width=0.90\textwidth]{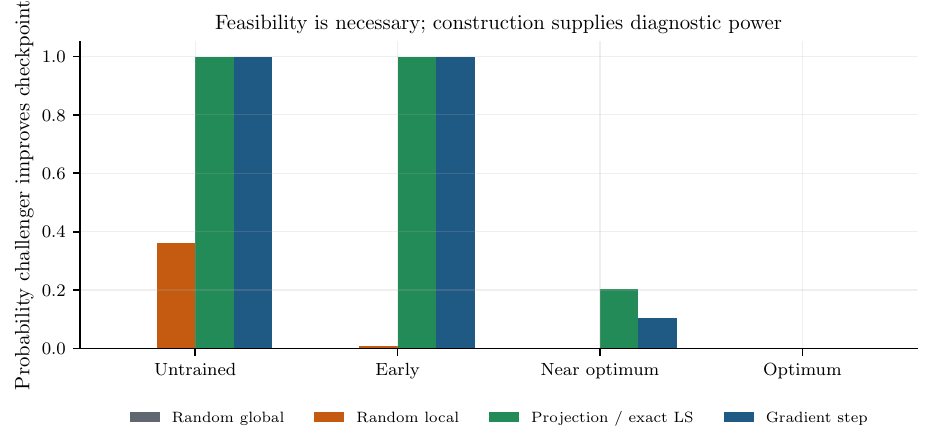}
\caption{Feasibility validates a realized witness; construction supplies diagnostic power. Across $4{,}000$ noiseless linear-regression problems, exact least squares and a certified gradient step expose all untrained and early checkpoints, while arbitrary feasible proposals rarely improve them. At the exact optimum, no challenger reports a material improvement.}
\label{fig:challenge-power-calibration}
\end{figure}

\section{Proof of the Stable-Core ReLU Theorem}
\label{app:relu-proof}

This section proves \cref{thm:relu-green} with all normalizations exposed. Let $D_j^0=D(w_j^0)$ and
\[
s_j=\mathbf1\left\{\min_i|x_i^\top w_j^0|\ge2\gamma\right\}.
\]
Define the truncated population kernel
\[
H_\gamma=\E[s(w)D(w)XX^\top D(w)].
\]
For $w\sim N(0,I_d)$ and unit $x_i$, $x_i^\top w\sim N(0,1)$ and
\[
\Prb\{|x_i^\top w|<2\gamma\}
\le \frac{4\gamma}{\sqrt{2\pi}}.
\]
A union bound and $\opnorm{DXX^\top D}\le R$ imply
\begin{align}
\opnorm{H^\infty-H_\gamma}
&\le R\Prb\{\min_i|x_i^\top w|<2\gamma\}\\
&\le \frac{4n\gamma R}{\sqrt{2\pi}}
\le\frac{\lambda_0}{2}.
\end{align}
By Weyl's inequality,
$\lambda_{\min}(H_\gamma)\ge\lambda_0/2$.

The matrices
$X_j=s_jD_j^0XX^\top D_j^0$
are independent, positive semidefinite, bounded by $RI$, and satisfy $\E X_j=H_\gamma$. The lower-tail matrix Chernoff inequality at relative deviation $1/2$ gives
\begin{align}
\Prb\left\{
\lambda_{\min}\!\left(\frac1p\sum_{j=1}^pX_j\right)
\le\frac12\lambda_{\min}(H_\gamma)
\right\}
&\le n\exp\left(-\frac{p\lambda_{\min}(H_\gamma)}{8R}\right)\\
&\le n\exp\left(-\frac{p\lambda_0}{16R}\right).
\end{align}
Under \cref{eq:relu-width1}, this probability is at most $\delta$. Therefore
\[
\lambda_{\min}(H_{\mathrm{st}})
=\lambda_{\min}\!\left(\frac1p\sum_jX_j\right)
\ge\frac{\lambda_0}{4},
\]
which proves \cref{eq:relu-kernel}. Notice that the proof needs only the eigenvalue inequality; it does not assert the stronger relative Loewner relation
$H_{\mathrm{st}}\succeq H_\gamma/2$.

Fix a checkpoint satisfying the current-region assumptions. For a stable neuron,
\[
\nabla_{w_j}J(W)
=
\frac{a_j}{n\sqrt p}X^\top D_j e,
\qquad
\norm{\nabla_{w_j}J(W)}
\le\frac{\sqrt R\norm e}{n\sqrt p}.
\]
The proposed displacement obeys
\[
\eta\norm{\nabla_{w_j}J(W)}
\le
\frac{mn}{2R}\frac{\sqrt R E_{\max}}{n\sqrt p}
=
\frac{mE_{\max}}{2\sqrt R\sqrt p}
\le\frac\gamma2.
\]
The current stable neuron lies within $\gamma/2$ of initialization; every point on the current-to-proposed segment therefore lies within $\gamma$ of $w_j^0$. Since the initialization margin is at least $2\gamma$ and $\|x_i\|=1$, each stable gate keeps its sign throughout the step.

Within the fixed gate cell, let $G_r$ be the residual Jacobian with respect to the stable neurons in balanced bin $r$. Every other parameter, including all unstable neurons, remains fixed. Because each bin has size at most $\lceil p/m\rceil$ and $p\ge m$,
$\lceil p/m\rceil\le2p/m$, so
\[
G_rG_r^\top
=
\frac1p\sum_{j\in r}s_jD_j^0XX^\top D_j^0
\preceq\frac{2R}{m}I.
\]
The block Hessian of $J$ is $G_r^\top G_r/n$, hence the block gradient is $L_r$-Lipschitz with $L_r\le2R/(mn)$. The choice $\eta=mn/(2R)$ satisfies $\eta L_r\le1$, so the descent lemma gives
\[
J(W)-J(W^{(r,+)})
\ge\frac\eta2\norm{\nabla_rJ(W)}^2.
\]
Taking the maximum, averaging over the $m$ balanced bins, and using
\[
\sum_r\norm{\nabla_rJ(W)}^2
=
\frac1{n^2}e^\top H_{\mathrm{st}}e
\]
yields
\begin{align}
\max_r[J(W)-J(W^{(r,+)})]
&\ge\frac\eta{2m}\sum_r\norm{\nabla_rJ(W)}^2\\
&=\frac1{4nR}e^\top H_{\mathrm{st}}e\\
&\ge\frac{\lambda_0}{16nR}\norm e^2
=\frac{\lambda_0}{8R}J(W).
\end{align}
Current Green bounds the left side by $\tau_G$, so
$J(W)\le8R\tau_G/\lambda_0$; nonnegativity then gives the asserted global-gap bound.

Taking equality in the allowed choice of $\gamma$ gives the explicit sufficient width
\[
p\ge
\max\left\{
 m,
 \frac{16R}{\lambda_0}\log\frac n\delta,
 \frac{32m^2E_{\max}^2n^2R}{\pi\lambda_0^2}
\right\}.
\]
This is a conservative certificate threshold, not a prediction of practical interpolation width.

\section{Proof of the Exact Non-Global Staircase}
\label{app:staircase-proof}

The construction is presented first as an open support-preserving family and then as a closed-form boundary instance. Let $m\ge3$, $n=m+1$, inputs $x_i=e_i$, and target
\[
y=(c_1,\ldots,c_m,d)^\top,
\qquad c_i>0,
\quad d>0,
\quad C^2=\norm c^2.
\]
Consider the complete two-unit class
\[
f_\theta(x)=a[w^\top x]_++b[v^\top x]_+,
\qquad
J(\theta)=\frac1{2n}\sum_{i=1}^n(f_\theta(x_i)-y_i)^2.
\]
The architecture permits both output coefficients to vary. The declared audited trainer holds the active coefficient fixed at $a=A>0$. Initialize $v_i<0$ for all coordinates and $w=(u,w_n)$ with $u>0$ and $w_n<0$. The second unit is dead on every sample, while the first is active on the first $m$ samples and inactive on the final sample. These strict inequalities define an open activation region. Inactive coordinates receive zero gradient, so full-batch gradient flow preserves the support.

After rescaling time, the active dynamics are
\begin{equation}
\dot u=-A(Au-c),
\qquad
u(t)=\frac cA+e^{-A^2t}z,
\qquad
u(0)=\frac cA+z.
\label{eq:stair-dynamics}
\end{equation}
The trainer converges to the support-restricted loss $d^2/(2n)$. The complete architecture has zero empirical loss: one unit can realize $(c^\top,0)$ and the other can activate only on the final coordinate and realize $d$. Thus $J^\star=0$.

At a checkpoint, freeze the hidden representation and solve the output head globally. The dead feature contributes nothing; the active feature equals $(u(t)^\top,0)$. The exact challenge chooses
\[
a_{\mathrm{ch}}(t)=\frac{c^\top u(t)}{\norm{u(t)}^2}
\]
and has value
\begin{equation}
B(t)=\frac1{2n}\left[d^2+C^2-\frac{(c^\top u(t))^2}{\norm{u(t)}^2}\right].
\label{eq:stair-head-value}
\end{equation}
The projection identity gives
\begin{equation}
J(t)-B(t)
=
\frac{\{A\norm{u(t)}^2-c^\top u(t)\}^2}{2n\norm{u(t)}^2}.
\label{eq:stair-positive-gap}
\end{equation}
For the strict open conditions
\[
u(0)>0,
\qquad
A^2\norm z^2>C^2,
\qquad
c^\top z>0,
\qquad
z\notin\operatorname{span}\{c\},
\]
the numerator in \cref{eq:stair-positive-gap} is positive at every finite time and
\[
\frac{d^2}{2n}<B(t)<J(t).
\]
The event policy is part of the construction. At initialization, compute and retain the exact head comparator as the online frontier. Before Green-0, the mandatory current suite replays that comparator and a declared identity/no-improvement call; both complete at every issuance. The state-dependent exact-head solver is invoked only at initialization, at Green-0, and at later hit events. The idealized construction uses exact crossing and acceptance, $\tau_{\mathrm{hit}}=\tau_{\mathrm{acc}}=0$; equivalently, one may use a vanishing tolerance schedule. Continuously recomputing the exact head value as the online frontier would define a different policy and can postpone Green entry.

Use the zero predictor as the fixed Red boundary. The monotonically decreasing trainer crosses the Red boundary, then the retained initialization frontier, and then each newly computed exact head stair. Each new stair lies strictly between the current trainer loss and the limiting floor, so continuity guarantees a later hit. Hit times cannot accumulate at a finite time: otherwise continuity would produce a finite checkpoint at which $J=B$, contradicting \cref{eq:stair-positive-gap}. Hence the open family yields infinitely many exact reached stairs. The family is nonempty: begin with the orthogonal vector used below and add a sufficiently small positive multiple of $c$. Coordinate positivity and $A^2\|z\|^2>C^2$ persist, while $c^\top z>0$ and $z\notin\operatorname{span}\{c\}$ become strict.

For a closed-form boundary reference, set $c=\mathbf 1_m$ and
\[
z=\lambda(m-1,-1,\ldots,-1)^\top,
\qquad
\frac1{A\sqrt{m-1}}<\lambda<\frac1A.
\]
Then $c^\top z=0$, every coordinate of $u(0)$ is positive, and
\[
E_0=A^2\norm z^2=A^2\lambda^2m(m-1)>m=C^2.
\]
Define $E(t)=E_0e^{-2A^2t}$. Direct substitution gives
\begin{align}
J(t)&=\frac{d^2+E(t)}{2n},
\label{eq:stair-J}\\
B(t)&=\frac{d^2+C^2E(t)/(C^2+E(t))}{2n},
\label{eq:stair-B}\\
J(t)-B(t)&=\frac{E(t)^2}{2n(C^2+E(t))}>0.
\label{eq:stair-gap-closed}
\end{align}
Let
\[
E_{j+1}=\frac{C^2E_j}{C^2+E_j},
\qquad
G=B(0)=\frac{d^2+E_1}{2n}.
\]
Since $E_0>C^2>E_1$, the trainer starts Red, enters Yellow, and reaches Green when its energy reaches $E_1$. The first post-Green challenge is computed at that crossing and equals
\[
S_1=\frac{d^2+E_2}{2n}<G.
\]
Inductively,
\[
S_j=\frac{d^2+E_{j+1}}{2n}.
\]
Taking reciprocals in the recurrence yields
\[
\frac1{E_{j+1}}=\frac1{E_j}+\frac1{C^2},
\qquad
E_j=\frac{C^2E_0}{C^2+jE_0}.
\]
Therefore every stair is an exact conditional head optimum for the frozen representation, strictly below the triggering loss, and reached by the actual trainer, while
\[
S_j\downarrow\frac{d^2}{2n}>0=J^\star,
\qquad
J(t)\downarrow\frac{d^2}{2n}.
\]
This proves \cref{thm:non-global}. Literal full-batch gradient descent obeys the same support-preserving argument when $0<\gamma A^2/n<1$, with crossing-based hit times allowing discrete overshoot.

\section{Challenge-Closed Endpoints}
\label{app:closure}

This appendix records the strongest generic endpoint statement available without global coverage. Work on a compact metric state space $\mathcal K$ containing model parameters, optimizer state, buffers, finite policy memory, and realized randomness. Let $p_\Theta:\mathcal K\to\Theta$ be the continuous model projection, let $J_\Theta:\Theta\to\R$ be continuous, and write $J=J_\Theta\circ p_\Theta$. For each hierarchy level $m$, let $\mathcal C_m(x)\subseteq\mathcal K$ be a nonempty compact endpoint correspondence containing the identity and nested in $m$. Define
\[
b_m(x)=\min_{y\in\mathcal C_m(x)}J(y),
\qquad
g_m(x)=J(x)-b_m(x),
\]
and
\[
\mathcal C_\infty(x)=\operatorname{cl}\, \bigcup_m\mathcal C_m(x),
\qquad
b_\infty(x)=\min_{y\in\mathcal C_\infty(x)}J(y),
\qquad
g_\infty(x)=J(x)-b_\infty(x).
\]
A state is $\tau$ challenge-closed when $g_\infty(x)\le\tau$.

\paragraph{Proof of \cref{thm:closure-limit-main}.}
Fix a finite level $m$. Since $m_k\to\infty$, eventually $m_k\ge m$. Nesting gives $b_{m_k}(x_k)\le b_m(x_k)$, while the solve/hit inequalities give the upper bound
\[
J(x_{k+1})-b_m(x_k)\le\epsilon_k+h_{k+1}.
\]
Continuity of $J$ and upper semicontinuity of $b_m$ imply
\[
J(\bar x)-b_m(\bar x)\le e.
\]
The lower bound is established separately at the limit: identity feasibility gives $b_m(\bar x)\le J(\bar x)$ and hence $0\le J(\bar x)-b_m(\bar x)$. The result holds for every finite $m$; taking the infimum over $m$ yields the exhaustive-closure bound.

The theorem is a limit characterization, not an assumption-free convergence theorem. A KL merit-function argument can establish convergence for compatible coupled policies, but arbitrary adaptive intervention paths need not possess such a merit function. Globality requires $\mathcal C_\infty(x)$ to be complete or a separate coverage theorem to bound $b_\infty(x)-J^\star$.

\section{Bellman Completion of Challenge Power}
\label{app:bellman}

For a bounded continuous value function $V$ define
\[
(\mathsf T_BV)(x)=\inf_{y\in\mathcal C_B(x)}V(y).
\]
The operator is monotone, translation equivariant, and nonexpansive in the supremum norm. If identity is feasible, $\mathsf T_BV\le V$; if budgets are nested, $B_2\ge B_1$ implies $\mathsf T_{B_2}V\le\mathsf T_{B_1}V$. If every $B_1$-then-$B_2$ composition belongs to the total-budget endpoint set,
\[
\mathsf T_{B_1+B_2}V\le\mathsf T_{B_1}\mathsf T_{B_2}V,
\]
with equality when the endpoint sets coincide exactly. For fixed stage budgets, backward recursion gives the exact finite-horizon intervention value.

\begin{theorem}[Scalar power surfaces do not determine transitions]
\label{thm:scalar-incomplete}
Two finite challenge systems can have the same gap function and identical $\Psi$ and $E$ surfaces. Yet, from the same named state $x$,
\[
(\mathsf T_1^A)^2\Delta(x)=0,
\qquad
(\mathsf T_1^B)^2\Delta(x)=0.9.
\]
Repeated unit-budget best-endpoint adoption reaches the optimum in two calls in System A and eleven in System B.
\end{theorem}
\begin{proof}
Use states $g,e,x,x',h_1,\ldots,h_{10}$ with
\[
\Delta(g)=0,
\quad
\Delta(e)=\Delta(h_1)=1,
\quad
\Delta(h_j)=1.1-0.1j,
\quad
\Delta(x)=\Delta(x')=2.
\]
All nonidentity edges cost one. Both systems contain $e\to g$ and $h_1\to h_2\to\cdots\to h_{10}\to g$. System A assigns $x\to e$ and $x'\to h_1$; System B swaps those assignments. At every budget, the improvements of the equal-gap states $x,x'$ are exchanged, leaving every scalar set of pairs $(\Delta,I(B))$ unchanged. Hence all infima $\Psi$ and suprema $E$ agree. Endpoint identity nevertheless changes the two-step value and the trajectory from $x$.
\end{proof}

\Cref{fig:bellman-counterexample} depicts the two statically indistinguishable systems and the different endpoint chains that the Bellman operator retains.

\begin{figure}[t]
\centering
\includegraphics[width=0.97\linewidth]{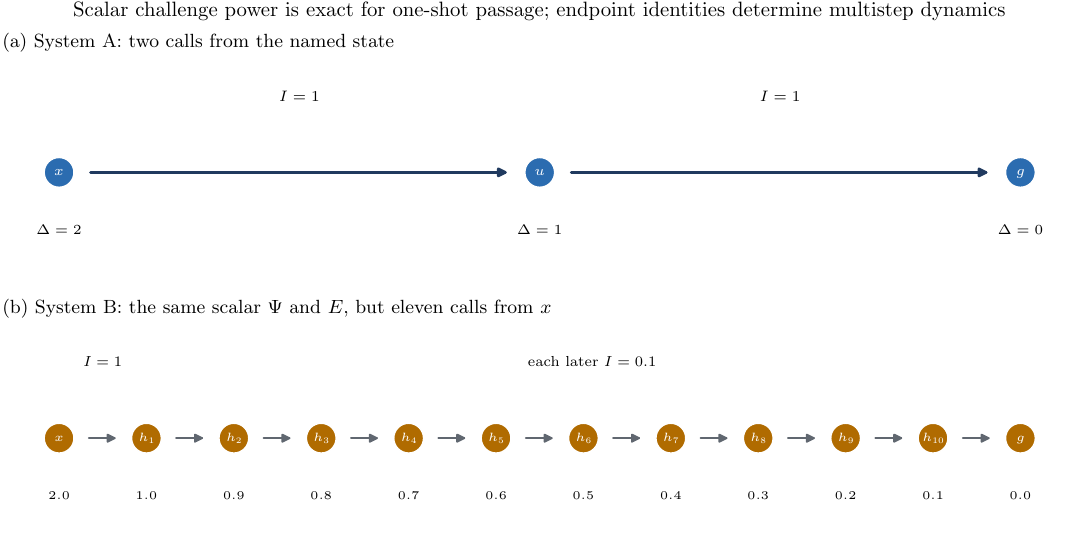}
\caption{Two challenge systems have the same complete static power and undetected-gap surfaces but different endpoint identities. The scalar modulus is exact for one-shot passage; the Bellman operator is needed for exact multistep dynamics.}
\label{fig:bellman-counterexample}
\end{figure}

\begin{theorem}[Continuous-test endpoint completeness]
\label{thm:operator-complete}
Let $\mathcal C_B(x)$ and $\widetilde{\mathcal C}_B(x)$ be nonempty compact endpoint sets in a metric space. If
\[
\inf_{y\in\mathcal C_B(x)}V(y)
=
\inf_{y\in\widetilde{\mathcal C}_B(x)}V(y)
\]
for every continuous $V$, then $\mathcal C_B(x)=\widetilde{\mathcal C}_B(x)$.
\end{theorem}
\begin{proof}
If $y\in\mathcal C_B(x)\setminus\widetilde{\mathcal C}_B(x)$, use $V(z)=d(z,y)$. Its infimum over the first compact set is zero and over the second is positive, a contradiction. Reverse the sets for the other inclusion.
\end{proof}

\section{Projection, Bottleneck, and Representation Anchors}
\label{app:projection-anchors}

These exact linear and rank-constrained problems calibrate three distinctions used throughout the paper: value inaccessible from the input, value excluded by the architecture, and value present in a representation but unused by the current decoder.

\subsection{Frozen heads and exact conditional gaps}

Let $H\in\mathbb R^{p\times n}$ and $P_H=H^\dagger H$. Because every $AH$ satisfies $AH=(AH)P_H$,
\[
Y-AH=Y(I-P_H)+(YP_H-AH),
\]
and the two terms are orthogonal. Hence
\begin{equation}
\|Y-AH\|_F^2
=
\|Y(I-P_H)\|_F^2+
\|YP_H-AH\|_F^2.
\label{eq:appendix-head-decomposition}
\end{equation}
The least-squares optimum is $A^\star=YH^\dagger$. The first term in \cref{eq:appendix-head-decomposition} is the best achievable loss of a linear decoder on the frozen representation; the second is the current decoder's exact conditional under-use.

For ridge-regularized heads,
\[
J_\lambda(A)=\|Y-AH\|_F^2+\lambda\|A\|_F^2,
\qquad
A^\star=YH^\top(HH^\top+\lambda I)^{-1},
\]
and completing the square gives
\[
J_\lambda(A)-J_\lambda(A^\star)
=
\operatorname{tr}\!\left((A-A^\star)(HH^\top+\lambda I)(A-A^\star)^\top\right).
\]
Logistic and softmax heads retain convexity; their certificates record the penalty, solver status, and finite optimality tolerance.

\subsection{Arbitrary-width deep-linear optimum}

Let a deep linear network have widths $d_0,\ldots,d_K$, bottleneck rank $r=\min_{0\le k\le K}d_k$, input matrix $X$, and $P_X=X^\dagger X$.

\begin{theorem}[Exact optimum for arbitrary-width deep linear networks]
\label{thm:deep-linear}
For squared loss,
\begin{equation}
\inf_{W_1,\ldots,W_K}
\|Y-W_K\cdots W_1X\|_F^2
=
\|Y(I-P_X)\|_F^2
+
\sum_{j>r}\sigma_j^2(YP_X).
\label{eq:deep-linear-exact}
\end{equation}
An end-to-end matrix attaining the optimum factors through every declared hidden width.
\end{theorem}
\begin{proof}
Every end-to-end output $Z=W_K\cdots W_1X$ satisfies $Z=ZP_X$ and $\operatorname{rank}(Z)\le r$. Since $Y(I-P_X)$ is orthogonal to every matrix whose rows lie in the row space of $X$,
\[
\|Y-Z\|_F^2
=
\|Y(I-P_X)\|_F^2+
\|YP_X-Z\|_F^2.
\]
The Eckart--Young--Mirsky theorem gives the singular-value tail in \cref{eq:deep-linear-exact} \citep{eckart1936,mirsky1960}. The truncated SVD $(YP_X)_r$ has row space contained in that of $X$, so $M=(YP_X)_rX^\dagger$ satisfies $MX=(YP_X)_r$ and $\operatorname{rank}(M)\le r$. Every matrix of rank at most $r$ factors through all hidden widths at least $r$.
\end{proof}

The frozen-representation counterpart follows from the same argument.

\begin{proposition}[Rank-constrained head optimum]
\label{prop:rank-head}
For $r\le\min(p,q)$,
\[
\min_{\operatorname{rank}(A)\le r}\|Y-AH\|_F^2
=
\|Y(I-P_H)\|_F^2+
\sum_{j>r}\sigma_j^2(YP_H).
\]
\end{proposition}
\begin{proof}
Apply \cref{eq:appendix-head-decomposition} and then the Eckart--Young--Mirsky theorem to the attainable component $YP_H$ \citep{eckart1936,mirsky1960}.
\end{proof}

\subsection{Target codes and class-preserving embeddings}

When a hidden width differs from the output dimension, a target code must respect that dimension.

\begin{proposition}[Optimal linear target code]
\label{prop:target-code}
For $0\le r\le d_Y$,
\[
\min_{C\in\mathbb R^{r\times d_Y},\,D\in\mathbb R^{d_Y\times r}}
\|Y-DCY\|_F^2
=
\sum_{j>r}\sigma_j^2(Y).
\]
For $r\ge1$, a leading-singular-vector code attains the optimum; at $r=0$, the reconstruction is zero.
\end{proposition}
\begin{proof}
Every $DCY$ has rank at most $r$, and the truncated SVD attains the Eckart--Young--Mirsky lower bound \citep{eckart1936,mirsky1960}.
\end{proof}

Zero-padding preserves the architecture when it is a constrained embedding. Let $E_k:\mathbb R^{d_k}\to\mathbb R^D$ be the canonical isometry and restrict every padded weight to $\widehat W_k=E_kW_kE_{k-1}^\top$. Coordinatewise activations satisfying $\Omega_k(0)=0$ then preserve $\widehat H_k=E_kH_k$ by induction, so the padded and original objectives agree. Free padded coordinates define a wider model class and therefore a different certificate problem.

\subsection{Projection references through depth}

For a dense one-layer map with activation $\Omega$, the minimum-norm linear solve $A_{\mathrm{lin}}=YX^\dagger$ yields the feasible architecture-executed reference
\[
B_{\mathrm{1L}}=\|Y-\Omega(YX^\dagger X)\|_F^2.
\]
For equal-width hidden states, projection can be iterated. Set $H_0=X$ and
\[
A_k^{(0)}=YH_{k-1}^\dagger,
\qquad
H_k^{(0)}=\Omega_k(YH_{k-1}^\dagger H_{k-1}).
\]

\begin{theorem}[One-step projection monotonicity]
\label{thm:projection-monotonicity}
Let $H,Y\in\mathbb R^{q\times n}$. If $\Omega$ is $1$-Lipschitz in Frobenius norm and fixes the target, $\Omega(Y)=Y$, then, for $H^+=\Omega(YH^\dagger H)$,
\[
\|Y-H^+\|_F
\le
\|Y-YH^\dagger H\|_F
\le
\|Y-H\|_F.
\]
Thus an equal-width architecture-valid projection reference is nonincreasing through depth.
\end{theorem}
\begin{proof}
The first inequality follows from target fixing and nonexpansiveness. For the second, $YH^\dagger$ minimizes $\|Y-AH\|_F$, while $A=I$ is feasible.
\end{proof}

Structured operators fit the same pattern. If $W(a)=\sum_{j=1}^{p}a_jB_j$, define
\[
\Phi(H)=
\begin{bmatrix}
\operatorname{vec}(B_1H)&\cdots&\operatorname{vec}(B_pH)
\end{bmatrix}.
\]
Then $a^\star=\Phi(H)^\dagger\operatorname{vec}(T)$ preserves the declared convolutional sharing, Toeplitz/circulant structure, graph filter, sparsity pattern, or unfolded operator family encoded by the $B_j$. The complete model evaluation supplies the certificate value.

\subsection{Reference tightness}

Let $B_0$ be a fixed executable reference and $\epsilon_0=B_0-J^\star\ge0$. Then
\[
J(\theta)-J^\star=[J(\theta)-B_0]+\epsilon_0.
\]
The height above the reference is a lower bound on the global gap. A proof-bearing estimate $\epsilon_0\le\bar\epsilon_0$ gives the two-sided interval
\[
J(\theta)-B_0
\le
J(\theta)-J^\star
\le
J(\theta)-B_0+\bar\epsilon_0.
\]
This relation calibrates the Red boundary: exactness of the reference turns its vertical height into the exact empirical gap; an approximate reference carries its certified tightness allowance.

\subsection{Rank vacuity, selectivity, and cross-fitting}

If $\operatorname{rank}(H)=n$, then $P_H=I_n$ and an in-sample linear challenge fits every target matrix, including permuted labels. Reusable representation evidence therefore adds effective rank, ridge degrees of freedom
\[
\operatorname{df}_\lambda
=
\sum_j\frac{\sigma_j^2(H)}{\sigma_j^2(H)+\lambda},
\]
permutation selectivity
\[
\operatorname{Sel}(H,Y)
=
\mathbb E_\pi[B(H,\pi Y)]-B(H,Y),
\qquad
B(H,Y)=\|Y(I-P_H)\|_F^2,
\]
and, most decisively, cross-fitting. The decoder is fit on one representation sample and the materialized predictor is evaluated on untouched data. \Cref{fig:crossfit} displays the resulting interpolation boundary and audit behavior.

\begin{figure}[t]
\centering
\includegraphics[width=0.97\textwidth]{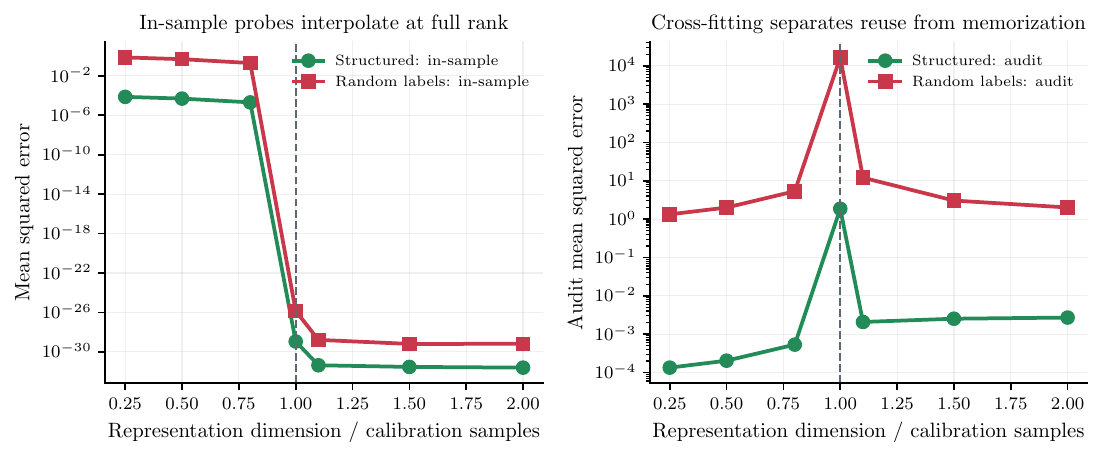}
\caption{Cross-fitting separates reusable representation value from in-sample interpolation. Structured and random labels can both be fitted when the representation reaches full sample rank, while only the structured relation transfers to untouched data; permutation selectivity collapses at the same boundary.}
\label{fig:crossfit}
\end{figure}

\section{Strong Solvers and Attainability Results}
\label{app:strong-solvers}

\subsection{Full-objective witnesses, conditional gaps, and solver intervals}

For a frozen-context block with exact conditional optimum $B_r^\star$, suppose a solver returns a feasible block value $U_r$ and a valid lower bound $L_r$:
\[
L_r\le B_r^\star\le U_r.
\]

\begin{theorem}[Convexification certificate]
\label{thm:convexification-plugin}
Materialize the primal block in the exact network and let $\widehat B_r$ be its complete reevaluated objective. Then
\[
J(\theta)-J^\star\ge J(\theta)-\widehat B_r.
\]
If the reduced problem is exact and objective-consistent so that $\widehat B_r=U_r$, then
\[
J(\theta)-U_r
\le
J(\theta)-B_r^\star
\le
J(\theta)-L_r,
\]
and the conditional interval has width $U_r-L_r$.
\end{theorem}
\begin{proof}
The materialized block is feasible for the full class, so $J^\star\le\widehat B_r$. The conditional interval follows by subtracting the primal--dual bracket from $J(\theta)$.
\end{proof}

The record therefore contains three distinct objects: a full-model witness, an exact or bracketed conditional gap, and residual solver uncertainty. A canonical minimum-norm or lexicographic materialization makes the challenge replayable even when parameter symmetries make the optimizer nonunique.

\subsection{Complete activation-pattern challenges}

For fixed input $H$ and target $T$, consider
\[
\min_{A\in\mathbb R^{q\times d}}
\|T-\operatorname{ReLU}(AH)\|_F^2.
\]
For $M\in\{0,1\}^{q\times n}$ define
\[
\mathcal P_M=\{A:(2M-\mathbf1)\odot AH\ge0\},
\qquad
B_M=\min_{A\in\mathcal P_M}\|T-M\odot AH\|_F^2,
\]
with $B_M=+\infty$ for an empty mask region.

\begin{theorem}[Complete activation-pattern challenge]
\label{thm:relu-pattern-complete}
Every nonempty mask problem attains its minimum and
\[
\inf_A\|T-\operatorname{ReLU}(AH)\|_F^2
=
\min_{M\in\{0,1\}^{q\times n}}B_M.
\]
Exact enumeration therefore gives the global conditional value. A complete pattern solver with gap at most $\varepsilon_{\mathrm{solve}}$ and a checkpoint within $\tau$ of its materialized value yields a $(\tau+\varepsilon_{\mathrm{solve}})$ conditional optimality certificate for the declared segment.
\end{theorem}
\begin{proof}
Every $A$ induces a sign mask, with zeros assigned arbitrarily, for which $A\in\mathcal P_M$ and $\operatorname{ReLU}(AH)=M\odot AH$. Conversely, every mask-feasible point has the same ReLU and masked-linear outputs. Each $\mathcal P_M$ is a closed polyhedron; its image under the linear map $A\mapsto M\odot AH$ is a closed polyhedron, so the quadratic distance minimum is attained.
\end{proof}

Partial pattern libraries, cutting planes, cone generation, and branch-and-bound form a graded solver hierarchy. Figure~\ref{fig:convexification-dividend} compares the canonical global conditional value with restart-dependent Adam outcomes.

\begin{figure}[t]
\centering
\includegraphics[width=0.94\textwidth]{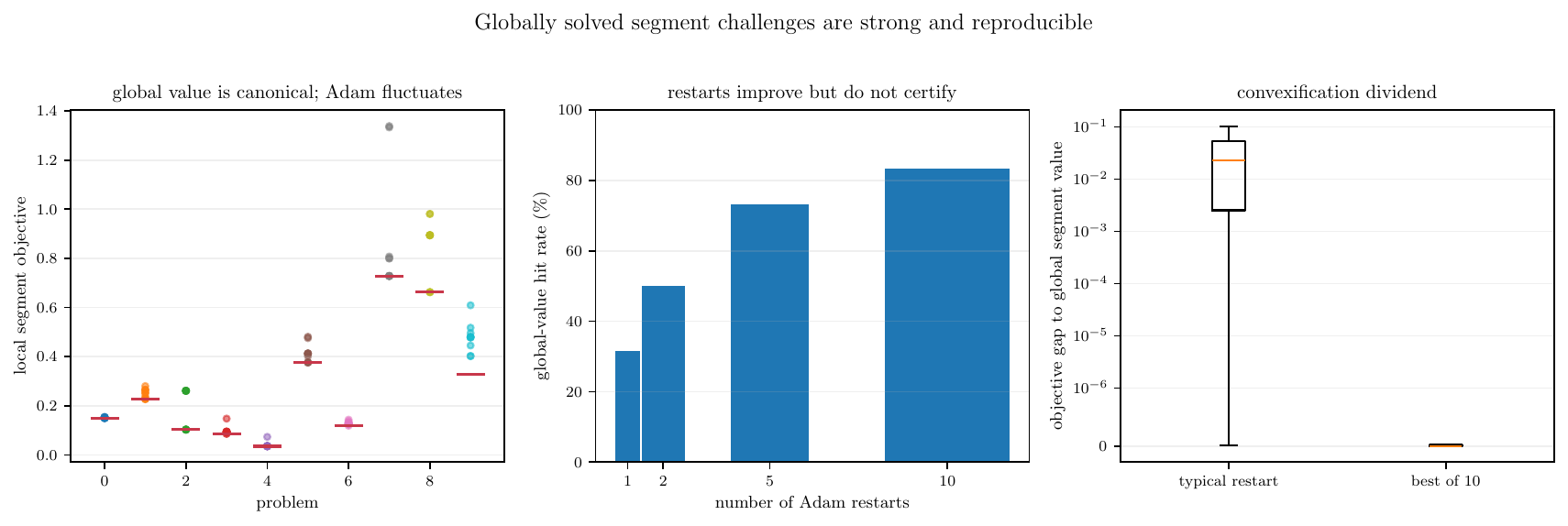}
\caption{An exact conditional challenge supplies a canonical value that local restarts can miss. Across fixed-input ReLU segment problems, exact activation-pattern enumeration is initialization-independent; additional Adam restarts improve the hit rate while leaving an uncertified omitted gap.}
\label{fig:convexification-dividend}
\end{figure}

\subsection{Mini-batches and shared-model feasibility}

Suppose $J(\theta)=\sum_b\pi_bJ_b(\theta)$ with positive weights summing to one.

\begin{proposition}[Batchwise infima can be infeasible]
\label{prop:minibatch-invalid}
\[
\sum_b\pi_b\inf_{\theta\in\Theta}J_b(\theta)
\le
\inf_{\theta\in\Theta}\sum_b\pi_bJ_b(\theta),
\]
and strict inequality can occur. The average of independently optimized batchwise infima is a certificate value exactly when one common parameter state realizes the displayed full objective.
\end{proposition}
\begin{proof}
For each fixed $\theta$, $\inf_\vartheta J_b(\vartheta)\le J_b(\theta)$. Sum and minimize. Strictness holds for a scalar constant model on two singleton batches with labels $1$ and $-1$: the separate optima have average loss zero, whereas the best shared model has positive full-data loss.
\end{proof}

Mini-batches can optimize one shared candidate that is subsequently frozen and evaluated on the complete objective, or estimate the value of an already fixed model with valid uncertainty control.

\subsection{Stopped stochastic stair tracking}
\label{app:attainability-plugins}

The deterministic finite-time and target-relative PL results appear in \cref{thm:pl-attainment}. The stochastic analogue uses the complete post-Green episode as a stopped process.

\begin{theorem}[Conditional-drift stair tracking after Green entry]
\label{thm:stair-drift-appendix}
Let $\tau_0<\infty$ almost surely be the Green-entry stopping time. Let the active stair $S$ and hit tolerance $\varepsilon_{\mathrm{hit}}>0$ be $\mathcal F_{\tau_0}$-measurable. For $\mathcal G_n=\mathcal F_{\tau_0+n}$ define
\[
\sigma=\inf\{n\ge0:J(\theta_{\tau_0+n})\le S+\varepsilon_{\mathrm{hit}}\},
\qquad
\widetilde\Delta_n=(J(\theta_{\tau_0+n})-S)_+\mathbf1_{\{n<\sigma\}}.
\]
Assume integrability, that $\sigma$ is a stopping time, and that almost surely
\begin{equation}
\mathbb E[\widetilde\Delta_{n+1}\mid\mathcal G_n]
\le
(1-\alpha)\widetilde\Delta_n+\beta,
\qquad 0<\alpha\le1,\quad\beta\ge0.
\label{eq:stopped-drift}
\end{equation}
Then
\begin{align}
\mathbb E[\widetilde\Delta_n\mid\mathcal F_{\tau_0}]
&\le
(1-\alpha)^n\widetilde\Delta_0+
\frac{\beta}{\alpha}[1-(1-\alpha)^n],
\label{eq:stopped-drift-expectation}\\
\mathbb P(\sigma>n\mid\mathcal F_{\tau_0})
&\le
\frac{(1-\alpha)^n\widetilde\Delta_0+(\beta/\alpha)[1-(1-\alpha)^n]}
{\varepsilon_{\mathrm{hit}}}.
\label{eq:stopped-hit-probability}
\end{align}
\end{theorem}
\begin{proof}
The indicator absorbs the process at zero after the first hit. Apply \cref{eq:stopped-drift} to the shifted filtration and use the tower property inductively. Summing the geometric series gives \cref{eq:stopped-drift-expectation}. On $\{\sigma>n\}$, $\widetilde\Delta_n>\varepsilon_{\mathrm{hit}}$; conditional Markov inequality gives \cref{eq:stopped-hit-probability}.
\end{proof}

\begin{corollary}[Noisy-SGD instantiation with overshoot control]
\label{cor:sgd-stair-drift}
Suppose $J$ is $L$-smooth and before the hit
\[
\theta_{k+1}=\theta_k-\gamma g_k,
\quad
\mathbb E[g_k\mid\mathcal F_k]=\nabla J(\theta_k),
\quad
\mathbb E[\|g_k\|^2\mid\mathcal F_k]
\le c\|\nabla J(\theta_k)\|^2+\sigma_g^2.
\]
Assume $\|\nabla J(\theta)\|^2\ge2\mu_S[J(\theta)-S]$ on the pre-hit band, $0<\gamma\le1/(Lc)$, $0<\mu_S\gamma\le1$, and
\[
\mathbb E[(S-J(\theta_{k+1}))_+\mid\mathcal F_k]\le\bar\omega.
\]
Then \cref{eq:stopped-drift} holds with
\[
\alpha=\mu_S\gamma,
\qquad
\beta=\frac{L\gamma^2\sigma_g^2}{2}+\bar\omega.
\]
The expected tracking error has geometric transient and floor $\beta/\alpha$.
\end{corollary}
\begin{proof}
Smoothness, conditional unbiasedness, and the second-moment bound give
\[
\mathbb E[J(\theta_{k+1})-S\mid\mathcal F_k]
\le
(1-\mu_S\gamma)[J(\theta_k)-S]
+
\frac{L\gamma^2\sigma_g^2}{2}.
\]
For $X=J(\theta_{k+1})-S$, $X\mathbf1_{\{X>\varepsilon_{\mathrm{hit}}\}}\le X+(-X)_+$. The overshoot allowance converts the signed drift into the stopped nonnegative recursion.
\end{proof}

\subsection{Task-calibrated translations}

Several objectives admit direct task translations when the monitored quantity equals the task risk. For image-domain MSE with peak value $M$,
\[
\operatorname{PSNR}=10\log_{10}\frac{M^2}{J_{\mathrm{MSE}}},
\]
so an upper bound on MSE yields a PSNR floor. For average softmax cross-entropy with natural logarithms and a fixed argmax tie rule, every misclassified sample contributes at least $\log2$, hence
\[
\widehat{\operatorname{err}}_{0-1}\le\frac{J_{\mathrm{CE}}}{\log2}.
\]
For a contraction $T$ with modulus $\rho<1$ and fixed point $x^\star$,
\[
\|x-x^\star\|\le\frac{\|x-T(x)\|}{1-\rho}.
\]
These translations apply when the certification objective is the task quantity or when a transfer theorem connects them.

\section{Coverage Certificate Details}
\label{app:coverage-details}

\subsection{Atomic ReLU bracket and complete separation}

For nonzero hidden weight $u=tv$ with $t=\|u\|$ and $\|v\|=1$, positive homogeneity gives
$a(Zu)_+=(at)(Zv)_+$. For a fixed signed atomic coefficient $c=at$,
\begin{equation}
\frac12(a^2+t^2)\ge |c|,
\label{eq:atomic-balancing}
\end{equation}
with equality at $|a|=t=\sqrt{|c|}$. Thus every width-$M$ adapter maps to an atomic representation with at most $M$ signed atoms and no larger objective. Conversely, every finite signed atomic representation maps to balanced neural parameters with the same objective.

The Fenchel conjugate of $\lambda\|\cdot\|_{\mathcal A_Z}$ is the indicator of the polar constraint $\sigma_Z(\nu)\le\lambda$. Since the squared-loss term is continuous everywhere, Fenchel--Rockafellar duality gives
$P_Z^\star=D_Z^\star$ \citep{rockafellar1970}. Combining this equality with \cref{eq:atomic-balancing} yields
\[
P_Z^\star\le J_{\mathrm{ad,M}}^\star
\]
for every finite width. If an atomic optimum uses at most $M$ atoms, equality holds. An optimal nonzero prediction normalized by its gauge lies in the convex hull of the signed atom set in $\R^n$, so Carath\'eodory gives a representation using at most $n+1$ atoms \citep{caratheodory1911}; the zero prediction needs none. Hence $M\ge n+1$ is sufficient.

The separator is constructive. For a proposed dual vector $\nu$, a complete separator returns an upper enclosure
$\overline\sigma\ge\sigma_Z(\nu)$. With
\[
\rho=
\begin{cases}
1, & \overline\sigma\le\lambda,\\
\lambda/\overline\sigma, & \overline\sigma>\lambda,
\end{cases}
\qquad
\rho\sigma_Z(\nu)\le\rho\overline\sigma\le\lambda.
\]
Hence $\rho\nu$ is dual feasible. Downward rounding of
$y^\top(\rho\nu)-\|\rho\nu\|^2/2$
therefore preserves a valid floor. Every outward-rounded complete adapter preserves a valid ceiling. Taking the maximum retained floor and minimum retained ceiling proves \cref{thm:adapter-coverage}.

A violated polar constraint also supplies an improving atom. If $h=(Zv)_+$ satisfies $|\nu^\top h|=\sigma>\lambda$, adding the favorable sign with its optimal one-dimensional coefficient lowers the restricted atomic objective by
\[
\frac{(\sigma-\lambda)^2}{2\|h\|^2}
\ge
\frac{(\sigma-\lambda)^2}{2\|Z\|_{\mathrm{op}}^2}.
\]
Thus complete separation either produces quantitative executable headroom or certifies the floor.

For rank two, the lines $z_i^\top v=0$ partition the unit disk into at most $2n$ angular sectors. On each sector, $v\mapsto\nu^\top(Zv)_+$ is linear; its maximum absolute value occurs at a normalized sector gradient when feasible or at a boundary ray. Sorting the rays and checking these candidates is complete. The implementation interprets every serialized binary64 input as an exact rational, computes squared support values exactly, encloses the final square root from above by a dyadic rational, rounds dual values downward, and reevaluates every materialized primal network outward. Arithmetic error can therefore only widen the bracket.

Completeness is essential. For $Z=I_2$, $y=e_2$, and $\lambda=1/2$, a library containing only atom $e_1$ accepts a false dual value because the omitted $e_2$ constraint is active. Likewise, a bracket for a frozen adapter does not lower-bound a larger class that may alter the backbone, regularizer, data, or width convention.

\subsection{Finite-trial population calibration}

For a frozen policy, let $q_i$ be the checkpoint-specific probability of constructing a material witness and let
$S_i\mid q_i\sim\operatorname{Binomial}(m,q_i)$. The calibration checkpoints are i.i.d., and conditional on those checkpoints all policy-call randomness is fresh and independent across calibration episodes and calls. For a candidate $p$, consider the composite null
\[
H_p:\quad \Prb(q_i\le p)\ge\beta.
\]
Because
$g_{m,c}(q)=\Prb\{\operatorname{Binomial}(m,q)\le c\}$
is nonincreasing,
\begin{equation}
\Prb(S_i\le c)=\E g_{m,c}(q_i)
\ge \beta g_{m,c}(p).
\label{eq:mixing-domination}
\end{equation}
Consequently,
$T=\sum_i\mathbf1\{S_i\le c\}$
stochastically dominates a binomial variable with success probability $\beta g_{m,c}(p)$. Its lower-tail probability is therefore a conservative $p$-value for $H_p$ and is monotone in $p$. Inverting these nested tests gives the two beta inversions in \cref{eq:population-pl}: the first upper-bounds the marginal probability of a low-count checkpoint and the second maps that upper bound through the binomial lower tail into a lower bound on the latent detection quantile.

With calibration confidence $1-\alpha$, at most a fraction $\beta$ of the declared failure population has $q<p_L$. On the covered fraction, $k$ conditionally independent future calls miss with probability at most $(1-p_L)^k$. Assigning miss probability one to the uncovered fraction gives \cref{eq:population-miss}. No common per-checkpoint detection probability is assumed.

The whole frozen adaptive episode can also be treated as one Bernoulli unit. If $K$ of $n$ independent calibration episodes detect a material witness, the one-sided Clopper--Pearson lower endpoint \citep{clopper1934}
\[
\underline d_\alpha(K,n)=
\begin{cases}
0,&K=0,\\
F^{-1}_{\operatorname{Beta}(K,n-K+1)}(\alpha),&K>0
\end{cases}
\]
gives, with confidence $1-\alpha$, marginal miss probability at most $1-\underline d_\alpha(K,n)$. Adaptation inside an episode is immaterial because only the final detect/miss outcome is calibrated.

Unequal predeclared trial counts can be handled by replacing the binomial tail with a Poisson--binomial tail. Hyperparameters may be selected on independent development data. Reusing calibration data for selection requires a fresh split, a valid family-wise correction, or an anytime-valid procedure. A rare support component with $q=0$ proves that no finite distribution-free method can guarantee a positive pointwise detection probability at every support point of an unrestricted population.

\subsection{Coverage regressions}

The deterministic suite contains independent separator checks against floating optimization, the incomplete-library counterexample, 36 randomized bracket instances, and frozen-CNN representations. The statistical suite contains 1,000 heterogeneous-population repetitions and 5,000 policy-selection repetitions. All registered exact-arithmetic, interval-containment, confidence-coverage, and selection-correction regressions passed; machine-readable ledgers and figure scripts accompany the source.

\section{Architecture-Native Exact Blocks and Materialization}
\label{app:architecture}

\subsection{Residual CNN and pre-LayerNorm transformer motifs}

The affine block theorem applies when a graph cut freezes all nonlinear upstream computation into sample-specific features $h_i$ and offsets $b_i$, and the challenged internal variable enters the complete prediction as
\[
\widehat y_i(U)=b_i+WU h_i.
\]
In a preactivation residual CNN, $U$ can be the final internal convolution before a linear residual addition, global average pooling, and fixed classifier. In a pre-LayerNorm transformer, $U$ can be the MLP down-projection before the second residual addition and a fixed linear pooling/classifier, provided there is no terminal normalization or other post-residual nonlinearity. Convolutional sharing is represented by its structured im2col design; the algebra remains exact.

The motif boundary is sharp. A terminal LayerNorm generally destroys affinity. For $x(u)=(u,0)$, ordinary two-dimensional LayerNorm with zero shift, unit scale, and $\epsilon>0$, followed by a classifier selecting the first normalized coordinate, produces
\[
g(u)=\frac{u}{\sqrt{u^2+4\epsilon}},
\]
with $g(2a)\ne2g(a)$ for $a\ne0$. Terminally normalized architectures therefore require a different conditional challenge or an additional structural argument. The shared affine motif, rank-deficient formula, and terminal-LayerNorm boundary follow from the preceding argument for both architectures.

\subsection{Exact QNN output-projection materialization}

In the QNN denoiser, the final quantized feature map is followed by a full-precision bias-free $1\times1$ bottleneck $U\in\R^{d\times p}$ and a full-precision $1\times1$ RGB readout $W\in\R^{3\times d}$ with bias $b$. For a solved affine map $AH(x)+b\mathbf1^\top$ and $d\ge3$, set
\[
U=\begin{bmatrix}A\\0\end{bmatrix},
\qquad
W=\begin{bmatrix}I_3&0\end{bmatrix}.
\]
There is no activation, normalization, clipping, shared buffer, or quantizer between the two altered layers, so $WUH(x)+b\mathbf1^\top=AH(x)+b\mathbf1^\top$ pointwise. The float64 solve is cast to model precision and then the complete QNN is reevaluated; any casting discrepancy is therefore included in the candidate value. An operator-level regression over the released states found a maximum direct-affine versus two-layer discrepancy of $6.7\times10^{-16}$ on deterministic test tensors.

\section{Proper-Score and Finite-Sample Details}
\label{app:information-details}

Let $S(Q,y)$ be a proper scoring loss, $P_X=P(Y\in\cdot\mid X)$, $P_Z=P(Y\in\cdot\mid Z)$, and $Q_Z=q_w(\cdot\mid Z)$. Define score entropy and regret divergence
\[
H_S(P)=\E_{Y\sim P}S(P,Y),
\qquad
D_S(P,Q)=\E_{Y\sim P}S(Q,Y)-H_S(P).
\]
Assuming measurability and no undefined $\infty-\infty$ subtraction,
\begin{equation}
R_S(Q_Z)-R_{S,X}^\star
=
\E D_S(P_Z,Q_Z)
+
\E D_S(P_X,P_Z).
\label{eq:proper-score-decomp}
\end{equation}
Conditioning on $Z$ gives the first term; the Markov relation $Y-X-Z$ and iterated expectation identify the entropy difference with the second. For strict scores, the representation term vanishes exactly when $P_{Y\mid X}=P_{Y\mid Z}$ almost surely. Under log loss, it is $I(Y;X\mid Z)$.

For a finite independent audit sample and bounded loss in $[0,B]$, suppose the encoder, classes, solver policy, and hyperparameters are frozen before the data are inspected. If empirical solvers give
\[
\widehat L_j\le\inf_{f\in\mathcal F_j}\widehat R_n(f)\le\widehat U_j,
\qquad j\in\{X,Z\},
\]
and uniform deviations satisfy $\sup_{f\in\mathcal F_j}|R(f)-\widehat R_n(f)|\le\epsilon_j$, then
\[
\widehat L_j-\epsilon_j
\le R_j^\star\le
\widehat U_j+\epsilon_j.
\]
Applying \cref{thm:nested-rep} to these corrected intervals yields a population representation certificate. For finite classes of sizes $N_j$, Hoeffding's inequality and a union bound give \citep{hoeffding1963}
\[
\epsilon_j=B\sqrt{\frac{\log(4N_j/\delta)}{2n}}
\]
with joint probability at least $1-\delta$. Rademacher, PAC--Bayes, compression, or stability bounds may replace this expression when their selection assumptions match the protocol.

\section{Experimental and Reproducibility Details}
\label{app:experiments}

\subsection{Theory-regression commands}

From the artifact root:
\begin{verbatim}
python repro/challenge_power/challenge_power_regression.py
python repro/neural_green/src/run_regressions.py
python repro/information_sufficiency/src/theorem_regression.py
cd repro/current_state_certificate
python current_state_spectral_certificate.py
\end{verbatim}
The first and third programs are deterministic CPU tests. The neural program is also CPU-only under the registered seeds and writes its complete CSV/JSON/figure outputs to the package directories. The second independently supplied neural package can be rerun with
\begin{verbatim}
cd repro/neural_green_alternative
python source/run_experiments.py --root .
\end{verbatim}
when that optional archive is present.

\subsection{Current-state nonzero-optimum certificate}

The current-state study uses the scikit-learn Digits data. A binary even/odd label is used only to train a nonlinear representation: a 64--128--64 tanh feature network with a scalar logistic head is trained on a stratified training split. A disjoint certification split is standardized using training statistics, and 256 certification samples are selected by the registered seed. The frozen 64-dimensional feature vector is augmented with an intercept, producing a $256\times65$ design of rank 65. Certification uses squared loss on labels encoded as $\{-1,+1\}$; the full scalar head optimum is computed by least squares.

The native suite splits the 65 ordinary head coordinates into eight balanced bins. The designed suite constructs twelve rank-18 subspaces in the reachable prediction space, solves the E-optimal mixture numerically, and maps every prediction direction back to an ordinary head-weight direction through the representation pseudoinverse. Every challenge exactly solves the frozen-context least-squares problem in its subspace, materializes an ordinary head vector, and reevaluates the complete certification objective. The maximum relative materialization residual is $2.20\times10^{-14}$. The optimum residual projection onto the reachable prediction space is $7.64\times10^{-14}$, verifying the normal-residual premise to numerical precision. The CSV ledger records the exact objective, optimum, current gap, best improvement, certificate, and tightness ratio at every checkpoint. No held-out task claim is inferred from this certification study.

\subsection{QNN data, suites, and replay}

The denoiser uses eight source images for 4,096 fixed noisy/clean $48\times48$ RGB training patches with Gaussian noise standard deviation $25/255$. A disjoint 512-patch audit and three crops from source images withheld entirely from training and challenge construction provide protected checks. The Core suite, Full suite, strong post-Green policy, current-state recertification rule, and same-checkpoint controls are defined in the accompanying QNN README files.

The 15-run centerpiece archive contains 290 run-specific states and five baseline states in the external full artifact. Every exported run-specific state was independently reloaded and forward-evaluated with zero recorded objective discrepancy. The 125-trajectory exact-head trainer-design study and 75-endpoint same-checkpoint control have compact integration records, hashes, paired summaries, and exact replay ledgers. The paper package does not fabricate or replace the large external state archive; it records its canonical filename and checksum protocol.

\subsection{Compute}

The QNN centerpiece ran on an NVIDIA A100-SXM4-40GB GPU. 
Certification was therefore comparable to training cost in this exhaustive study. Theory regressions are CPU-scale.

\subsection{Visual and numerical checks}

All central mathematical claims have exact or high-precision regression tests for normalization, boundary cases, degenerate rank, solver error, and the positive/negative constructions. 

\begingroup
\small
\setlength{\bibsep}{2pt plus 0.3ex}
\bibliography{references}
\endgroup

\end{document}